\documentclass[twoside,11pt]{article}

\usepackage{jmlr2e}
\usepackage{subcaption}
\usepackage{url}
\usepackage{epsfig,epstopdf,algorithmic,algorithm,amsfonts,url,textcomp,multirow,moresize}
\usepackage{graphicx}
\usepackage{booktabs}
\usepackage{algorithm}
\usepackage{algorithmic}
\usepackage{amsfonts}
\usepackage{subcaption}
\usepackage{amsmath}\allowdisplaybreaks
\usepackage[framemethod=tikz]{mdframed}
\usepackage[flushleft]{threeparttable}
\usepackage{enumitem}
\DeclareMathOperator*{\argmin}{argmin}

\usepackage{etoolbox}

\newenvironment{myitemize}
{ \begin{center} \begin{em} }
{ \end{em} \end{center} }

\begin{document}

\title{Privacy-Preserving Asynchronous Federated Learning Algorithms for Multi-Party Vertically Collaborative Learning}

\author{\name Bin Gu \email jsgubin@gmail.com      \\
\addr JD Finance America Corporation \\
\name An Xu \email {an.xu@pitt.edu} \\
   \name Zhouyuan	Huo \email {zhouyuan.huo@pitt.edu} \\
 \addr Department  of  Electrical  \&  Computer Engineering, University of Pittsburgh, USA \\
\name Cheng Deng \email {chdeng@mail.xidian.edu.cn} \\
 \addr School   of   Electronic   Engineer-ing,  Xidian  University,  Xi'an,  China \\
   \name Heng Huang \email {heng@uta.edu} \\
\addr JD Finance America Corporation \\ University of Pittsburgh, USA\\
}

\editor{}
\maketitle

\begin{abstract}
The privacy-preserving federated learning for vertically partitioned data has shown promising results as the solution of the emerging multi-party joint modeling application, in which the data holders (such as government branches, private finance and e-business companies) collaborate throughout the learning process rather than relying on a trusted third party to hold data.
However, existing federated  learning  algorithms  for vertically partitioned data are limited to synchronous computation.  To improve the efficiency when the unbalanced computation/communication resources are common among the parties in the federated learning system, it is essential to develop asynchronous training algorithms for vertically partitioned data while keeping the data privacy.  In this paper, we  propose an asynchronous  federated SGD (AFSGD-VP) algorithm and its SVRG and SAGA variants  on the vertically partitioned data. Moreover, we provide the convergence analyses of AFSGD-VP and its SVRG and SAGA variants under the condition of strong convexity. We also discuss their model privacy, data privacy, computational complexities and communication
costs. To the best of our knowledge,  AFSGD-VP and
its SVRG and SAGA variants are the first asynchronous  federated learning algorithms for vertically partitioned data. Extensive experimental results on  a variety of vertically partitioned datasets not only verify the theoretical results of AFSGD-VP and its SVRG and SAGA variants, but also show that our algorithms have much higher efficiency than the corresponding synchronous algorithms.
\end{abstract}

\section{Introduction}
Federated learning facilitates the collaborative model learning without the sharing of raw data, and increasingly attracts attentions from both tech giants and industries where privacy protection is required. Especially, in the emerging multi-party joint modeling application, the data locate at multiple (two or more) data holders and  each maintains its  own records of different feature sets with common entities, which are called as vertically partitioned  data \citep{yang2019federated}.
While an integrated dataset improves the performance of a trained learning model, organizations cannot share data due to legal restrictions or competition between participants.
For example, a digital finance company, an E-commerce company,  and a bank collect different information of the same person. The digital finance company  has access
to  online consumption, loan and repayment   information. The E-commerce company has access
to the online shopping information. The bank has customer information like average monthly deposit, account balance. If the person submits a loan application  to the digital finance company, it might want to evaluate the  credit risk  of approving this financial loan   by   comprehensively utilizing  the information stored in all the three parties. Such scenarios have been popularly appearing in recent industrial applications and raise the need of efficient federated learning algorithms on the vertically partitioned data.

For the vertically partitioned  data, the direct access to the data in other providers or  sharing of the data are often prohibited due to the legal and commercial issues. For the legal reason, most countries  worldwide have made laws in protection of data security
and privacy. For example, the
European Union made the General Data Protection Regulation (GDPR) \cite{Regulation2018} to protect users' personal privacy and data
security.  The recent data breach by Facebook has caused a wide range of protests \cite{badshah2018facebook}.
For the commercial reason, customer data is usually a valuable
business asset for corporations. For example, the real online shopping information of customers can be used to train a  recommended model which could provide valuable product recommendations to customers.
Thus, both of the causes require  federated learning on the vertically  partitioned data without the disclosure of data.

In the literature, there are many privacy-preserving federated learning algorithms for vertically partitioned data in various applications, for example, cooperative statistical analysis \citep{du2001privacy}, linear regression \citep{gascon2016secure,karr2009privacy,sanil2004privacy,gascon2017privacy}, association rule-mining \cite{vaidya2002privacy}, k-means clustering \citep{vaidya2003privacy}, logistic regression \citep{hardy2017private,nock2018entity}, XGBoost \citep{cheng2019secureboost}, random forest \citep{liu2019federated}, support vector machine \citep{yu2006privacy}. From the optimization standpoint,  \citep{wan2007privacy} proposed privacy-preservation gradient descent algorithm for vertically partitioned data.  \citep{zhang2018feature} proposed a feature-distributed SVRG algorithm (FD-SVRG) for high-dimensional linear classification.
 However, to the best of our knowledge,  existing  federated  learning  algorithms  on the vertically partitioned data are limited to synchronous computation.

 Stochastic gradient descent (SGD) algorithm \citep{bottou2010large} and its   variants \citep{gu2018faster,defazio2014saga,schmidt2017minimizing,fang2018spider,gu2019scalable,huo2018accelerated}   have been dominant to train  large-scale  machine learning problems. Specifically, at each iteration SGD independently samples a sample, and uses the stochastic gradient with respect to the sampled sample to update the solution. The stochasticity makes each iteration
of SGD cheap while it also causes a large variance of stochastic gradients due to  random sampling. To reduce the variance of stochastic gradients, the SGD variants with different  variance reduction techniques (including SVRG \cite{gu2018faster},  SAGA \citep{defazio2014saga}, SAG \citep{schmidt2017minimizing}, SARAH \citep{pham2019proxsarah}, SPIDER \citep{fang2018spider}) were proposed to speed up SGD.  SVRG and SAGA are the most popular ones among them.
 In addition, SGD and its  adaptive variants (\emph{e.g.}, Adagrad, RMSProp and Adam \citep{goodfellow2016deep}) have  shown their successes for  the training of deep neural networks.

However,  it is still vacant for SGD and its  various  variance reduction variants to train vertically partitioned data in parallel and asynchronously while keeping  data and model privacy. To the best of our knowledge, FD-SVRG \cite{zhang2018feature} is the only work of privacy-preservation SGD-like methods for vertically partitioned data. However,   the updating rules in FD-SVRG \cite{zhang2018feature} are executed synchronously. As we know, the asynchronous computation
is much more efficient than the synchronous computation,
because it keeps all computational resources busy all the time (please see Figure \ref{Asysyn}).
Although there have been a lot of  asynchronous SGD-like algorithms    proposed to solve   large-scale learning problems on horizontally partitioned data \citep{zhao2016fast,mania2015perturbed,huo2017asynchronousAAAI,leblond2017asaga,meng2016asynchronous,gu2016asynchronous,kungurtsev2019asynchronous,gu2019asynchronous,gu2018asynchronous,gu2018asynchronous2,huo2018decoupled,huo2018training},  it is still a  challenge for SGD-like methods  to train the vertically partitioned data asynchronously  while keeping  data and model privacy.
  \begin{figure}[ht]
 \center
 \includegraphics[scale=0.6]{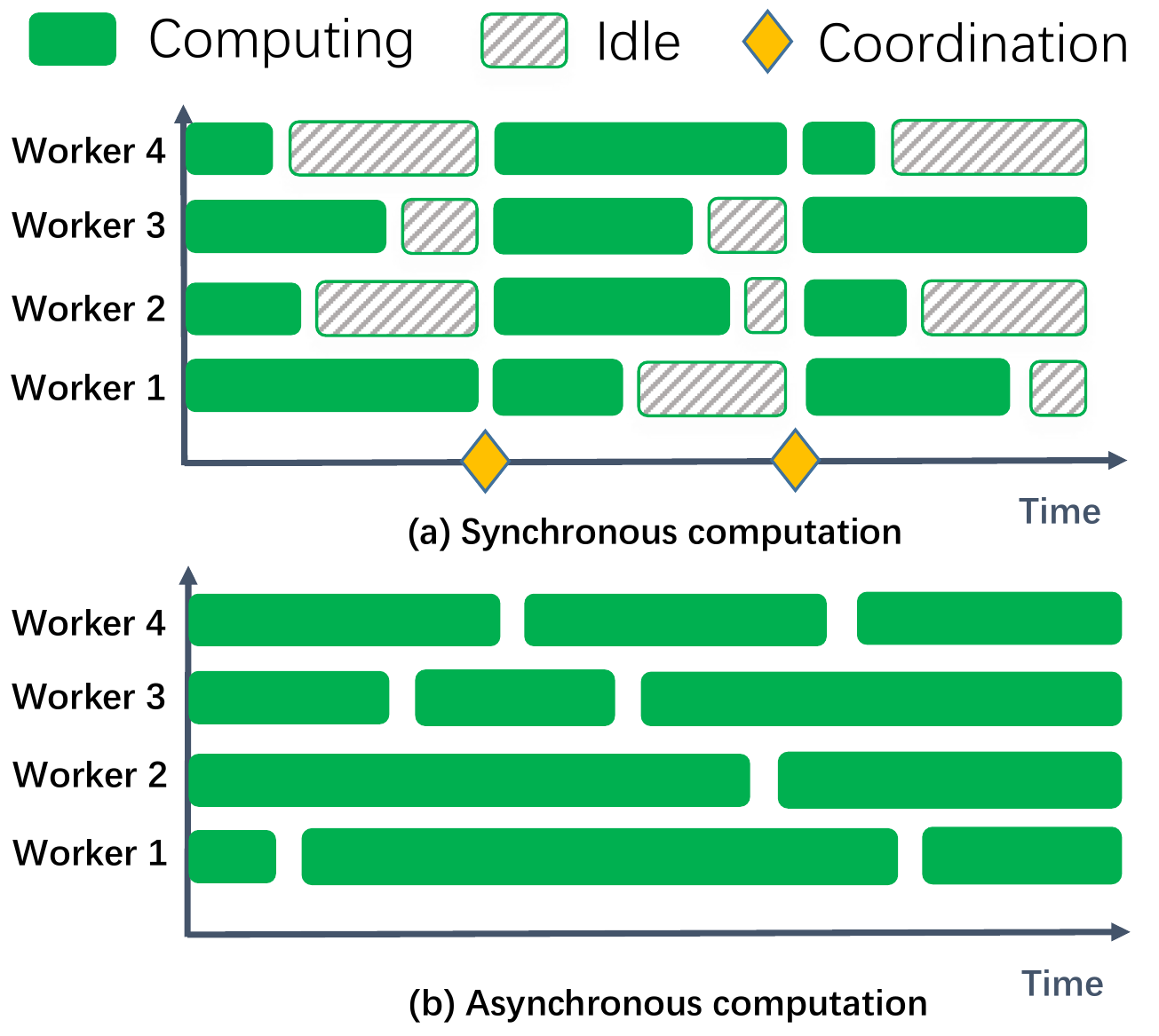}
  \caption{Asynchronous  computation vs. synchronous  computation.}
  \label{Asysyn}
 \end{figure}

To address this challenging problem, in this paper, we  propose an asynchronous  federated  SGD (AFSGD-VP) algorithm and its SVRG and SAGA variants  for vertically partitioned data. More importantly, we provide the convergence rates of AFSGD-VP and its SVRG and SAGA variants under the condition of  strong convexity for the objective function. We also discuss their model privacy, data privacy, computational complexities and communication
costs. To the best of our knowledge,  the proposed algorithms are the first asynchronous  federated learning algorithms for vertically partitioned data. Extensive experimental results on  a variety of vertically partitioned datasets not only verify the theoretical results of AFSGD-VP and its SVRG and SAGA variants, but also show that our algorithms have much higher efficiency than the corresponding synchronous algorithms.
We summarize the  main contributions of this paper as
follows.
\begin{enumerate}[leftmargin=0.2in]
\item We propose asynchronous federated stochastic gradient algorithm (\emph{i.e.}, AFSGD-VP) and its SVRG and SAGA variants for vertically partitioned data. We provide  their convergence rates  under  the
condition of strong convexity.

\item Based on the   \textit{semi-honest} assumption (\textit{i.e.}, Assumption \ref{ass_semi_honest}), we prove  that our AFSG-VP and its SVRG and SAGA variants can prevent the exact and approximate  inference attacks.

\end{enumerate}

%
%
%


\noindent \textbf{Notations.} In order to make  notations easier to follow, we give a summary of  notations in the
following list.
\begin{description}[leftmargin=4.4em,style=nextline]
	\item[$\widehat{w}$] $w$ that inconsistently  read from different workers.
\item[$\widetilde{w}$]  The snapshot of $w$ after a certain number  of iterations.
\item[$q$] The size of workers.
\item[$b^{\ell}$] A random number generated on the $\ell$-th  worker.
\item[$\xi(t,\ell)$]  The local time counter for the global time counter $t$ on the $\ell$-th worker.
\item[$\xi^{-1}(u,\ell)$]  The corresponding global time counter to a local  time counter $u$ on the $\ell$-th worker.
\item[$\psi(t)$]  The corresponding worker to obtain $\widehat{w}_t^T x_i$.
\item[$\psi^{-1}(\ell,K)$] All the elements in $K$ such that $\psi(\psi^{-1}(\ell,K))=\ell$.
\item[$Leaf(\cdot)$]  All leaves  of a tree.
\end{description}

\section{Asynchronous Federated Learning  for Vertically Partitioned Data}\label{sec2}
In this section, we first introduce the problem addressed in this paper, and then give a brief review of SGD, SVRG and SAGA. Next, we give the system structure of our asynchronous federated learning algorithms. Finally, we propose  our AFSGD-VP, AFSVRG-VP and AFSAGA-VP  algorithms.
\subsection{Problem Statement}

In this paper, we consider the model in a linear form of $w^T x$.  Given a training set $\mathcal{S}=\{(x_i,y_i)\}_{i=1}^{l}$, where $x_i \in \mathbb{R}^{d}$ and  $y_i\in \{+1,-1\}$  for  binary classification or $y_i\in \mathbb{R}$   for regression. The loss function w.r.t. the sample $(x_i,y_i) $ and the model weights $w$ can be formulated as $L(w^T x_i,y_i)$. Thus, we consider to optimize the following regularized empirical
risk minimization problem.
   \begin{eqnarray}\label{formulation1}
\min_{w\in \mathbb{R}^d} f(w)  = \frac{1}{l}\sum_{i=1}^{l} \underbrace{L(w^T x_i,y_i)+ g(w)}_{f_i(w)}\,,
\end{eqnarray}
where $g(w)$ is a regularization term, and each $f_i:\mathbb{R}^d \to \mathbb{R}$ is considered as  a smooth, possibly non-convex function  in this paper. Obviously, the empirical  risk minimization problem   is a special case of the problem (\ref{formulation1}). In addition to the empirical risk minimization problem,   problem (\ref{formulation1})  summarizes  an extensive number of important  regularized learning problems, such as, $\ell_2$-regularized logistic regression \cite{conroy2012fast}, ridge regression \cite{shen2013novel} and least squares SVM \cite{suykens1999least}.

As mentioned previously, in a lot of real-world machine learning applications,  the input of training sample $(x,y)$ is partitioned vertically into
$q$ parts, \emph{i.e.}, we have a partition $\{ \mathcal{G}_{1},\cdots,{\mathcal{G}_{q}} \}$ of $d$ features. Thus, we have $x = [x_{\mathcal{G}_1},x_{\mathcal{G}_2},\ldots,x_{\mathcal{G}_q}]$, where  $x_{\mathcal{G}_\ell} \in \mathbb{R}^{d_\ell}$ is
stored on the $\ell$-th worker,  and $\sum_{\ell=1}^q d_\ell=d$. According to whether the label is included in a worker, we divide the workers into two types: one is active worker and the other is passive worker, where the active worker is the data provider who holds the label of a sample, and the passive worker only has the input of a sample. The active worker would be a dominating server in federated learning, while passive workers play the role of clients \cite{cheng2019secureboost}.  We let  $D^\ell$  denote the data stored on the $\ell$-th worker. Note that the  labels $y_i$ are distributed
on active workers. Our goal in this paper can be presented as follows.
\begin{mdframed}
 \begin{myitemize}
\noindent \textbf{Goal:} Make active workers to cooperate   with passive workers to  solve the regularized empirical
risk minimization  problem (\ref{formulation1}) on the vertically partitioned data $\{ D^\ell\}_{\ell=1}^q$   in parallel and asynchronously with  the SGD and  its SVRG and SAGA variants, while keeping the vertically partitioned data private.
 \end{myitemize}
\end{mdframed}

\subsection{Brief Review of SGD, SVRG and SAGA}\label{sec2.0}
 As mentioned before, SGD-like algorithms    have been the popular  algorithms for solving large-scale   machine learning problems. We first give a brief review of the update framework of  SGD-like algorithms which include multiple  variants of variance reduction methods. Specifically, given an unbiased stochastic gradient ${v}$ (\emph{i.e.}, $\mathbb{E} {v} = \nabla f(w)$), the updating rule  of  SGD-like algorithms can be formulated as follows.
  \begin{eqnarray}
  w \leftarrow w - \gamma v
\end{eqnarray}
 where $\gamma$ is the learning rate. In the following, we present the specific forms to the unbiased stochastic gradient ${v}$ w.r.t. SGD, SVRG and SAGA.

\noindent \textbf{SGD:} At each iteration SGD \cite{bottou2010large} independently samples a sample $(x_i,y_i)$, and uses the stochastic gradient $\nabla f_i(w)$
  with respect to the sampled sample $(x_i,y_i)$ to update the solution  as follows.
\begin{eqnarray}
v =\nabla f_i(w)
\end{eqnarray}

\noindent \textbf{SVRG:} For SVRG \cite{xiao2014proximal,gu2018faster}, instead of directly using the stochastic gradient $\nabla f_i(w)$, they use an unbiased stochastic gradient ${v}$  as follows to update the solution.
\begin{eqnarray}
v=\nabla f_i(w)-\nabla f_i(\widetilde{w}) + \nabla f(\widetilde{w})
\end{eqnarray}
where $\widetilde{w}$ denotes snapshot of $w$ after
a certain number  of iterations.

\noindent \textbf{SAGA:} For SAGA \cite{defazio2014saga},  the  unbiased stochastic gradient ${v}$ is formulated as follows.
\begin{eqnarray}
v=\nabla f_i(w)- \alpha_i + \frac{1}{l} \sum_{i=1}^l \alpha_i
\end{eqnarray}
where $\alpha_i$ is  the  latest historical gradient of $\nabla f_i(w)$.
 which can be updated  in an online fashion.

\subsection{System Structure of Our Algorithms}
As mentioned before,  AFSG-VP, AFSVRG-VP and
AFSAGA-VP  are privacy-preserving asynchronous federated learning algorithms on the vertically partitioned data.  Figure \ref{structureFDSKL} presents their  system structure.  Specifically, we  give  detailed descriptions of tree-structured
communication, and  data and model privacy, respectively, as follows.
 \begin{figure}[h]
\center
\hspace*{-0.6cm}
\includegraphics[scale=0.27]{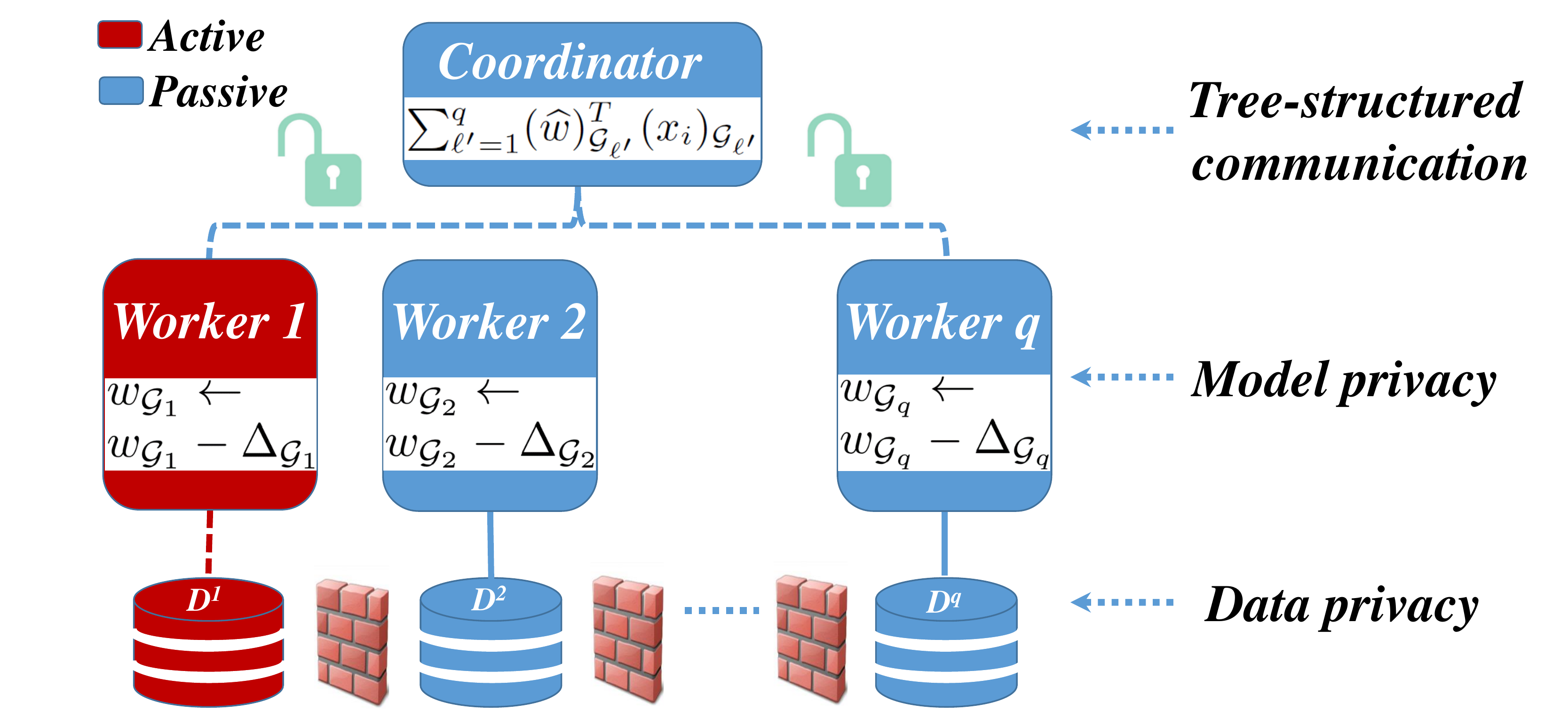}
 \caption{System structure of our  privacy-preserving asynchronous federated learning algorithms.}
 \label{structureFDSKL}
\end{figure}

\subsubsection{Tree-Structured Communication}  To   obtain $w^T x_i$,  we need to accumulate the local results from different workers. Zhang et al. \cite{zhang2018feature} proposed an efficient  tree-structured communication  scheme to
get the global sum which is faster than the simple strategy of  sending the
results from all workers directly to the coordinator for sum. Take 4 workers as an example, we pair the workers
so that while worker 1 adds the result from worker 2,
worker 3 can add the result from worker 4 simultaneously. Finally, the results from  the two pairs of workers are sent to the coordinator and we  obtain the global sum (please see Figure \ref{FigureTSC1}).
 In this paper, we use the tree-structured communication  scheme to obtain $w^T x_i$. Note that, our tree-structured communication  scheme works with the  asynchronous pattern  to obtain $w^T x_i$, that means that we do not align the iteration numbers of  $w_{\mathcal{G}_\ell}$ from different workers to compute $w^T x_i$. It is significantly different from the  synchronous pattern  used in \cite{zhang2018feature} where all $w_{\mathcal{G}_\ell}$ have   one and the same iteration number.

 Based on the tree-structured communication  scheme, we summarize the basic algorithm of computing $\sum_{\ell'=1}^q w_{\mathcal{G}_{\ell'}}^T (x_i)_{\mathcal{G}_{\ell'}}$  on the   $\ell$-th active worker in Algorithm \ref{protocol1}.
 \begin{algorithm}[!ht]
\renewcommand{\algorithmicrequire}{\textbf{Input:}}
\renewcommand{\algorithmicensure}{\textbf{Output:}}
\caption{Basic algorithm of computing $\sum_{\ell'=1}^q w_{\mathcal{G}_{\ell'}}^T (x_i)_{\mathcal{G}_{\ell'}}$  on the   $\ell$-th active worker} \label{protocol1}
\begin{algorithmic}[1]
\REQUIRE $w$, $x_i$
\\
\COMMENT{This loop asks multiple workers running in parallel.}
\FOR{$\ell'=1,\ldots,q$} \label{step1}
\STATE \label{step3} Calculate $ w_{\mathcal{G}_{\ell'}}^T (x_i)_{\mathcal{G}_{\ell'}} $.
\ENDFOR \label{step4}
\STATE \label{step5} Use  tree-structured communication scheme  to compute $\xi =\sum_{\ell'=1}^q  w_{\mathcal{G}_{\ell'}}^T (x_i)_{\mathcal{G}_{\ell'}} $.

\ENSURE $\xi$.
\end{algorithmic}
\end{algorithm}

\begin{figure}[htbp!]
\centering
 \hspace*{-0.9cm}
 	\begin{subfigure}[b]{0.4\textwidth}
		\includegraphics[height=1.3in]{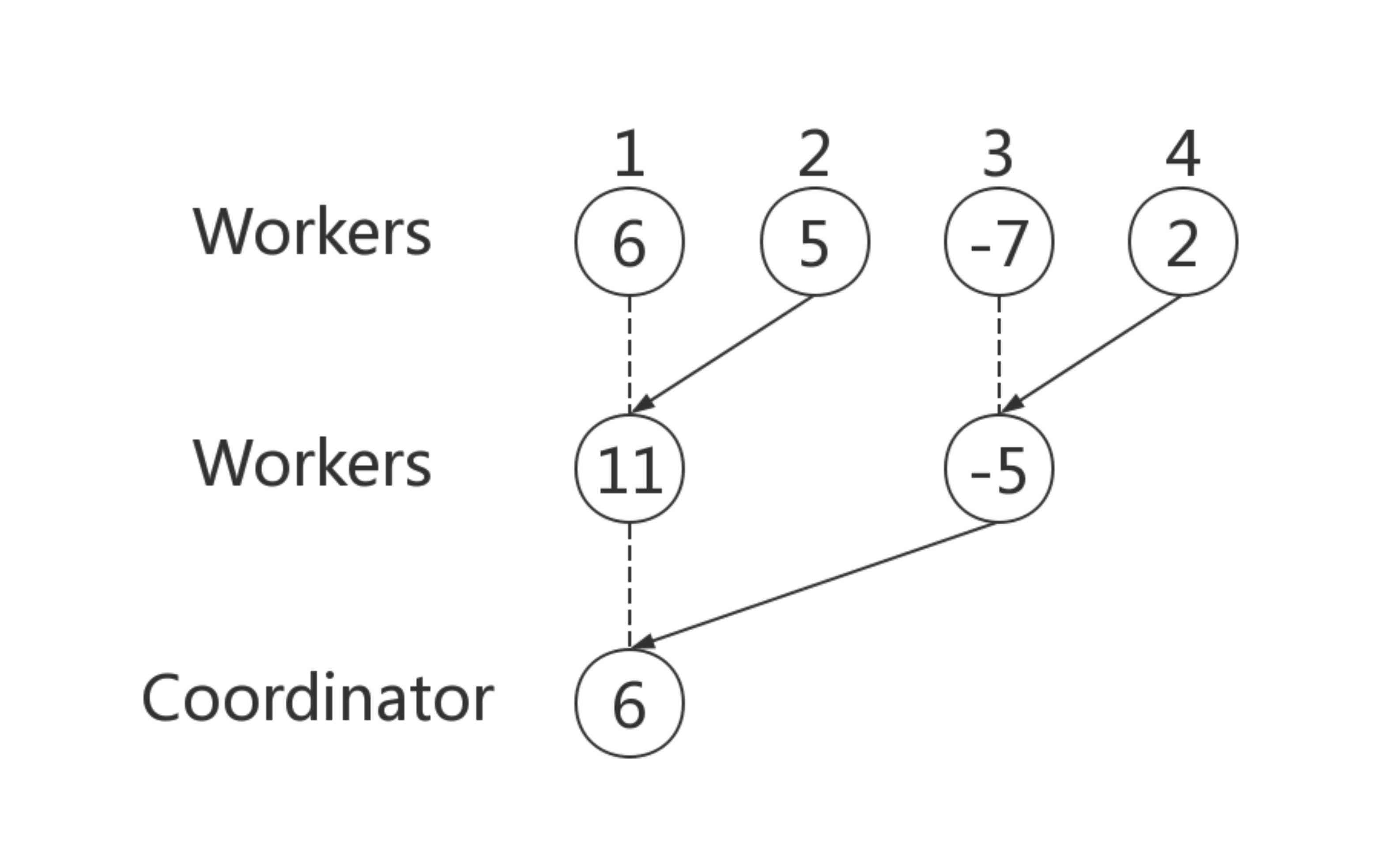}
\caption{Tree structures $T_1$}
\label{FigureTSC1}
	\end{subfigure}
 	\begin{subfigure}[b]{0.4\textwidth}
		\includegraphics[height=1.3in]{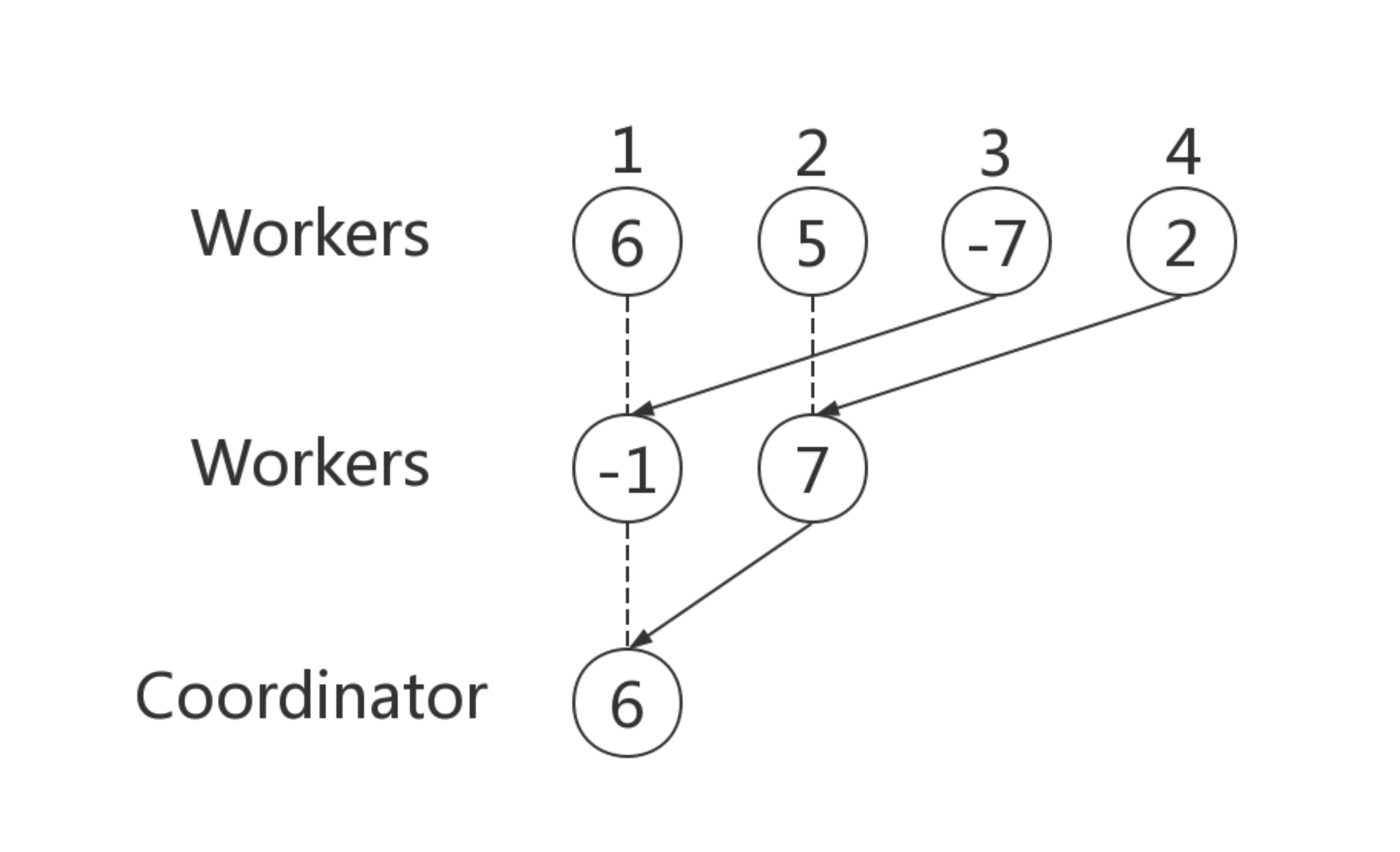}
\caption{Tree structures $T_2$}
\label{FigureTSC2}
	\end{subfigure}
\caption{Illustration of  tree-structured communication with two significantly different tree structures.}
\label{FigureTSC}
\end{figure}


\subsubsection{Data and Model  Privacy}   To keep the vertically partitioned  data and model privacy,  we  save the  data $(x_i)_{\mathcal{G}_\ell}$ and  model weights  $w_{\mathcal{G}_\ell}$  in the $\ell$-th worker separately and  privately. We do not directly  transfer the local data $(x_i)_{\mathcal{G}_\ell}$ and local model weights $w_{\mathcal{G}_\ell}$ to other workers. To obtain $w^T x_i$, we   locally compute $w_{\mathcal{G}_\ell}^T (x_i)_{\mathcal{G}_\ell} $  and  only transfer $w_{\mathcal{G}_\ell}^T (x_i)_{\mathcal{G}_\ell}$  to other workers for computing $w^T x$ as shown in Algorithm \ref{protocol1}. It is  not trivial to infer the the local model coefficients $w_{\mathcal{G}_\ell} $ and $(x_i)_{\mathcal{G}_\ell}$ based on the value of  $w_{\mathcal{G}_\ell}^T (x_i)_{\mathcal{G}_\ell}$ which is discussed in detail in Section \ref{section_sec}. Thus, we achieve the data and model  privacy.


 Although it is not trivial to exactly infer the the local model coefficients $w_{\mathcal{G}_\ell} $ and $(x_i)_{\mathcal{G}_\ell}$ based on the value of  $w_{\mathcal{G}_\ell}^T (x_i)_{\mathcal{G}_\ell}$, it  has the risk of approximate inference attack (please refer to Definition \ref{defi_inf_attack2}). To address this issue, we propose a safer algorithm to compute $\sum_{\ell'=1}^q w_{\mathcal{G}_{\ell'}}^T (x_i)_{\mathcal{G}_{\ell'}}$ in Algorithm \ref{protocol2}. Specifically, we add a random number $b^{\ell'}$ into $w_{\mathcal{G}_{\ell'}}^T (x_i)_{\mathcal{G}_{\ell'}}$, and then use the tree-structured communication scheme on a tree structure $T_1$  to compute $\sum_{\ell'=1}^q  ( w_{\mathcal{G}_{\ell'}}^T (x_i)_{\mathcal{G}_{\ell'}} + b^{\ell'}  )$ which can improve the  data and model security for the operation of transferring the value of $w_{\mathcal{G}_{\ell'}}^T (x_i)_{\mathcal{G}_{\ell'}} + b^{\ell'}$. Finally, we need to recover the value of $\sum_{\ell'=1}^q w_{\mathcal{G}_{\ell'}}^T (x_i)_{\mathcal{G}_{\ell'}}$ from $\sum_{\ell'=1}^q \left ( w_{\mathcal{G}_{\ell'}}^T (x_i)_{\mathcal{G}_{\ell'}} + b^{\ell'} \right )$.   In order to prevent leaking any sum of $b^{\ell'}$ of a subtree of $T_1$, we use a \textit{significantly different tree structure $T_2$} on all workers    (please see Definition \ref{def_total_differ_tree} and Figure \ref{FigureTSC}) to compute $\overline{b} =\sum_{\ell'=1}^q  b^{{\ell'}}$.
\begin{definition}[Two significantly different tree structures]\label{def_total_differ_tree} For two tree structures $T_1$ and $T_2$ on all workers  $\{1, \ldots,q\}$, they are  significantly different if there does not exist a subtree $\widehat{T}_1$  of $T_1$ and a  subtree $\widehat{T}_2$ of $T_2$ whose sizes are larger than 1 and smaller than $T_1$ and $T_1$ respectively, such that $Leaf(\widehat{T}_1)=Leaf(\widehat{T}_2)$. 
\end{definition}

 \begin{algorithm}[!ht]
\renewcommand{\algorithmicrequire}{\textbf{Input:}}
\renewcommand{\algorithmicensure}{\textbf{Output:}}
\caption{Safer algorithm of computing $\sum_{\ell'=1}^q w_{\mathcal{G}_{\ell'}}^T (x_i)_{\mathcal{G}_{\ell'}}$  on the   $\ell$-th active worker} \label{protocol2}
\begin{algorithmic}[1]
\REQUIRE $w$, $x_i$
\\
\COMMENT{This loop asks multiple workers running in parallel.}
\FOR{$\ell'=1,\ldots,q$} \label{step1}
\STATE \label{step2} Generate a random number $b^{\ell'}$.
\STATE \label{step3} Calculate $ w_{\mathcal{G}_{\ell'}}^T (x_i)_{\mathcal{G}_{\ell'}} + b^{\ell'}$.
\ENDFOR \label{step4}
\STATE \label{step5} Use  tree-structured communication scheme based on the tree structure $T_1$ on all workers $\{1, \ldots,q\}$ to compute $\xi =\sum_{\ell'=1}^q \left ( w_{\mathcal{G}_{\ell'}}^T (x_i)_{\mathcal{G}_{\ell'}} + b^{\ell'} \right )$.

\STATE \label{step7} Use  tree-structured communication scheme based on the totally  different tree structure $T_2$ on all workers $\{1, \ldots,q\}$  to compute $\overline{b} =\sum_{\ell'=1}^q  b^{{\ell'}}$.
\ENSURE $\xi-\overline{b}$.
\end{algorithmic}
\end{algorithm}
\subsection{Algorithms}\label{sec2.1}
In this subsection, we propose  our three asynchronous federated stochastic gradient algorithms   (\emph{i.e.},  AFSG-VP, AFSVRG-VP and AFSAGA-VP) on the vertically partitioned data.
\subsubsection{AFSGD-VP}
 AFSGD-VP repeats the following four steps  concurrently for each  worker without any lock.
\begin{enumerate}[leftmargin=0.2in]
\item \textbf{Pick up an index:} AFSGD-VP  picks up an index $i$ randomly from $\{1,\ldots,l\}$ and obtain the local instance $(x_i)_{\mathcal{G}_\ell}$  from the local data $D^\ell$.
\item \textbf{Compute $\widehat{w}^T x_i$:}  AFSGD-VP   uses the tree-structured communication  scheme with  asynchronous pattern (\emph{i.e.}, Algorithm \ref{protocol1} or \ref{protocol2}) to obtain $\widehat{w}^T x_i=\sum_{\ell'=1}^q (\widehat{w})_{\mathcal{G}_{\ell'}}^T (x_i)_{\mathcal{G}_{\ell'}}$, where $\widehat{w}$ denotes $w$ inconsistently  read from different workers and two $(\widehat{w})_{\mathcal{G}_{\ell'}}$ may be in different local iteration stages.  Note that we always have that $(w)_{\mathcal{G}_\ell}=(\widehat{w})_{\mathcal{G}_\ell}$.
\item  \textbf{Compute  stochastic local gradient:}  Based on $\widehat{w}^T x_i$, we can compute the  unbiased stochastic local
gradient as $\widehat{v}^\ell = \nabla_{\mathcal{G}_\ell} f_i (\widehat{w})$.
\item \textbf{Update:} AFSGD-VP updates the  local model
weights $w_{\mathcal{G}_\ell}$ by $w_{\mathcal{G}_\ell} \leftarrow w_{\mathcal{G}_\ell}  - \gamma \cdot  \widehat{v}^\ell$, where ${\gamma}$ is the learning rate.
\end{enumerate}
 We summarize our AFSGD-VP  algorithm in Algorithm \ref{algorithm1}.

\begin{algorithm}[htbp]
\renewcommand{\algorithmicrequire}{\textbf{Input:}}
\renewcommand{\algorithmicensure}{\textbf{Output:}}
\renewcommand{\algorithmicloop}{\textbf{Keep doing in parallel}}
\renewcommand{\algorithmicendloop}{\textbf{End parallel loop}}
\caption{Asynchronous  federated SGD algorithm  (AFSGD-VP) for vertically partitioned data on the  $\ell$-th active worker}
	\begin{algorithmic}[1] 
\REQUIRE Local data $D^\ell$, learning rate $ \gamma$.

\STATE Initialize $w_{\mathcal{G}_\ell} \in \mathbb{R}^{d_\ell}$.
 \LOOP
		\STATE Pick up an index $i$ randomly from $\{1,\ldots,l\}$ and obtain the local instance $(x_i)_{\mathcal{G}_\ell}$  from the local data $D^\ell$.

    \STATE Compute $(w)_{\mathcal{G}_\ell}^T (x_i)_{\mathcal{G}_\ell} $.
   \STATE  Compute $\widehat{w}^T x_i= \sum_{\ell'=1}^q (\widehat{w})_{\mathcal{G}_{\ell'}}^T (x_i)_{\mathcal{G}_{\ell'}}$ based on Algorithm \ref{protocol1} or \ref{protocol2}.
\STATE Compute $\widehat{v}^\ell = \nabla_{\mathcal{G}_\ell} f_i (\widehat{w})$.
\STATE Update $w_{\mathcal{G}_\ell} \leftarrow w_{\mathcal{G}_\ell}  - \gamma \cdot  \widehat{v}^\ell$.

\ENDLOOP
\ENSURE $w_{\mathcal{G}_\ell}$
\end{algorithmic}
\label{algorithm1}
\end{algorithm}

\subsubsection{AFSVRG-VP}
Stochastic gradients in AFSGD-VP have a  large variance due to the  random sampling similar to  SGD algorithm
\cite{bottou2010large}. To handle the large variance, AFSVRG-VP uses the SVRG  technique \cite{gu2018faster} to  reduce the variance of the stochastic gradient, and  propose a faster AFSGD-VP algorithm (\emph{i.e.}, AFSVRG-VP). We summarize our AFSVRG-VP algorithm in  Algorithm \ref{algorithm2}.  Compared to  AFSGD-VP, AFSVRG-VP has the following three differences.
\begin{enumerate}[leftmargin=0.2in]
\item  The first one is that AFSVRG-VP  is to  compute the  full local gradient $\nabla_{\mathcal{G}_\ell} f(w^s)= \frac{1}{l} \sum_{i=1}^l \nabla_{\mathcal{G}_\ell} f_i(w^s)$ in the outer loop which will be used as the snapshot of full gradient, where the superscript $s$ denotes the $s$-th out loop.
\item The second one is that we  compute not only $\widehat{w}^T x_i$  but also  $(w^s)^T x_i$ for each iteration.
\item The third one is that AFSVRG-VP  computes the  unbiased stochastic local gradient as
   $\widehat{v}^\ell =  \nabla_{\mathcal{G}_\ell} f_i (\widehat{w}) - \nabla_{\mathcal{G}_\ell} f_i (w^s)
  +  \nabla_{\mathcal{G}_\ell} f(w^s)  $.
\end{enumerate}
  \begin{algorithm}[htbp]
\renewcommand{\algorithmicrequire}{\textbf{Input:}}
\renewcommand{\algorithmicensure}{\textbf{Output:}}
\renewcommand{\algorithmicloop}{\textbf{Keep doing in parallel}}
\renewcommand{\algorithmicendloop}{\textbf{End parallel loop}}
\caption{Asynchronous   federated SVRG algorithm  (AFSVRG-VP) for vertically partitioned data on the  $\ell$-th active worker}
	\begin{algorithmic}[1] 
\REQUIRE Local data $D^\ell$, learning rate $ \gamma$.

\STATE Initialize $w_{\mathcal{G}_\ell}^0 \in \mathbb{R}^{d_\ell}$.
\FOR{$s=0,1,2,\cdots,S-1$}
\STATE Compute the full local gradient $\nabla_{\mathcal{G}_\ell} f(w^s)= \frac{1}{l} \sum_{i=1}^l \nabla_{\mathcal{G}_\ell} f_i(w^s)$ by using  tree-structured communication scheme.
\STATE ${w}_{\mathcal{G}_\ell} = w^{s}_{\mathcal{G}_\ell} $.
 \LOOP
 \STATE Pick up a local instance $(x_i)_{\mathcal{G}_\ell}$ randomly from the local data $D^\ell$.

    \STATE Compute $(w)_{\mathcal{G}_\ell}^T (x_i)_{\mathcal{G}_\ell} $ and $(w^s)_{\mathcal{G}_\ell}^T (x_i)_{\mathcal{G}_\ell} $.
   \STATE  Compute $\widehat{w}^T x_i= \sum_{\ell'=1}^q (\widehat{w})_{\mathcal{G}_{\ell'}}^T (x_i)_{\mathcal{G}_{\ell'}}$ and $(w^s)^T x_i= \sum_{\ell'=1}^q (w^s)_{\mathcal{G}_{\ell'}}^T (x_i)_{\mathcal{G}_{\ell'}}$ based on Algorithm \ref{protocol1} or \ref{protocol2}.
		
\STATE Compute $\widehat{v}^\ell =  \nabla_{\mathcal{G}_\ell} f_i (\widehat{w}) - \nabla_{\mathcal{G}_\ell} f_i (w^s)
  +  \nabla_{\mathcal{G}_\ell} f(w^s)  $.

\STATE Update $w_{\mathcal{G}_\ell} \leftarrow w_{\mathcal{G}_\ell}  - \gamma \cdot \widehat{v}^\ell $.

\ENDLOOP
\STATE $w^{s+1}_{\mathcal{G}_\ell} =   w_{\mathcal{G}_\ell}$.
\ENDFOR
\ENSURE $w_{\mathcal{G}_\ell}$
\end{algorithmic}
\label{algorithm2}
\end{algorithm}

\subsubsection{AFSAGA-VP}
As mentioned above, the stochastic gradients in SGD have a  large variance due to the  random sampling. To handle the large variance,  AFSAGA-VP uses the SAGA  technique \cite{defazio2014saga} to  reduce the variance of the stochastic gradients. We summarize our AFSAGA-VP algorithm in  Algorithm \ref{algorithm1}. Specifically, we  maintain a table of   latest historical local gradients $\alpha_i^\ell$ which is achieved by the updating rule of $\widehat{\alpha}_i^\ell \leftarrow \nabla_{\mathcal{G}_\ell} f_i (w)$ for each iteration. Based on the table of   latest historical local gradients $\widehat{\alpha}_i^\ell$,  the  unbiased stochastic local gradient in  AFSAGA-VP is computed as $\widehat{v}^\ell = \nabla_{\mathcal{G}_\ell} f_i (\widehat{w}) - \widehat{\alpha}_i^\ell +  \frac{1}{l} \sum_{i=1}^l \widehat{\alpha}_i^\ell$.
\begin{algorithm}[ht]
\renewcommand{\algorithmicrequire}{\textbf{Input:}}
\renewcommand{\algorithmicensure}{\textbf{Output:}}
\renewcommand{\algorithmicloop}{\textbf{Keep doing in parallel}}
\renewcommand{\algorithmicendloop}{\textbf{End parallel loop}}
\caption{Asynchronous  federated SAGA algorithm  (AFSAGA-VP) for vertically partitioned data on the  $\ell$-th active worker}
	\begin{algorithmic}[1] 
\REQUIRE Local data $D^\ell$, learning rate $ \gamma$.

\STATE Initialize $w_{\mathcal{G}_\ell} \in \mathbb{R}^{d_\ell}$.

\STATE Compute the local  gradients ${\alpha}_i^\ell = \nabla_{\mathcal{G}_\ell} f_i({w})$, $\forall i\in \{1,\ldots,n \}$ by using  tree-structured communication scheme, and locally save them.

 \LOOP
\STATE Pick up a local instance $(x_i)_{\mathcal{G}_\ell}$ randomly from the local data $D^\ell$.

    \STATE Compute $(w)_{\mathcal{G}_\ell}^T (x_i)_{\mathcal{G}_\ell} $.
   \STATE  Compute $\widehat{w}^T x_i = \sum_{\ell=1}^q (\widehat{w})_{\mathcal{G}_\ell}^T (x_i)_{\mathcal{G}_\ell}$ based on Algorithm \ref{protocol1} or \ref{protocol2}.
		
\STATE Compute $\widehat{v}^\ell = \nabla_{\mathcal{G}_\ell} f_i (\widehat{w}) - \widehat{\alpha}_i^\ell +  \frac{1}{l} \sum_{i=1}^l \widehat{\alpha}_i^\ell$.

\STATE Update $w_{\mathcal{G}_\ell} \leftarrow w_{\mathcal{G}_\ell}  - \gamma \cdot \widehat{v}^\ell $.
\STATE Update $\widehat{\alpha}_i^\ell \leftarrow \nabla_{\mathcal{G}_\ell} f_i (\widehat{w})$.

\ENDLOOP

\ENSURE $ w_{\mathcal{G}_\ell}$.
\end{algorithmic}
\label{algorithm3}
\end{algorithm}
\section{Theoretical Analyses}\label{sec3}
In this section, we  provide the convergence, security and complexity analyses to AFSG-VP, AFSVRG-VP and
AFSAGA-VP.   All the  proofs can be found in the Appendix.
%
\subsection{Convergence Analyses}
We first make several basic  assumptions,  then  provide  the results of convergence  of   AFSG-VP, AFSVRG-VP and
AFSAGA-VP.

\subsubsection{Preliminaries}
In this part, we  give  the assumptions of  strong convexity (Assumption \ref{assumption4}),  different Lipschitz smoothness (Assumption \ref{assumption3}) and  block-coordinate bounded gradients (Assumption \ref{assumption1}), which are standard for convex analysis \cite{defazio2014saga,xiao2014proximal,zhao2016fast,siamjo-BeckT13,LiZALH17,LiZALH16}.

\begin{assumption}[Strong convexity]\label{assumption4}
The differentiable  function $f_i$ ($\forall i \in \{1,\cdots,l\}$ in the problem  (\ref{formulation1}) is strongly convex with parameter $\mu > 0$, which means that     $\forall w$ and $\forall w'$,
we have
\begin{eqnarray} \label{Strong_convex}
 f_i(w)
\geq   f_i(w') + \langle \nabla f_i(w') , w- w' \rangle + \frac{\mu}{2}  \left \| w - w' \right \|^2
\end{eqnarray}
\end{assumption}

%
%
%


\begin{assumption}[Lipschitz smoothness]\label{assumption3}
The   function $f_i$ ($\forall i \in \{1,\cdots,l\}$ in the problem  (\ref{formulation1}) is Lipschitz smooth with   constant $L$, which means that,  $\forall w$ and $\forall w'$, we have:
   \begin{eqnarray}\label{equdef1}
\| \nabla f_i(w) - \nabla f_i(w') \|  \leq L \|w - w' \|\,.
\end{eqnarray}
The function $f_i$ ($\forall i \in \{1,\cdots,l\}$ in the problem  (\ref{formulation1}) is block-coordinate Lipschitz smooth w.r.t. the $\ell$-th block $\mathcal{G}_{\ell}$ with  constant  $L_{\ell} $, such that, $\forall w$, and $\forall \ell \in \{ 1,\cdots,q\}$,  we have:
   \begin{eqnarray}\label{equdef2}
\| \nabla_{\mathcal{G}_\ell} f_i(w+ \textbf{U}_\ell \Delta_\ell) - \nabla_{\mathcal{G}_\ell} f_i(w) \|  \leq L_{\ell} \left \| \Delta_\ell \right \|\,.
\end{eqnarray}
where $\Delta_\ell \in \mathbb{R}^{d_\ell}$, $\textbf{U}_{\ell} \in \mathbb{R}^{d\times d_\ell}$ and $[\textbf{U}_1,\textbf{U}_2,\ldots,\textbf{U}_q]=\textbf{I}_d$.
\end{assumption}
According to the definition of block-coordinate Lipschitz smooth constant $L_{\ell} $ in Assumption \ref{assumption3},  we define $L_{\max}=\max_{\ell=1,\ldots,q} L_\ell$. Furthermore, we have $L \leq q L_{\max}$ which is proved in Lemma 2 of \cite{nesterov2012efficiency}.

\begin{assumption}[Block-coordinate bounded gradients] 
\label{assumption1}
For smooth function $f_i(x)$ ($\forall i \in \{1,\ldots,l\}$) in (\ref{formulation1}),   the  block-coordinate gradient $\nabla_{\mathcal{G}_\ell} f_i ({w})$ is called  bounded if there exists  a parameter $G$ such that $
\| \nabla_{\mathcal{G}_\ell} f_i ({w})\|^2 \leq G $, $\forall i \in \{1,\ldots,l\}$ and $\forall \ell \in \{1,\ldots,q\}$.
\end{assumption}
\subsubsection{Difficulties}
In this part, we discuss the difficulties of globally labeling the iterates, global updating rules and the relationship between $w_t$ and $\widehat{w}_t$.


\noindent \textbf{Globally labeling the iterates:} \ As shown in Algorithms \ref{algorithm1} and \ref{algorithm2}, we do not globally label the iterates from different workers. Although it is fine for the implementation, how we choose to define the iteration counter $t$ to label an iterate
$w_t$ matters in the analysis.
More specifically, the global time
counter  plays a fundamental role in the  convergence rate analyses of   AFSG-VP, AFSVRG-VP and
AFSAGA-VP.
To address this issue, we propose the strategy of ``after communication'' labeling \cite{leblond2017asaga}, in which we update our iterate counter  as one worker finishes computing $\widehat{w}^T x_i$.
This means that $\widehat{w}_{t}$ (or $\widehat{w}_{t}^s$) is the $(t+1)$-th fully completed the computation of $\widehat{w}^T x_i$. The  strategy of ``after communication'' labeling guarantees both that the $i_t$ are uniformly distributed and that $i_t$ and $\widehat{w}_t$ are
independent. 

We define  a minimum set of successive iterations of fully visiting all coordinates from the time counter $t$ as $K(t)$ in Definition \ref{minimum_set}.
\begin{definition}[Set  $K(t)$]\label{minimum_set}
Let $\overline{K}(t)= \{ \{t,t+1,\ldots,t+ \sigma \}:  \psi(\{t,t+1,\ldots,t+ \sigma \})=\{1,\ldots,q\} \}$. The minimum set of successive iterations of fully visiting all coordinates from the time counter $t$ is defined as $K(t)=\argmin_{K'(t) \in \overline{K}(t)} |K'(t)|$.
\end{definition}
Let $\psi^{-1}(\ell,K)$ denote  all the elements in $K$ such that $\psi(\psi^{-1}(\ell,K))=\ell$. We assume that there exists an upper bound $\eta_1$ to the size of  $\psi^{-1}(\ell,K(t))$ (Assumption \ref{assMax}).
\begin{assumption}[Bounded size of $\psi^{-1}(\ell,K(t))$]\label{assMax}
 $\forall t$, and  $\forall \ell \in \{1,\ldots,q\}$, the sizes of all $\psi^{-1}(\ell,K(t))$ are upper bounded by $\eta_1$, \emph{i.e.}, $|\psi^{-1}(\ell,K(t))| \leq \eta_1$.
\end{assumption}

Based on the definition of $K(t)$, we define the  epoch number of fully visiting all coordinates for the global $t$-th iteration as  $\upsilon(t)$, and the start start time  counter  in one epoch as $\varphi(t)$  in Definition \ref{epochNumber}. Our convergence rate analyses are build on the epoch number $\upsilon(t)$.
\begin{definition}[Epoch number $\upsilon(t)$ and start time  counter $\varphi(t)$]\label{epochNumber}
Let $P(t)$ is a partition of  $\{0,1,\ldots,t-\sigma' \}$,  where $\sigma' \geq 0$. For  any $\kappa \in P(t)$ we have that, there exists $t'\leq t$ such that $K(t')=\kappa$, and there exists $\kappa_1 \in P(t)$ such that $K(0)=\kappa_1$. The epoch number $\upsilon(t)$ is defined as the maximum  cardinality  of $P(t)$. Given a  global  time counter $u\leq t$, if there exists $\kappa \in P(t)$ such that $u\in \kappa$, we define the start time counter $\varphi(t)$ as the minimum element of $\kappa$, otherwise $\varphi(t)=t-\sigma'+1$.
\end{definition}

\noindent \textbf{Global updating rule:} \  The updating rules (such as $w_{\mathcal{G}_\ell} \leftarrow w_{\mathcal{G}_\ell}  - \gamma \cdot \widehat{v}^\ell $) in Algorithms \ref{algorithm1}, \ref{algorithm2} and \ref{algorithm3} are  updating rules locally working on a certain worker. To provide the convergence rate analyses of AFSG-VP, AFSVRG-VP and
AFSAGA-VP, we need provide the global updating rules of AFSG-VP, AFSVRG-VP and
AFSAGA-VP. Due to the commutativity of the add operations used in $w_{\mathcal{G}_\ell} \leftarrow w_{\mathcal{G}_\ell}  - \gamma \cdot \widehat{v}^\ell $,  the order in which these
updates are finished in the corresponding worker is irrelevant. Hence,  we provide the  global updating rules of AFSG-VP, AFSVRG-VP and
AFSAGA-VP as follows.
\begin{eqnarray}\label{EqGlobalUpdate}
{w}_{t+1} ={w}_{t} - \gamma  \textbf{U}_{\psi(t)} \widehat{v}^{\psi(t)}_t
\end{eqnarray}
Note that the global updating rule (\ref{EqGlobalUpdate}) which defines the relation of two  adjacent iterates,  does not conflict with the rule of globally labeling the iterates  due to the  commutativity of the add operations.

\noindent \textbf{Relationship between $w_t$ and $\widehat{w}_t$:} \
As mentioned before,  AFSG-VP, AFSVRG-VP and
AFSAGA-VP use the tree-structured communication  scheme with  asynchronous pattern to obtain $\widehat{w}^T x_i=\sum_{\ell'=1}^q (\widehat{w})_{\mathcal{G}_{\ell'}}^T (x_i)_{\mathcal{G}_{\ell'}}$, where $\widehat{w}$ denotes $w$ inconsistently  read from different workers. Thus, the vector $(\widehat{w}_t)_{\mathcal{G}_{\ell'}}$ for $\ell' \neq \ell$  may be inconsistent to the vector $({w}_t)_{\mathcal{G}_{\ell'}}$, which means that some blocks of $\widehat{w}_t$ are same with the ones in ${w}_t$ (e.g., $(w)_{\mathcal{G}_\ell}=(\widehat{w})_{\mathcal{G}_\ell}$), but others are different to the ones in  ${w}_t$. To address the challenge,   we assume an upper bound to the delay of updating. Specifically, we  define a set $D(t)$ of   iterations, such that:
\begin{eqnarray} \label{pi_test1}
\widehat{w}_t-{w}_t = \gamma \sum_{u \in D(t)}  \textbf{U}_{\psi(u)} \widehat{v}^{\psi(u)}_u\,,
\end{eqnarray}
where $\forall  u\in D(t)  $, we have $u<t$. It is reasonable to assume that there exists an upper bound $\tau$ such that $\tau \geq t - \min\{t' | t' \in D(t)\}$ (\emph{i.e.}, Assumption \ref{ass4}).
\begin{assumption}[Bounded  overlap]\label{ass4}
There exists an upper   bound $\tau$ such that $\tau \geq t - \min\{u | u \in D(t)\}$  for all  iterations $t$ in AFSG-VP, AFSVRG-VP and
AFSAGA-VP.
 \end{assumption}
In addition, we assume that there exist an upper bound $\eta_2$ to the size of  $\psi^{-1}(\ell,D(t))$ (Assumption \ref{ass5}).
\begin{assumption}[Bounded  size of $\psi^{-1}(\ell,D(t))$]\label{ass5}
 $\forall t$, and  $\forall \ell \in \{1,\ldots,q\}$, the sizes of all $\psi(\psi^{-1}(\ell,D(t)))$ are upper bounded by $\eta_2$, \emph{i.e.}, $|\psi(\psi^{-1}(\ell,K))| \leq \eta_2$.
 \end{assumption}

\subsubsection{AFSGD-VP}
We provide the convergence result of AFSGD-VP in Theorem \ref{theorem1}.
\begin{theorem} \label{theorem1}
Under Assumptions \ref{assumption4}-\ref{ass5}, to achieve the accuracy $\epsilon$ of (\ref{formulation1}) for AFSGD-VP, \emph{i.e.}, $\mathbb{E} f (w_{t}) -f(w^*) \leq \epsilon$, we set  \begin{eqnarray}\gamma =  \frac{-   L_{\max}   + \sqrt{  L_{\max}^2  + \frac{{2 \mu \epsilon  (L^2 q \eta_1^2  + \eta_2  L^2 \tau )}}{G\eta_1   q}} }{2 L^2 (q \eta_1^2  + \eta_2   \tau )}\end{eqnarray}
 and the epoch  number $\upsilon(t)$ should satisfy the following condition.
\begin{eqnarray}\label{equ_theorem2_o.1}
 \upsilon(t) \geq \frac{2}{ \mu } \frac{2 L^2 (q \eta_1^2  + \eta_2   \tau )}{-   L_{\max}   + \sqrt{  L_{\max}^2  + \frac{{2 \mu \epsilon  (L^2 q \eta_1^2  + \eta_2  L^2 \tau )}}{G\eta_1   q}} } \cdot
\log\left ( \frac{2 \left (  f(w_0)-f(w^*) \right )}{\epsilon} \right )
\end{eqnarray}
\end{theorem}
\begin{remark}
Theorem \ref{theorem2} shows that, the convergence rate of AFSGD-VP is $\mathcal{O}(  \frac{1}{\sqrt{\epsilon}}\log\left ( \frac{1}{\epsilon}  \right ) )$ to reach the  accuracy $\epsilon$. The theorem shows that if we try to obtain a more accurate solution
with a smaller stepsize,  the convergence rate
slows down.
\end{remark}
\subsubsection{AFSVRG-VP} We provide the convergence result of AFSVRG-VP in Theorem \ref{theorem2}.
\begin{theorem} \label{theorem2} Under Assumptions \ref{assumption4}-\ref{ass5}, to achieve the accuracy $\epsilon$ of (\ref{formulation1}) for AFSVRG-VP, \emph{i.e.}, $\mathbb{E} f (w_{t}) -f(w^*) \leq \epsilon$, let $C={\left (\eta_1 \gamma L^2 q\eta_1 + L_{\max}  \right )\frac{\gamma^2}{2}}$ and $\rho = \frac{\gamma \mu}{2} -  \frac{16 L^2 \eta_1 q C}{\mu} $, we  choose $\gamma$ such that
\begin{align}
\label{eq8} \rho <& 0
\\ \frac{8 L^2 \eta_1 q C}{\rho \mu} \leq& 0.5
\\  \label{eq10}  \gamma^3 \left (  \left ( \frac{ 1  }{2} +     \frac{2C}{\gamma} \right )  \eta_2  \tau    + 4 \frac{C}{\gamma}     \eta_1^2 q  \right )   \frac{\eta_1 q  L^2 G}{\rho} \leq& \frac{\epsilon}{8}
\end{align}
the inner  epoch number  should satisfy $\upsilon(t) \geq \frac{\log 0.25}{\log (1 - \rho)}$, and the outer loop number  should satisfy $S \geq \frac{\log \frac{2 ( f (w_{0}) -f(w^*) )}{\epsilon }}{\log \frac{4}{3}} $.
\end{theorem}
\begin{remark}
Theorem \ref{theorem2} shows that, the convergence rate of AFSVRG-VP is $\mathcal{O}(  \log\left ( \frac{1}{\epsilon}  \right ) )$ to reach the  accuracy $\epsilon$.
\end{remark}
\subsubsection{AFSAGA-VP}  We provide the convergence result of AFSAGA-VP in Theorem \ref{theorem3}.
\begin{theorem} \label{theorem3}
Under Assumptions \ref{assumption4}-\ref{ass5}, to achieve the accuracy $\epsilon$ of (\ref{formulation1}) for AFSAGA-VP, \emph{i.e.}, $\mathbb{E} f (w_{t}) -f(w^*) \leq \epsilon$, let $c_0= (  ( \frac{\eta_2  }{2}   + 3    ( \gamma  q\eta_1^2 + L_{\max}  ) (\eta_1+2\eta_2)   ) \tau +  ( \gamma L^2 q\eta_1^2 + 8 L_{\max}  )  \eta_1 q   \eta_1   ) \gamma^4 L^2 \eta_1 q  G$, $c_1={\left ( \gamma L^2 q\eta_1^2 + L_{\max} \right ) \gamma^2 \eta_1 q 2 L^2 }$, $c_2={4\left ( \gamma L^2 q\eta_1^2 + L_{\max} \right )    \frac{L^2 \eta_1^2 q }{l } \gamma^2}$, and let $\rho \in (1 -\frac{1}{l},1)$, we  choose $\gamma$ such that
\begin{align}
\label{eq11}\frac{4 c_0}{ \gamma \mu (1-\rho) \left ( \frac{\gamma \mu^2}{4} -2 c_1-c_2 \right )} \leq& \frac{\epsilon}{2}
\\
0<1 -  \frac{\gamma \mu}{4} & <&1
\\ -  \frac{\gamma \mu^2}{4} +2 c_1+c_2 \left ( 1+   \frac{1}{1- \frac{ 1 -\frac{1}{l}}{\rho}}  \right ) \leq& 0
\\ \label{eq14}  -  \frac{\gamma \mu^2}{4} +c_2+  c_1 \left ( 2+ \frac{1}{1- \frac{ 1 -\frac{1}{l}}{\rho}} \right )   \leq& 0
\end{align}
the epoch  number should satisfy $ \upsilon(t)  \geq \frac{\log \frac{2 \left ({{2\rho- 1 +  \frac{\gamma \mu}{4} }} \right ) ( f (w_{0}) -f(w^*) ) }{\epsilon {(\rho- 1 +  \frac{\gamma \mu}{4} )\left ( \frac{\gamma \mu^2}{4} -2 c_1-c_2 \right )} }}{\log \frac{1}{\rho}}$.
\end{theorem}
\begin{remark}
Theorem \ref{theorem2} shows that, the convergence rate of AFSAGA-VP is $\mathcal{O}(  \log\left ( \frac{1}{\epsilon}  \right ) )$ to reach the  accuracy $\epsilon$.
\end{remark}

\subsection{Security Analysis}\label{section_sec}
We  discuss the data and model security (in other words,  prevent local data and model on one worker leaked to or inferred by  other workers) of  AFSG-VP, AFSVRG-VP and AFSAGA-VP under the \textit{semi-honest} assumption. Note that the   \textit{semi-honest} assumption (\textit{i.e.}, Assumption \ref{ass_semi_honest}) is commonly used in previous works \cite{wan2007privacy,hardy2017private,cheng2019secureboost}.
\begin{assumption}[Semi-honest security] \label{ass_semi_honest}
All workers  will follow the algorithm to perform the
correct computations. However, they may retain records of
the intermediate computation results which they may use
later to infer the other work's data and model.
\end{assumption}

Before discussing the data and model privacy in detail, we first introduce the concepts of exact and approximate inference attacks in Definitions \ref{defi_inf_attack} and \ref{defi_inf_attack2}.

\begin{definition}[Exact inference attack]\label{defi_inf_attack}
An exact  inference attack on the $\ell$-th  worker is  to exactly  infer  some feature group $\mathcal{G}$ of one sample $x$ or model $w$ which belongs  from other workers
 without directly accessing it.
\end{definition}
\begin{definition}[$\epsilon$-approximate inference attack]\label{defi_inf_attack2}
An $\epsilon$-approximate inference attack on the $\ell$-th  worker is  to  infer some feature group $\mathcal{G}$ of one sample $x$ (model $w$) as $\widehat{x}_{\mathcal{G}}$ ($\widehat{w}_{\mathcal{G}}$) with the accuracy of $\epsilon$ (i.e., $\| \widehat{x}_{\mathcal{G}} -{x}_{\mathcal{G}}\|_\infty \leq \epsilon$ or $\| \widehat{w}_{\mathcal{G}} -{w}_{\mathcal{G}}\|_\infty \leq \epsilon$) which belongs  from other workers
 without directly accessing it.
\end{definition}
\noindent \textbf{Security Analysis based on Algorithm \ref{protocol1}:} \ Firstly, we  show that  AFSG-VP, AFSVRG-VP and AFSAGA-VP based on Algorithm \ref{protocol1} can prevent the exact  inference attack, however has the risk of approximate inference attack.

Specifically, in order to infer the information of $(w_t)_{\mathcal{G}_{\ell}}$ on the ${\ell'}$-th worker where $\ell' \neq \ell$, we only  have a sequence of  linear system of $o_t= (w_t)_{\mathcal{G}_{\ell}}^T (x_i)_{\mathcal{G}_{\ell}}$ with a sequence of trials of  $(x_i)_{\mathcal{G}_{\ell}}$ and  $o_t$ while only $o_t$ are known. Thus, it is impossible to infer the exact information of  $(w_t)_{\mathcal{G}_{\ell}}$ from the linear system of $o_t= (w_t)_{\mathcal{G}_{\ell}}^T (x_i)_{\mathcal{G}_{\ell}}$  even the size of feature group ${\mathcal{G}_{\ell}}$ is one. Similarly, we cannot  infer the exact information of  $(x_i)_{\mathcal{G}_{\ell}}$.

However, it has the potential to approximately infer  $(w_t)_{\mathcal{G}_{\ell}}$ from the linear system of $o_j=w_{\mathcal{G}_{\ell}}^T (x_i)_{\mathcal{G}_{\ell}}$ if the size of feature group ${\mathcal{G}_{\ell}}$ is one. Specifically, if we know the region of $(x_i)_{\mathcal{G}_{\ell}}$ as $\mathcal{I}$, we can have that $o_j/w_{\mathcal{G}_{\ell}} \in \mathcal{I}$ which can infer $w_{\mathcal{G}_{\ell}}$ approximately. Further, we can infer $(x_i)_{\mathcal{G}_{\ell}}$ approximately.
We say that Algorithm \ref{protocol1} has the risk of  approximate inference attack.

\noindent \textbf{Security Analysis based on Algorithm \ref{protocol2}:} \ Next, we show that  AFSG-VP, AFSVRG-VP and AFSAGA-VP based on Algorithm \ref{protocol2} can prevent the approximate  inference attack.

As discussed above, the key to preventing the approximate inference attack is to  mask the value of $o_j$.
As described in lines \ref{step2}-\ref{step3} of Algorithm \ref{protocol2}, we add an extra random variable  $b^{{\ell}'}$  into $w_{\mathcal{G}_{\ell'}}^T (x_i)_{\mathcal{G}_{\ell'}}$, and transfer the value of $w_{\mathcal{G}_\ell'}^T (x_i)_{\mathcal{G}_\ell'}+b^{{\ell}'} $ to another worker. This operation  makes  the received part cannot directly get the value of $o_j$. Finally,  the $\ell$-th active worker gets  the global sum $\sum_{{\ell}'=1}^q  \left (w_{\mathcal{G}_{{\ell}'}}^T (x_i)_{\mathcal{G}_{{\ell}'}}+b^{{\ell}'} \right )$ by using a tree-structured communication  scheme based on the tree structure $T_1$. Thus, the lines \ref{step2}-\ref{step5} of Algorithm \ref{protocol2} keeps data privacy.

Line \ref{step7} of Algorithm \ref{protocol2} is trying to get $w^T x$ by  removing $\overline{b} =\sum_{\ell'=1}^q  b^{{\ell'}}$ from the sum $\sum_{{\ell}'=1}^q  \left (w_{\mathcal{G}_{{\ell}'}}^T (x_i)_{\mathcal{G}_{{\ell}'}}+b^{{\ell}'} \right )$. To prove that Algorithm \ref{protocol2} can reduce the risk of  approximate inference attack, we only need to prove that the calculation of $\overline{b} =\sum_{\ell'=1}^q  b^{{\ell'}}$ in line \ref{step7} of Algorithm \ref{protocol2} does not disclose the value of $b^{{{{\ell}'}}}$ or the sum of $b^{{{{\ell}'}}}$ on a node of tree $T_1$ (please see Lemma \ref{lemma2}, the proof is provided in the Appendix).

\begin{lemma}\label{lemma2}
Using a  tree structure $T_2$ on all workers    which is significantly different to  the tree $T_1$ to compute $\overline{b} =\sum_{\ell'=1}^q  b^{{\ell'}}$,  there is no risk to  disclose the value of $b^{{{{\ell}'}}}$, or the sum of $b^{{{{\ell}'}}}$ on all nodes of a subtree of $T_1$  whose sizes are larger than 1 and smaller than $T_1$.
\end{lemma}

\subsection{Complexity Analysis}
We give the computational complexities and communication costs of AFSG-VP, AFSVRG-VP and
AFSAGA-VP as follows.

The computational complexity for one iteration of AFSGD-VP  is $\mathcal{O}(d +q)$. Thus, the total computational complexity of AFSGD-VP is $\mathcal{O}((d +q)t)$, where $t$ denotes the iteration number. Further, the    communication cost for one iteration of AFSGD-VP is $\mathcal{O}(q)$, and the total communication cost   is $\mathcal{O}(qt)$.

For AFSVRG-VP, the computational complexity and communication cost of line 3 are $\mathcal{O}((d +q)l)$ and $\mathcal{O}(ql)$ respectively. Assume that the inner loop number of AFSVRG-VP is $t$. Thus,   the total computational complexity of AFSVRG-VP is $\mathcal{O}((d +q)(l+t)S)$, and the communication cost is $\mathcal{O}(q(l+t)S )$.

For AFSAGA-VP, the computational complexity and communication cost of line 2 are $\mathcal{O}((d +q)l)$ and $\mathcal{O}(ql)$ respectively. Assume that the loop number of AFSAGA-VP is $t$. Thus,   the total computational complexity of AFSAGA-VP is $\mathcal{O}((d +q)(l+t))$, and the communication cost is $\mathcal{O}(q(l+t) )$.

\section{Experimental Results}\label{sectionexperiments}

In this section, we first present the experimental setup, and then provide the experimental results and discussions.

\subsection{Experimental Setup}
\subsubsection{Design of Experiments}  In the experiments, we not only verify the theoretical results of AFSG-VP, AFSVRG-VP and
AFSAGA-VP, but also show that our algorithms have much better efficiency than the corresponding synchronous algorithms (\emph{i.e.}, FSG-VP, FSVRG-VP
and FSAGA-VP). We compare our asynchronous vertical SGD, SVRG and SAGA algorithms (\emph{i.e.}, AFSG-VP, AFSVRG-VP and AFSAGA-VP) with synchronous version of vertical SGD, SVRG   and SAGA (denoted as FSG-VP, FSVRG-VP and FSAGA-VP respectively) on  classification and regression tasks, where  FSVRG-VP is almost same to FD-SVRG \cite{zhang2018feature}. For the  classification tasks, we consider the $\ell_2$-norm regularized logistic regression model as follows:
\begin{equation}
    \min_{\textbf{w}} f(\textbf{w})=\frac{1}{l}\sum^{l}_{i=1}\log(1+e^{-y_i\textbf{w}^Tx_i}) + \frac{\lambda}{2}\|\textbf{w}\|^2
\end{equation}
For the regression tasks, we use the ridge linear regression method with $\ell_2$-norm regularization as follows:
\begin{equation}
    \min_{\textbf{w},b} f(\textbf{w},b)=\frac{1}{l}\sum^{l}_{i=1}(\textbf{w}^Tx_i+b-y_i)^2 + \frac{\lambda}{2} \left (\|\textbf{w}\|^2+b^2 \right )
\end{equation}

\subsubsection{Experiment Settings} We run all the experiments on a cluster with 32 nodes of 20-core Intel Xeon E5-2660 2.60 GHz (Haswell). The nodes are connected with 56 Gb FDR. We use OpenMPI \cite{graham2005open} v3.1.1 with multi-thread support for communication between worker processes and Armadillo \cite{sanderson2016armadillo} v9.700.3 for efficient matrix computation. Each worker is placed on a different machine node. For the $\ell_2$ regularization term, we set the coefficient $\lambda=1e^{-4}$ for all experiments. We also choose the best learning rate $\in (5e^{-1},1e^{-1},5e^{-2},1e^{-2},...)$ for each algorithm on different learning tasks. There is a synthetic straggler node which may be 40\% to 300\% slower than the fastest worker node to simulate the real application scenario. In practice, it is normal that different parties in a federated learning system will possess different computation and communication power and resources.

\subsubsection{Implementation Details}
Our asynchronous algorithms are implemented under the decentralized framework, where a worker own its own part of data and model parameters. There is no master node for aggregating data/features/gradients which may lead to undesired user information disclosure. Instead, we utilize a coordinator as in Figure \ref{structureFDSKL} to collect the product computed from local data and parameters from other workers. Each worker node can independently call the coordinator to enable the asynchronous model update. The aggregation of local product is performed in a demand-based manner, which means that only when a worker node needs to update its local parameter will it request the coordinator to pull the local product from other worker nodes. Different from horizontal federated learning \cite{yang2019federated,abs-1907-10218,abs-1902-00641}, it will be much harder for an attacker to restore the information of the user data in a worker node using the local product than the gradient.

Specifically, in our asynchronous algorithms, each worker node performs computation rather independently. The main thread of a worker process performs the major workload of gradient computation and model update operation. Another listener thread keeps listening for the request and sends back the local product to the requesting source. The computation diagram can be summarized as follows for a worker:
\begin{enumerate}
    \item Randomly select an index of the data.
    \item Call the coordinator to broadcast the index to the listeners of other workers.
    \item Reduce the sum of the local product back from the listeners.
    \item Perform gradient computation and model parameters update.
\end{enumerate}
Note that the local product is computed based on a worker's current parameters. Overall speaking, however, some workers may have updated their parameters more times than other workers. Different from common asynchronous horizontal algorithms \cite{meng2016asynchronous,gu2016asynchronous}, although the worker processes run asynchronously, all the parameters a worker uses to compute gradient is most up-to-date. The broadcast and reduce operation are also realized in a tree-structured scheme to reduce communication costs.

\subsubsection{Datasets}
\begin{table*}[htbp]
\caption{The  datasets used in the experiments.}
\scriptsize
 \setlength{\tabcolsep}{0.5mm}
    \centering
    \begin{tabular}{ccccccccc}
    \toprule
    \multicolumn{1}{c}{} & \multicolumn{6}{c}{Classification Tasks} & \multicolumn{2}{c}{\multirow{2}{*}{Regression Tasks}} \\ \cline{2-7}
    \multicolumn{1}{c}{} & \multicolumn{2}{c}{Financial} & \multicolumn{4}{c}{Large-Scale} & \multicolumn{2}{c}{}\\
    \cmidrule(lr){2-3} \cmidrule(lr){4-7} \cmidrule(lr){8-9}
     & UCICreditCard & GiveMeSomeCredit & news20 & rcv1 & url & webspam & E2006-tfidf & YearPredictionMSD\\
     \midrule
     \#Train & 24,000 & 96,257 & 15,997 & 677,399 & 1,916,904 & 280,000 & 16,087 & 463,715\\
     \#Test & 6,000 & 24,012 & 3,999 & 20,242 & 479,226 & 70,000 & 3,308 & 51,630\\
     \#Feature & 90 & 92 & 1,355,191 & 47,236 & 3,231,961 & 16,609,143 & 150,360 & 90\\
    \bottomrule
    \end{tabular}
    \label{dataset}
\end{table*}
To fully demonstrate the scalability of our asynchronous vertical federated learning algorithms, we conduct experiments on eight datasets as summarized in Table \ref{dataset} for binary classification and regression tasks. Two real and relatively small financial datasets, UCICreditCard and GiveMeSomeCredit are from the Kaggle\footnote{https://www.kaggle.com/datasets} website. The other six datasets are from the LIBSVM\footnote{https://www.csie.ntu.edu.tw/~cjlin/libsvmtools/datasets/} website \cite{CC01a}. We split news20, url and webspam datasets into training data and testing data randomly with a ratio of 4:1. We also use rcv1's testing data for training and training data for testing as there are more instances in the testing data.
\subsection{Results and Discussions}

\begin{figure*}[htbp]
\centering
    \begin{subfigure}[h]{0.24\textwidth}
    \includegraphics[width=\textwidth]{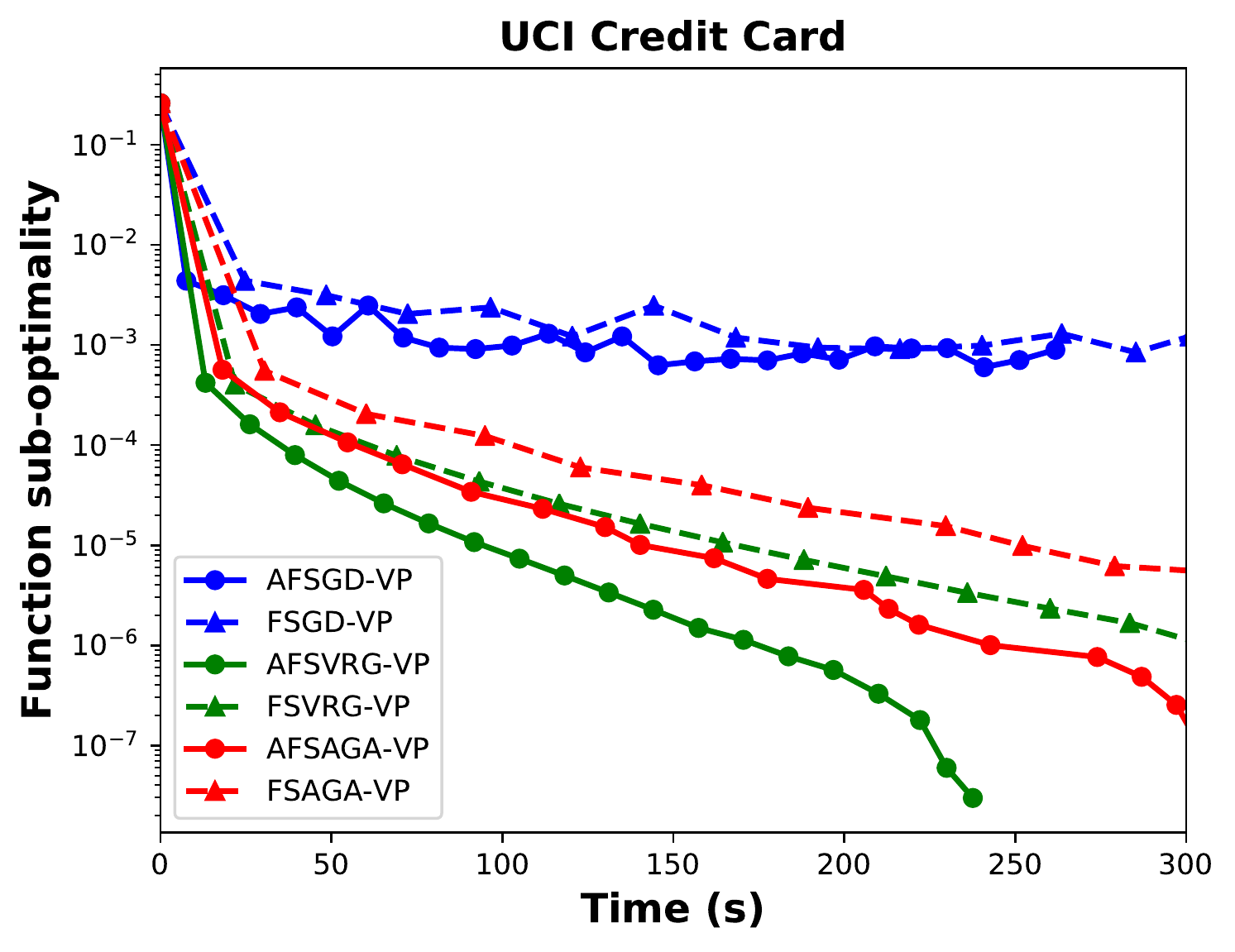}
    \end{subfigure}
    \begin{subfigure}[h]{0.24\textwidth}
    \includegraphics[width=\textwidth]{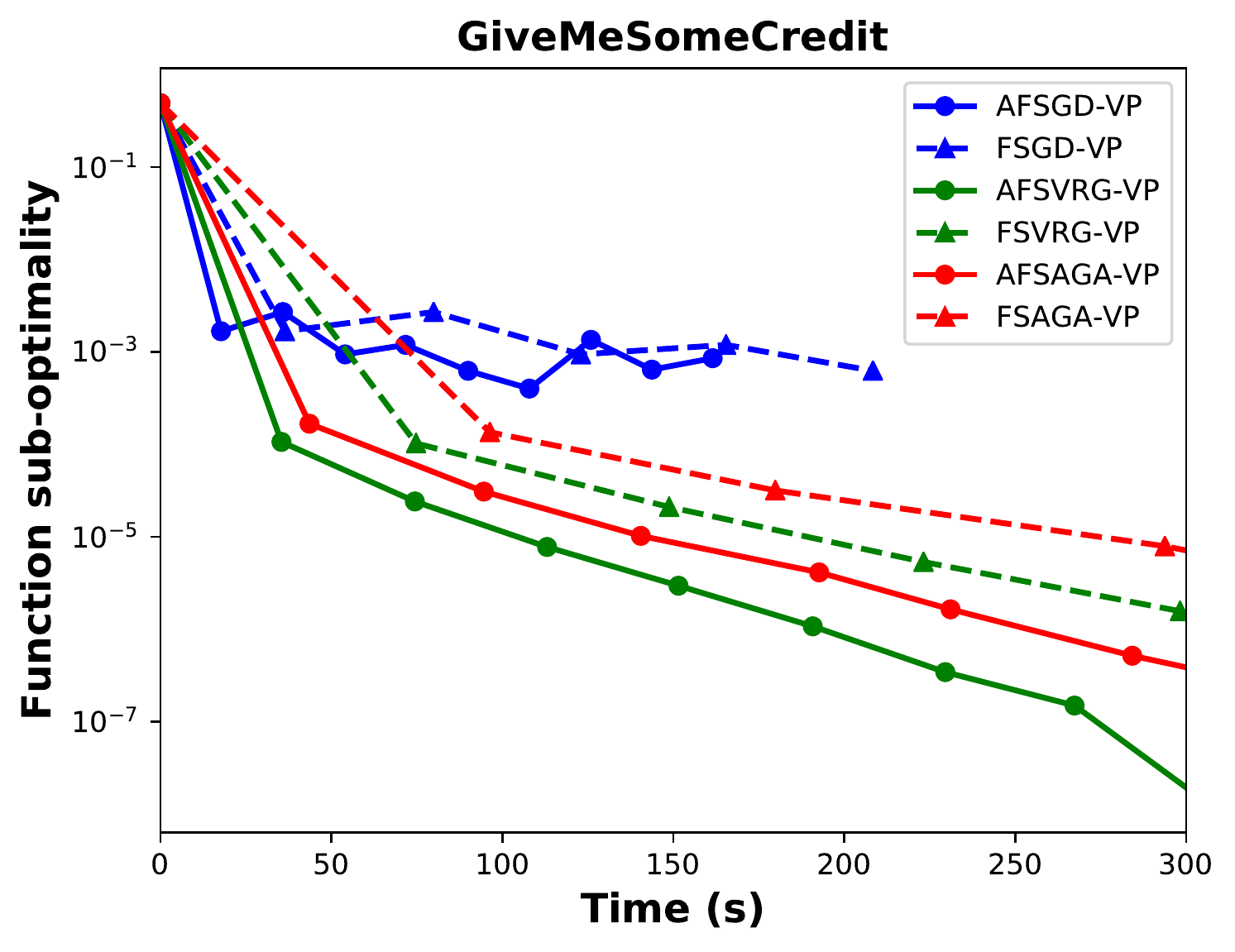}
    \end{subfigure}
    \begin{subfigure}[h]{0.24\textwidth}
    \includegraphics[width=\textwidth]{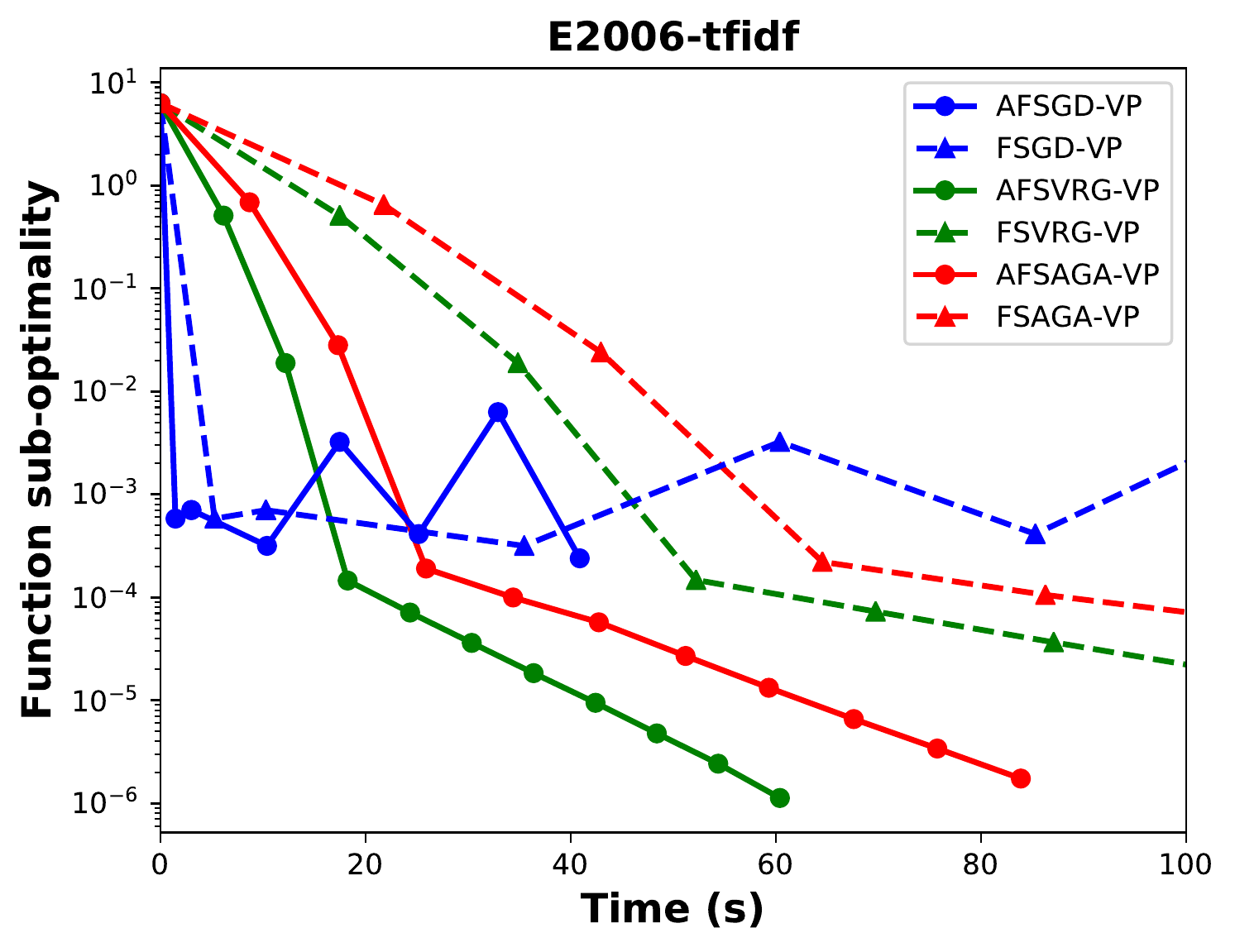}
    \end{subfigure}
    \begin{subfigure}[h]{0.24\textwidth}
    \includegraphics[width=\textwidth]{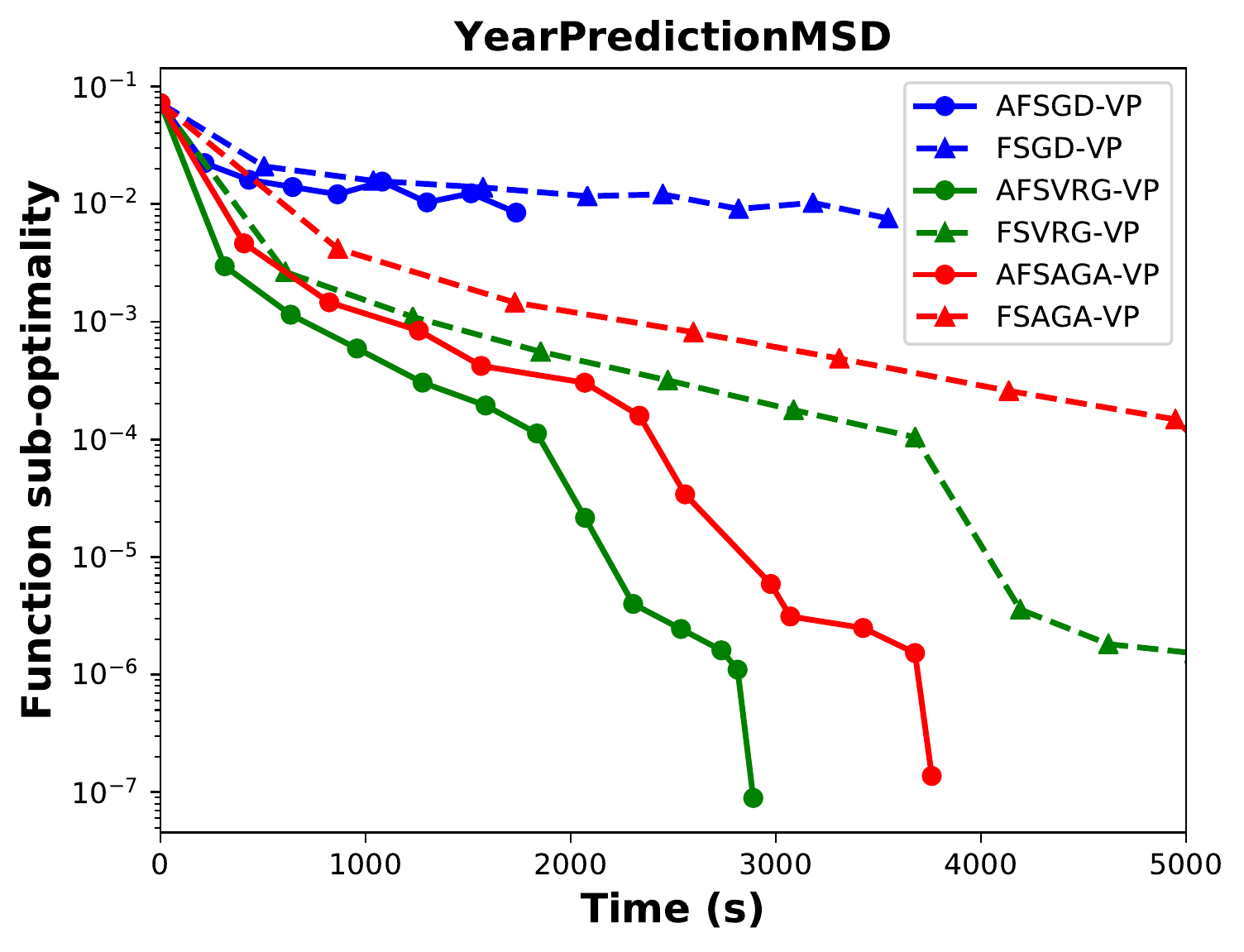}
    \end{subfigure}
\caption{Convergence of different algorithms for  classification  task (left two) and regression  task (right two).}
\label{financial, regression}
\end{figure*}

\begin{figure*}[htbp]
\centering
    \begin{subfigure}[h]{0.24\textwidth}
    \includegraphics[width=\textwidth]{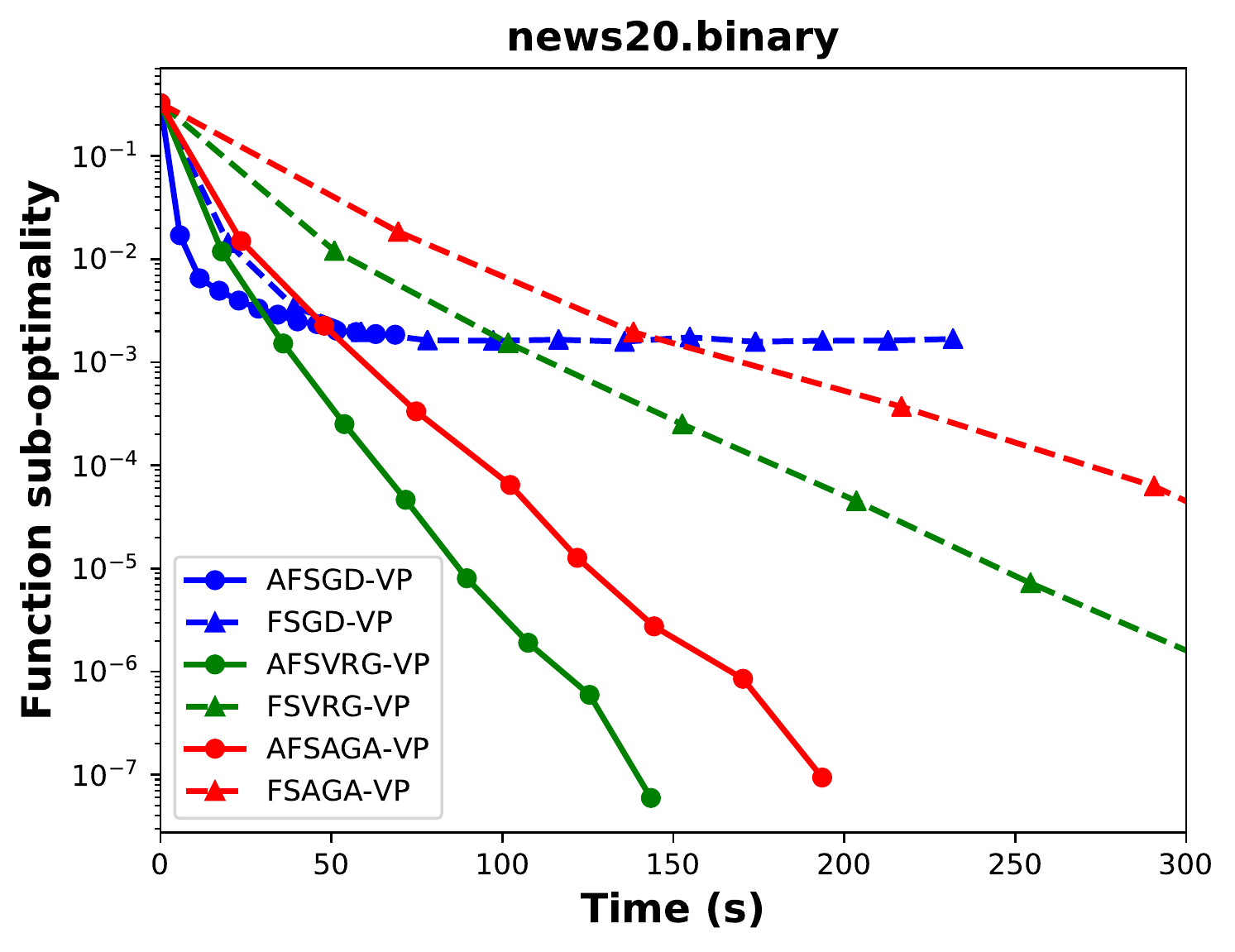}
    \end{subfigure}
    \begin{subfigure}[h]{0.24\textwidth}
    \includegraphics[width=\textwidth]{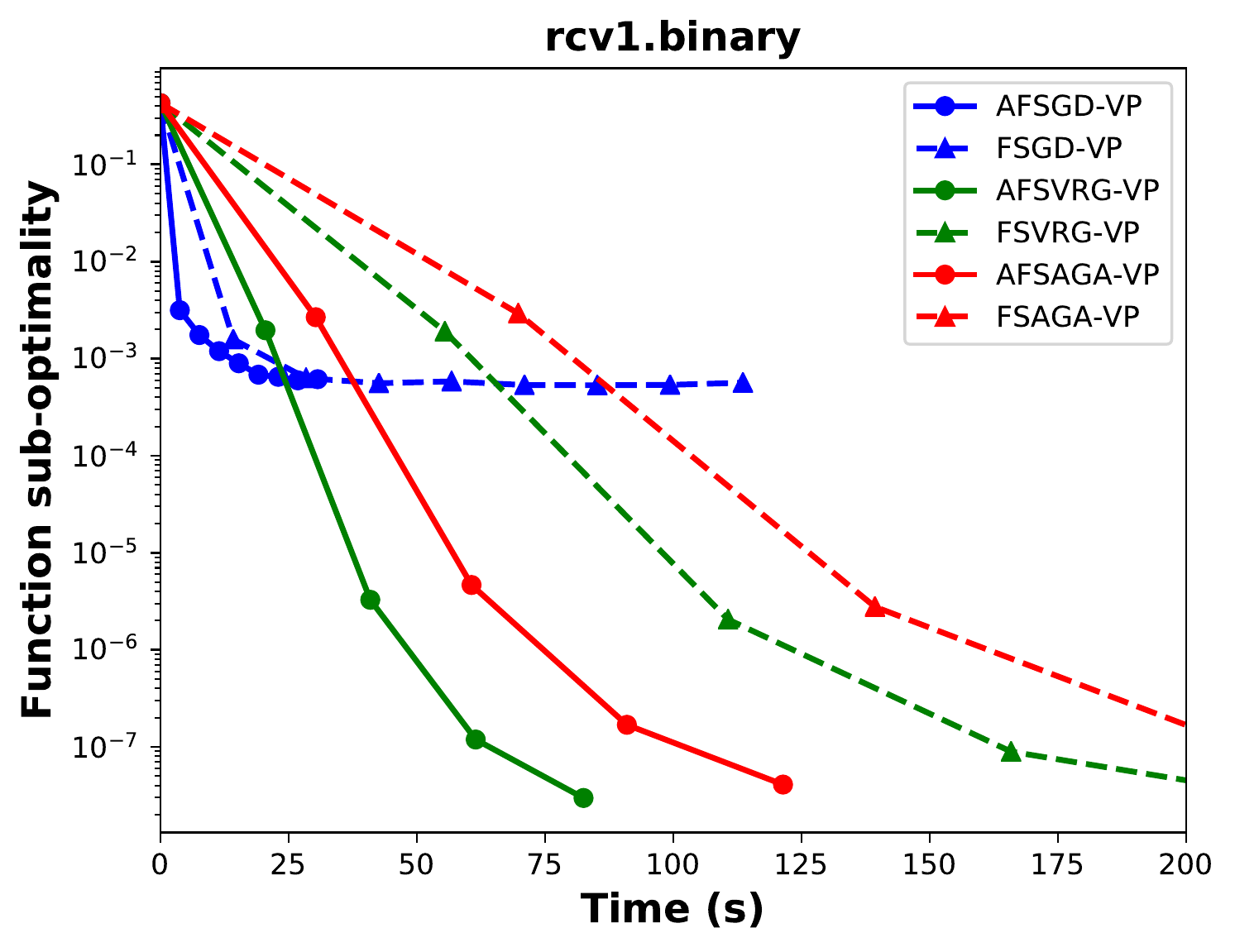}
    \end{subfigure}
    \begin{subfigure}[h]{0.24\textwidth}
    \includegraphics[width=\textwidth]{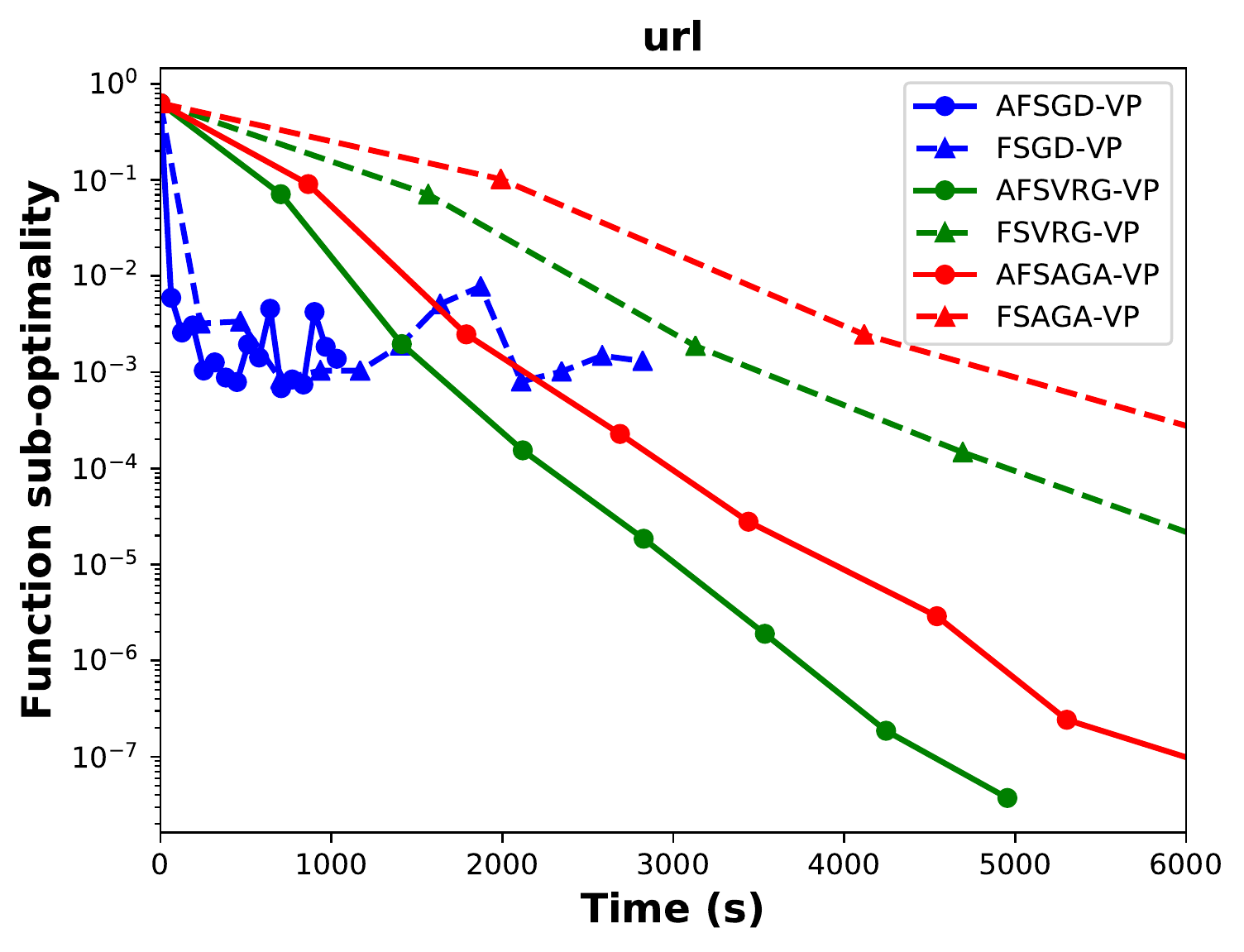}
    \end{subfigure}
    \begin{subfigure}[h]{0.24\textwidth}
    \includegraphics[width=\textwidth]{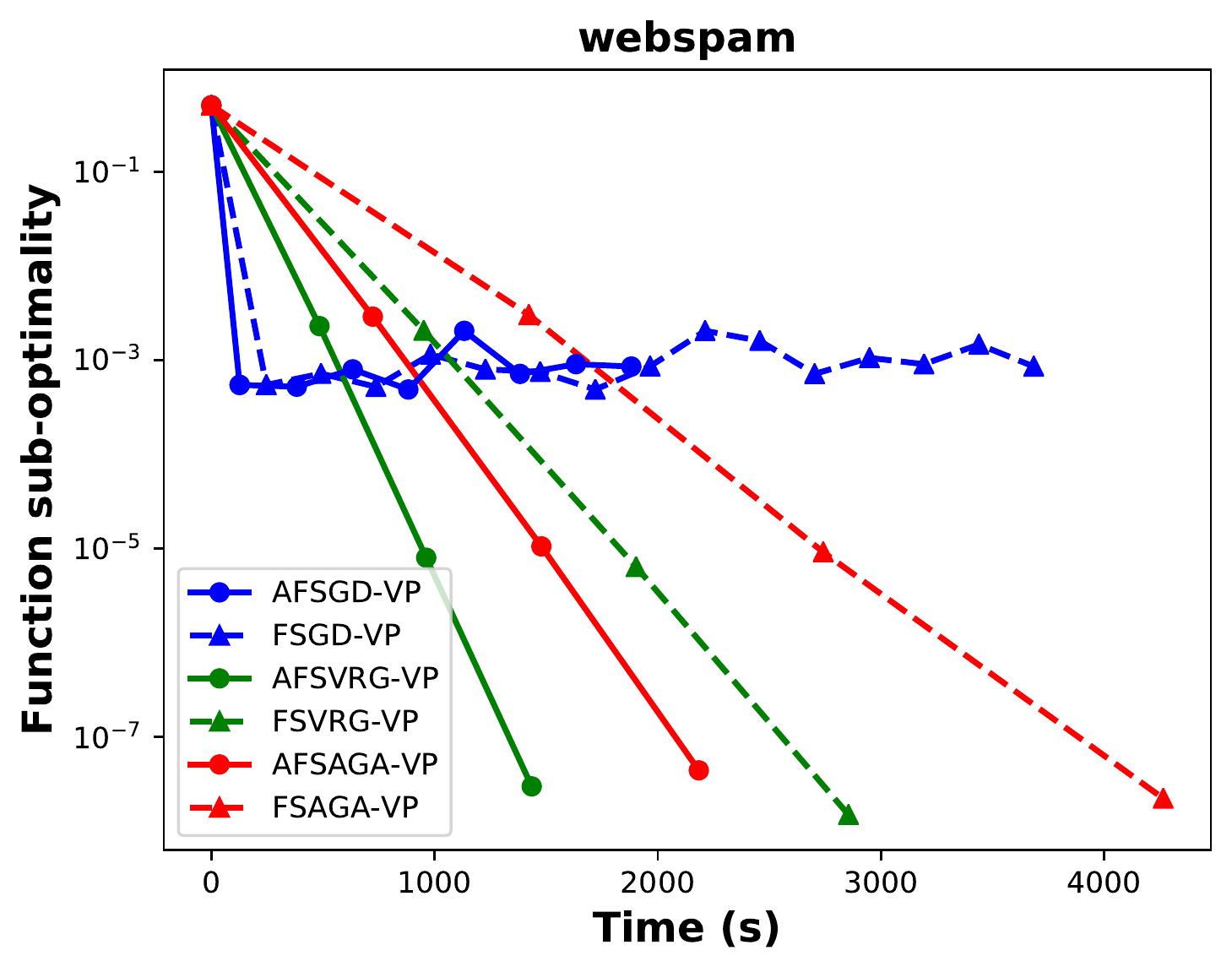}
    \end{subfigure}
\caption{Convergence of different algorithms for binary classification task on more large-scale datasets.}
\label{large scale}
\end{figure*}

\begin{figure*}[htbp]
\centering
    \begin{subfigure}[h]{0.3\textwidth}
    \includegraphics[width=\textwidth]{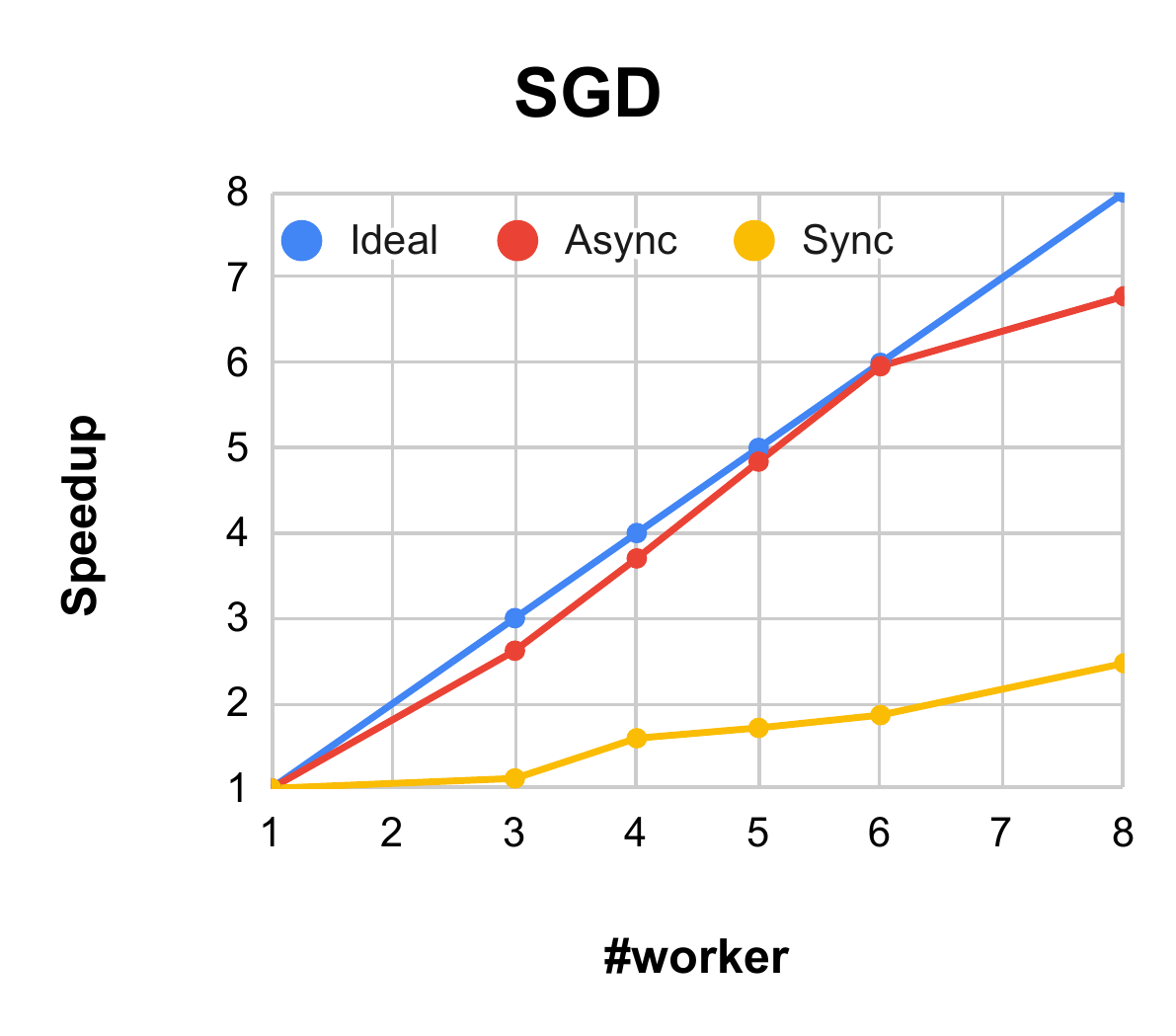}
    \end{subfigure}
    \begin{subfigure}[h]{0.3\textwidth}
    \includegraphics[width=\textwidth]{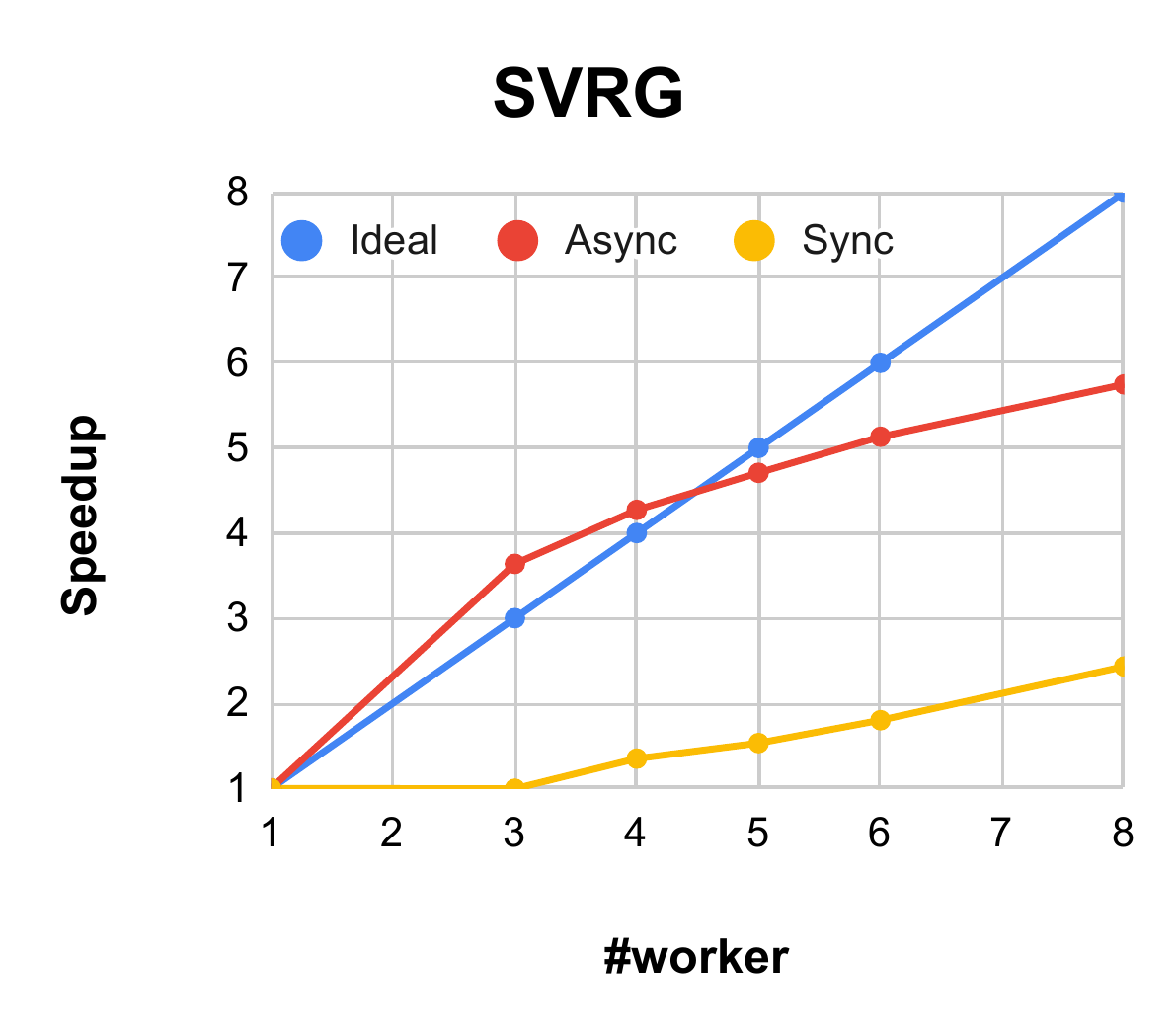}
    \end{subfigure}
    \begin{subfigure}[h]{0.3\textwidth}
    \includegraphics[width=\textwidth]{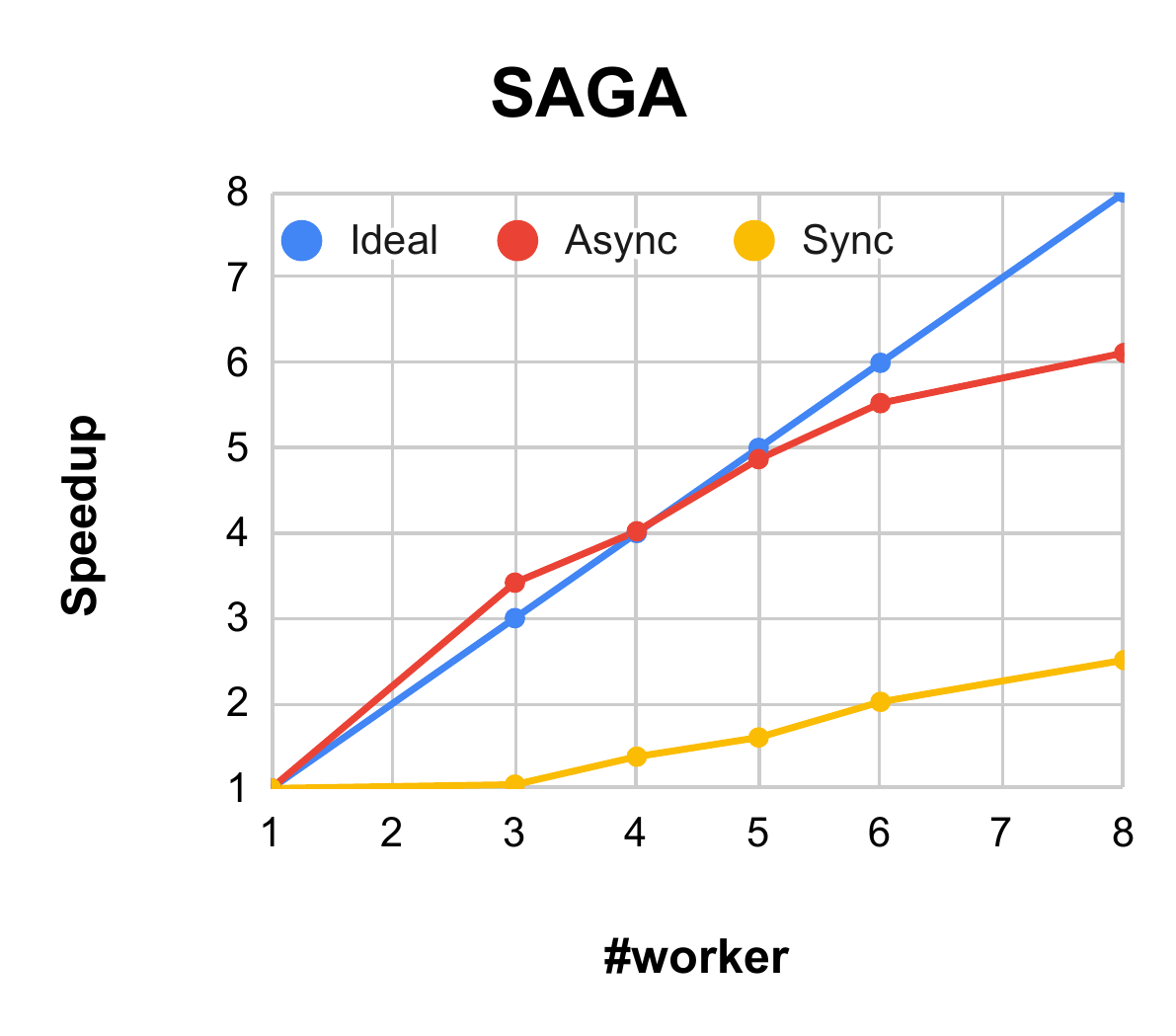}
    \end{subfigure}
\caption{Scalability (url dataset for classification task).}
\label{scalability}
\end{figure*}
\subsubsection{Classification Tasks}
We first compared our asynchronous federated learning algorithm  with synchronous version on financial datasets to demonstrate the ability to address real application. In asynchronous algorithms, each worker saves its local parameters every fixed interval for testing. In the synchronous setting, each worker saves the parameters every fixed number of iterations as all the workers run at the same pace. We follow this scheme for the other experiments.

The original total numbers of features of UCICreditCard and GiveMeSomeCredit dataset are 23 and 10 respectively. We apply one-hot encoding for categorical features and standardize other features column-wisely. The numbers of features become 90 and 92 respectively after the simple data preprocessing.

Four worker nodes are used in this part of the experiment. As shown by the left two figures in Figure \ref{financial, regression}, our asynchronous vertical algorithms consistently surpass their synchronous counterparts. The y-axis function sub-optimality represents the error of objective function to the global optimal. The shape of the convergence curve is firstly determined by the optimization method we choose, \emph{i.e.}, SGD, SVRG and SAGA. The error precision of SGD is usually higher than SVRG, while that of SAGA is similar to SVRG. Then the convergence speed is mostly influenced by the computation and communication complexity. In asynchronous settings there is no inefficient idle time to wait for other workers, so the update frequency is much higher, which results in faster convergence speed of our asynchronous algorithm with regard to wall clock time.

Previous experiments show that our asynchronous federated learning algorithms could address real financial problems more efficiently. In this part we will use large-scale benchmark datasets, \emph{i.e.} large number of data instances and high-dimensional features, for further validations. In our experiments, 8 worker nodes are used for experiments on new20 and rcv1 datasets; 16 worker nodes are used for experiments on url and webspam datasets. The results are visualized in Figure \ref{large scale}. As the total computation budget grows, the speedup of the asynchronous algorithm becomes more obvious. So it will be be much more efficient when put into large-scale practical use. Our asynchronous SGD, SVRG and SAGA still surpass their synchronous counterparts in the experiments on all the four datasets.

\subsubsection{Regression Tasks}
To further illustrate the advantages of asynchronous algorithms can scale to various tasks, we also conduct experiments on regression problems as shown by the right two figures in Figure \ref{financial, regression}. Both the E20060-tfidf with a smaller number of data instances but a larger number of features, and the YearPredictionMSD with larger number of instances but a smaller number of features are tested. 4 worker nodes are used in this experiment and similar conclusions as previous can be reached.

\subsubsection{Asynchronous Efficiency}

\begin{table}[htbp]
\caption{Asynchronous speedup.}
\centering
    \setlength{\tabcolsep}{4mm}
    \begin{tabular}{cccc}
    \toprule
   \multicolumn{1}{c}{ \multirow{2.5}{*}{Dataset} }
   & \multicolumn{3}{c}{Speedup}\\
    \cmidrule(lr){2-4}
    & SGD & SVRG & SAGA \\
    \midrule
    UCICreditCard & 1.82 & 1.93 & 1.95\\
    GiveMeSomeCredit & 1.89 & 2.15 & 2.11\\
    new20 & 3.37 & 2.84 & 2.89\\
    rcv1 & 3.70 & 2.64 & 2.28\\
    url & 2.74 & 2.51 & 2.38\\
    webspam & 1.96 & 1.97 & 1.99\\
    E2006-tfidf & 3.32 & 2.89 & 2.51\\
    YearPredictionMSD & 2.03 & 2.24 & 2.26\\
    \bottomrule
    \end{tabular}
    \label{speedup}
\end{table}

\begin{figure}[htb]
\centering
    \begin{subfigure}[h]{0.5\textwidth}
    \includegraphics[width=\textwidth]{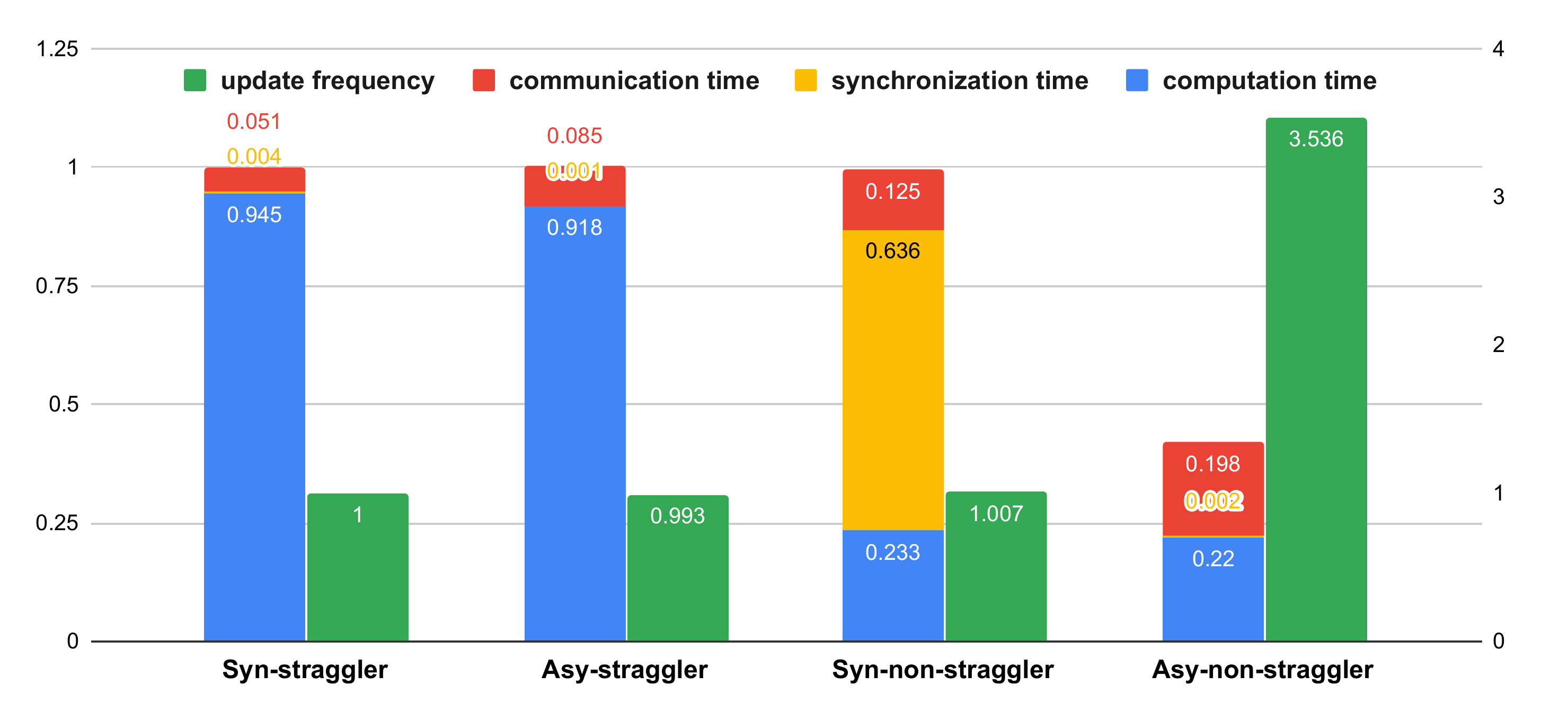}
    \end{subfigure}
\caption{Asynchronous efficiency (8 workers on url dataset for binary classification task).}
\label{efficiency}
\end{figure}
The speedup results of asynchronous algorithms compared with synchronous ones are summarized in Table \ref{speedup}. The speedup is computed based on the time when the algorithm reaches a certain precision of optimality ($1e^{-4}$ for SVRG and SAGA; $1e^{-2.5}$ or $1e^{-1.5}$ for SGD based on different datasets).

To further analyze the efficiency of our asynchronous algorithms, we quantify the composition of the time consumption of asynchronous and synchronous algorithms as in Figure \ref{efficiency}. The execution time and update frequency are scaled by those of the straggler in the synchronous algorithm. The computation time of stragglers is much higher than non-stragglers, which leads to a large amount of synchronization time for non-stragglers in synchronous algorithms. While in our asynchronous algorithms, non-stragglers pull the update-to-date product information from stragglers without waiting the straggler to finish its current iteration. As a result, the synchronization time is eliminated. Although the communication cost increases because each worker needs to independently aggregate product from other workers, we can achieve a large gain in terms of the update frequency.
\subsubsection{Scalability}

The scalability in terms of number of workers is shown in Figure \ref{scalability}. Synchronous algorithms cannot address the problem of straggler and behaves poorly. Using synchronization barrier keeps non-stragglers inefficiently waiting for the straggler. Our asynchronous algorithms behave like ideal in the beginning as they can address the straggler problem well, and deviate from ideal when the number of workers continues to grow because the communication overheads will limit the speedup.
\section{Conclusion} \label{section_concluding}
In this paper, we  proposed an  asynchronous  federated  SGD (AFSGD-VP) algorithm and its SVRG and SAGA variants  for vertically partitioned data. To the best of our knowledge,  these algorithms are the first asynchronous  federated learning algorithms for vertically partitioned data. Importantly, we provided the convergence rates of AFSGD-VP and its SVRG and SAGA variants under the condition of  strong convexity for the objective function. We also proved the  model privacy and
data privacy. Extensive experimental results on  a variety of vertically partitioned datasets not only verify the theoretical results of AFSGD-VP and its SVRG and SAGA variants, but also show that our algorithms have much better efficiency than the corresponding synchronous algorithms.

\section*{Appendix}
\subsection*{Appendix A: Proof of Theorem  \ref{theorem1}}
Before proving Theorem \ref{theorem1}, we first provide several basic inequalities in Lemma \ref{AsySPSAGA_lemma3}.
\begin{lemma} \label{AsySPSAGA_lemma3} For  AFSGD-VP, under Assumptions \ref{assumption3} and \ref{assMax},   we have that
\begin{align}\label{AsySPSAGA_lem3_0}
\nonumber  \sum_{u\in K(t)} \| \nabla_{\mathcal{G}_{\psi(u)}} f(w_u) \|^2
 \geq  \frac{1}{2} \sum_{u\in K(t)} \| \nabla_{\mathcal{G}_{\psi(u)}} f(w_t) \|^2
 \\   - \eta_1 \gamma^2 L^2  \sum_{u\in K(t)}   \sum_{v \in \{t,\ldots,u \}}\|  \widehat{v}^{\psi(v)}_v \|^2
\end{align}
\end{lemma}
\begin{proof} For any $u \in K(t)$, we have that
\begin{align}\label{eqLem1_1}
  \| \nabla_{\mathcal{G}_{\psi(u)}} f(w_t) \|^2
 \leq& \left ( \| \nabla_{\mathcal{G}_{\psi(u)}} f(w_t) - \nabla_{\mathcal{G}_{\psi(u)}} f(w_u) \| +  \| \nabla_{\mathcal{G}_{\psi(u)}} f(w_u) \| \right)^2
\\ \nonumber \leq& 2 \| \nabla_{\mathcal{G}_{\psi(u)}} f(w_t) - \nabla_{\mathcal{G}_{\psi(u)}} f(w_u) \|^2 +  2 \| \nabla_{\mathcal{G}_{\psi(u)}} f(w_u) \|^2
\\ \nonumber \leq& 2  \| \nabla f(w_t) - \nabla f(w_u) \|^2 +  2 \| \nabla_{\mathcal{G}_{\psi(u)}} f(w_u) \|^2
\\ \nonumber \stackrel{ (a) }{\leq} & 2 L^2 \| w_t - w_u \|^2 +  2 \| \nabla_{\mathcal{G}_{\psi(u)}} f(w_u) \|^2
\\ \nonumber =& 2 L^2 \gamma^2 \| \sum_{v \in \{t,\ldots,u \}} \textbf{U}_{\psi(v)} \widehat{v}^{\psi(v)}_v  \|^2 +  2 \| \nabla_{\mathcal{G}_{\psi(u)}} f(w_u) \|^2
\\ \nonumber \leq & 2 \eta_1 L^2 \gamma^2 \sum_{v \in \{t,\ldots,u \}} \| \widehat{v}^{\psi(v)}_v \|^2 +  2 \| \nabla_{\mathcal{G}_{\psi(u)}} f(w_u) \|^2
\end{align}
where the inequality (a) uses Assumption \ref{assumption3},  the last inequality uses Assumption \ref{assMax}.
According to (\ref{eqLem1_1}), we have that
\begin{align}\label{eqLem1_2}
\| \nabla_{\mathcal{G}_{\psi(u)}} f(w_u) \|^2
\geq  \frac{1}{2}\| \nabla_{\mathcal{G}_{\psi(u)}} f(w_t) \|^2 - \eta_1 L^2 \gamma^2 \sum_{v \in \{t,\ldots,u \}} \|  \widehat{v}^{\psi(v)}_v \|^2
\end{align}

Summing (\ref{eqLem1_2}) over all $u \in K(t)$, we obtain the conclusion.
This completes the proof.
\end{proof}
Based on the basic inequalities in Lemma \ref{AsySPSAGA_lemma3}, we provide the proof of Theorem \ref{theorem2} in the following.
\newtheorem{refproof}{Proof}

\begin{proof}
 For $u \in K(t)$, we have that
\begin{align}\label{EqThm1_1}
 \mathbb{E} f (w_{u+1})
 \stackrel{ (a) }{\leq}&  \mathbb{E}  ( f (w_{u}) + \langle \nabla f(w_{u}), w_{u+1}-w_{u}  \rangle + \frac{L_{\psi(u)}}{2} \|w_{u+1}-w_{u}   \|^2   )
\\ \nonumber =&  \mathbb{E}  ( f (w_{u}) -  \gamma \langle \nabla f(w_{u}),  \widehat{v}^{\psi(u)}_u  \rangle + \frac{L_{\psi(u)} \gamma^2}{2} \|  \widehat{v}^{\psi(u)}_u  \|^2   )
\\ \nonumber =&  \mathbb{E} \left ( f (w_{u}) + \frac{L_{\psi(u)} \gamma^2}{2} \|  \widehat{v}^{\psi(u)}_u  \|^2    -  \gamma \langle \nabla f(w_{u}),  \widehat{v}^{\psi(u)}_u - \nabla_{\mathcal{G}_{\psi(u)}} f_{i_u} ({w}_u) + \nabla_{\mathcal{G}_{\psi(u)}} f_{i_u} ({w}_u) \rangle  \right )
\\ \nonumber =&  \mathbb{E}  f (w_{u}) -  \gamma \mathbb{E} \langle \nabla f(w_{u}),  \nabla_{\mathcal{G}_{\psi(u)}} f ({w}_u) \rangle
\\ \nonumber & + \gamma \mathbb{E} \langle \nabla f(w_{u}),   \nabla_{\mathcal{G}_{\psi(u)}} f_{i_u} ({w}_u) - \nabla_{\mathcal{G}_{\psi(u)}} f_{i_u} (\widehat{w}_u)  \rangle
 + \frac{L_{\psi(u)} \gamma^2}{2} \mathbb{E} \|  \widehat{v}^{\psi(u)}_u  \|^2
 \\ \nonumber \stackrel{ (b) }{\leq}&  \mathbb{E}  f (w_{u}) -  \gamma \mathbb{E} \|  \nabla_{\mathcal{G}_{\psi(u)}} f ({w}_u) \|^2  + \frac{\gamma}{2} \mathbb{E} \| \nabla_{\mathcal{G}_{\psi(u)}} f ({w}_u) \|^2
 \\ \nonumber & +  \frac{\gamma}{2} \mathbb{E} \| \nabla_{\mathcal{G}_{\psi(u)}} f_{i_u} ({w}_u) - \nabla_{\mathcal{G}_{\psi(u)}} f_{i_u} (\widehat{w}_u) \|^2
 + \frac{L_{\psi(u)} \gamma^2}{2} \mathbb{E} \|  \widehat{v}^{\psi(u)}_u  \|^2
 \\ \nonumber \stackrel{ (c) }{\leq}&  \mathbb{E}  f (w_{u}) -  \frac{\gamma}{2} \mathbb{E} \|  \nabla_{\mathcal{G}_{\psi(u)}} f ({w}_u) \|^2
 +  \frac{\gamma L^2}{2} \mathbb{E} \| {w}_u - \widehat{w}_u \|^2
 + \frac{L_{\psi(u)} \gamma^2}{2} \mathbb{E} \|  \widehat{v}^{\psi(u)}_u  \|^2
  \\ \nonumber \stackrel{ (d) }{\leq}&  \mathbb{E}  f (w_{u}) -  \frac{\gamma}{2} \mathbb{E} \|  \nabla_{\mathcal{G}_{\psi(u)}} f ({w}_u) \|^2  + \frac{L_{\psi(u)} \gamma^2}{2} \mathbb{E} \|  \widehat{v}^{\psi(u)}_u  \|^2
 +  \frac{\eta_2 \gamma L^2 \gamma^2}{2}  \sum_{u' \in D(u)} \mathbb{E} \|   \textbf{U}_{\psi(u')} \widehat{v}^{\psi(u')}_{u'} \|^2
 \\ \nonumber \leq&  \mathbb{E}  f (w_{u}) -  \frac{\gamma}{2} \mathbb{E} \|  \nabla_{\mathcal{G}_{\psi(u)}} f ({w}_u) \|^2
 +  \frac{\eta_2 \gamma L^2 \gamma^2}{2}  \sum_{u' \in D(u)} \mathbb{E} \|   \widehat{v}^{\psi(u')}_{u'} \|^2 + \frac{L_{\psi(u)} \gamma^2}{2} \mathbb{E} \|  \widehat{v}^{\psi(u)}_u  \|^2
 \end{align}
where the  inequalities (a) and (c) use Assumption \ref{assumption3}, the inequality (b) uses the fact of $\langle a,b \rangle\leq \frac{1}{2}(\|a\|^2+\|b\|^2)$, the inequality (d) uses (\ref{pi_test1}).

Summing  (\ref{EqThm1_1}) over all $ u \in K(t)$, we obtain
\begin{align}\label{EqThm1_2}
& \mathbb{E} f (w_{t+|K(t)|}) - \mathbb{E}f (w_{t})
\\ \nonumber \leq&   -  \frac{\gamma}{2} \sum_{u \in K(t)} \mathbb{E} \|  \nabla_{\mathcal{G}_{\psi(u)}} f ({w}_u) \|^2   + \frac{L_{\max} \gamma^2}{2} \sum_{u \in K(t)}\mathbb{E} \|  \widehat{v}^{\psi(u)}_u  \|^2
 +  \frac{\eta_2 \gamma L^2 \gamma^2}{2}  \sum_{u \in K(t)}\sum_{u' \in D(u)} \mathbb{E} \|   \widehat{v}^{\psi(u')}_{u'} \|^2
\\ \nonumber \stackrel{ (a) }{\leq} &   -  \frac{\gamma}{2} \left ( \frac{1}{2} \sum_{u\in K(t)} \| \nabla_{\mathcal{G}_{\psi(u)}} f(w_t) \|^2   - \eta_1 \gamma^2 L^2  \sum_{u\in K(t)}   \sum_{v \in \{t,\ldots,u \}}\|\widehat{v}^{\psi(v)}_v \|^2 \right )
\\ \nonumber &  +  \frac{\eta_2 \gamma L^2 \gamma^2}{2}  \sum_{u \in K(t)}\sum_{u' \in D(u)} \mathbb{E} \|   \widehat{v}^{\psi(u')}_{u'} \|^2
+ \frac{L_{\max} \gamma^2}{2} \sum_{u \in K(t)}\mathbb{E} \|  \widehat{v}^{\psi(u)}_u  \|^2
 \\ \nonumber  = &   -  \frac{\gamma}{4} \sum_{u\in K(t)} \| \nabla_{\mathcal{G}_{\psi(u)}} f(w_t) \|^2  + \frac{L_{\max} \gamma^2}{2} \sum_{u \in K(t)}\mathbb{E} \|  \widehat{v}^{\psi(u)}_u  \|^2
 \\ & \nonumber + \frac{\eta_1 \gamma^3 L^2 }{2} \sum_{u\in K(t)}   \sum_{v \in \{t,\ldots,u \}}\|\widehat{v}^{\psi(v)}_v \|^2
  +  \frac{\eta_2 \gamma L^2 \gamma^2}{2}  \sum_{u \in K(t)}\sum_{u' \in D(u)} \mathbb{E} \|   \widehat{v}^{\psi(u')}_{u'} \|^2
  \\ \nonumber  \leq &   -  \frac{\gamma}{4}  \| \nabla f(w_t) \|^2 + \frac{\eta_1 \gamma^3 L^2 }{2} \sum_{u\in K(t)}   \sum_{v \in \{t,\ldots,u \}}\|\widehat{v}^{\psi(v)}_v \|^2
\\ \nonumber &  +  \frac{\eta_2 \gamma L^2 \gamma^2}{2}  \sum_{u \in K(t)}\sum_{u' \in D(u)} \mathbb{E} \|   \widehat{v}^{\psi(u')}_{u'} \|^2
+ \frac{L_{\max} \gamma^2}{b} \sum_{u \in K(t)}\mathbb{E} \|  \widehat{v}^{\psi(u)}_u  \|^2
   \\ \nonumber  \stackrel{ (b) }{\leq} &   -  \frac{\gamma \mu }{2}  \left (  f(w_t)-f(w^*) \right )
 + \left (\frac{\eta_1 \gamma^3 L^2 (q\eta_1)^2 }{2}+ \frac{\eta_2 \gamma L^2 \gamma^2 q\eta_1 \tau}{2} +  \frac{L_{\max} \gamma^2 q \eta_1}{2}  \right ) G
\\ \nonumber  = &   -  \frac{\gamma \mu }{2}  \left (  f(w_t)-f(w^*) \right )
 +  \underbrace{\frac{\eta_1 \gamma^2  q \left (  \gamma L^2 q \eta_1^2  +  \eta_2 \gamma L^2 \tau  +  L_{\max}  \right ) G}{2}}_{C}
\end{align}
where the  inequality (a) uses Lemma \ref{AsySPSAGA_lemma3}, the  inequality (b) uses Assumptions \ref{assumption4} and \ref{assumption1}.
According to (\ref{EqThm1_2}), we have that
\begin{align}\label{EqThm1_3}
 \mathbb{E} f (w_{t+|K(t)|}) -f(w^*)\leq  \left (1-  \frac{\gamma \mu }{2} \right ) \left (  f(w_t)-f(w^*) \right ) +  C
\end{align}
Assume that $\cup_{\kappa \in P(t)}=\{0,1,\ldots,t\}$, applying  (\ref{EqThm1_3}), we have that
\begin{align}\label{EqThm1_4}
\mathbb{E} f (w_{t}) -f(w^*)
  \leq & \left (1-  \frac{\gamma \mu }{2} \right )^{\upsilon(t)} \left (  f(w_0)-f(w^*) \right ) +  C \sum_{i=0}^{\upsilon(t)}\left (1-  \frac{\gamma \mu }{2} \right )^i
\\ \nonumber  \leq & \left (1-  \frac{\gamma \mu }{2} \right )^{\upsilon(t)} \left (  f(w_0)-f(w^*) \right ) +  C \sum_{i=0}^{\infty}\left (1-  \frac{\gamma \mu }{2} \right )^i
\\ \nonumber  = & \left (1-  \frac{\gamma \mu }{2} \right )^{\upsilon(t)} \left (  f(w_0)-f(w^*) \right )
 +   \frac{{\gamma \eta_1   q \left (  \gamma L^2 q \eta_1^2  +  \eta_2 \gamma L^2 \tau  +  L_{\max}  \right ) G}}{ \mu }
\end{align}
Let $ \frac{{\gamma \eta_1   q \left (  \gamma L^2 q \eta_1^2  +  \eta_2 \gamma L^2 \tau  +  L_{\max}  \right ) G}}{ \mu } \leq \frac{\epsilon}{2}$, we have that $\gamma \leq  \frac{-   L_{\max}   + \sqrt{  L_{\max}^2  + \frac{{2 \mu \epsilon  (L^2 q \eta_1^2  + \eta_2  L^2 \tau )}}{G\eta_1   q}} }{2 L^2 (q \eta_1^2  + \eta_2   \tau )}$.

Let $\left (1-  \frac{\gamma \mu }{2} \right )^{\upsilon(t)} \left (  f(w_0)-f(w^*) \right ) \leq \frac{\epsilon}{2}$, we have that
\begin{align}\label{EqThm1_5}
\log\left ( \frac{2 \left (  f(w_0)-f(w^*) \right )}{\epsilon} \right ) \leq \upsilon(t) \log \left ( \frac{1}{1-  \frac{\gamma \mu }{2} } \right )
\end{align}
Because $\log \left ( \frac{1}{\rho} \right ) \geq 1 -\rho$ for $0<\rho \leq 1$, we have that
\begin{align}\label{EqThm1_5}
\upsilon(t) \geq& \frac{2}{\gamma \mu } \log\left ( \frac{2 \left (  f(w_0)-f(w^*) \right )}{\epsilon} \right )
\\ \nonumber \geq& \frac{2}{ \mu } \frac{2 L^2 (q \eta_1^2  + \eta_2   \tau )}{-   L_{\max}   + \sqrt{  L_{\max}^2  + \frac{{2 \mu \epsilon  (L^2 q \eta_1^2  + \eta_2  L^2 \tau )}}{G\eta_1   q}} } \cdot
  \log\left ( \frac{2 \left (  f(w_0)-f(w^*) \right )}{\epsilon} \right )
\end{align}
This completes the proof.
\end{proof}

\subsection*{Appendix B: Proof of Theorem  \ref{theorem2}}

Before proving Theorem \ref{theorem2}, we first provide  Lemma \ref{AsySCGD+_lemma1} to provides an upper bound to $ \mathbb{E}  \left \|   \widehat{v}^{\ell}_u \right \|^2$.
\begin{lemma}\label{AsySCGD+_lemma1} For  AFSVRG-VP, under Assumptions \ref{assumption4} and \ref{assumption3},  let $u \in K(t)$,  we have that
\begin{align}
\label{AsySCGD+_lem3_0_1}
 \mathbb{E}  \left \|   \widehat{v}^{\ell}_u \right \|^2
 \leq &   \frac{16 L^2}{\mu}\mathbb{E} (f(w^s_t)-f(w^*))  +  \frac{8 L^2}{\mu} \mathbb{E} (f(w^s)-f(w^*))
\\ \nonumber    &+ 4 L^2 \gamma^2 \eta_1 \sum_{v \in \{t,\ldots,u \}} \mathbb{E} \left  \|   \widehat{v}^{\psi(v)}_v\right \|^2
 + 2 \eta_2 L^2 \gamma^2 \mathbb{E}  \sum_{u' \in D(u)}   \left \|    \widehat{v}^{\psi(u')}_{u'}  \right  \|^2
\end{align}
\end{lemma}
\begin{proof} Define ${v}^{\ell}_u= \nabla_{\mathcal{G}_\ell} f_i ({w}_u^s) - \nabla_{\mathcal{G}_\ell} f_i (w^s)
  +  \nabla_{\mathcal{G}_\ell} f(w^s)$. We have that
$
\label{AsySCGD+_lem3_0_2}  \mathbb{E}  \left \|   \widehat{v}^{ \ell }_u \right \|^2 =\mathbb{E}  \left \|   \widehat{v}^{ \ell }_u - {v}^{\ell}_u + {v}^{\ell}_u \right \|^2 \leq 2 \mathbb{E}   \left \|   \widehat{v}^{ \ell }_u - {v}^{\ell}_u \right  \|^2 +   2 \mathbb{E} \left  \| {v}^{\ell}_u \right \|^2
$. Firstly, we give the upper bound to $\mathbb{E} \left  \| {v}^{\ell}_u \right \|^2$ as follows.
\begin{align}\label{Lemma2_1}
& \mathbb{E} \left  \| {v}^{\ell}_u \right \|^2
\\ \nonumber =& \mathbb{E} \left  \|  \nabla_{\mathcal{G}_\ell} f_i ({w}_u^s) - \nabla_{\mathcal{G}_\ell} f_i (w^s)
  +  \nabla_{\mathcal{G}_\ell} f(w^s) \right \|^2
\\ \nonumber =&  \mathbb{E} \left  \|  \nabla_{\mathcal{G}_\ell} f_i ({w}_u^s) -\nabla_{\mathcal{G}_\ell} f_i ({w}^*) - \nabla_{\mathcal{G}_\ell} f_i (w^s)  +\nabla_{\mathcal{G}_\ell} f_i ({w}^*) +  \nabla_{\mathcal{G}_\ell} f(w^s) \right \|^2
 \\ \nonumber \leq&  2 \mathbb{E} \left  \|  \nabla_{\mathcal{G}_\ell} f_i ({w}_u^s) -\nabla_{\mathcal{G}_\ell} f_i ({w}^*)  \right \|^2
  +  2 \mathbb{E} \left  \| \nabla_{\mathcal{G}_\ell} f_i (w^s)
 -\nabla_{\mathcal{G}_\ell} f_i ({w}^*)  -  \nabla_{\mathcal{G}_\ell} f(w^s) +  \nabla_{\mathcal{G}_\ell} f(w^*) \right \|^2
  \\ \nonumber \stackrel{ (a) }{\leq}&  2 \mathbb{E} \left  \|  \nabla_{\mathcal{G}_\ell} f_i ({w}_u^s) -\nabla_{\mathcal{G}_\ell} f_i ({w}^*)  \right \|^2
  +  2 \mathbb{E} \left  \| \nabla_{\mathcal{G}_\ell} f_i (w^s)
 -\nabla_{\mathcal{G}_\ell} f_i ({w}^*) \right \|^2
 \\ \nonumber \stackrel{ (b) }{\leq}&  2 L^2\mathbb{E} \left  \|  {w}_u^s - {w}^*  \right \|^2 +  2 L^2 \mathbb{E} \left  \| w^s
 -{w}^* \right \|^2
  \\ \nonumber =&  2 L^2\mathbb{E} \left  \|  {w}_u^s - {w}_t^s + {w}_t^s - {w}^*  \right \|^2 +  2 L^2 \mathbb{E} \left  \| w^s
 -{w}^* \right \|^2
\\ \nonumber \leq &  4 L^2\mathbb{E} \left  \|  {w}_u^s - {w}_t^s \right \|^2  +  4 L^2\mathbb{E} \left  \|  {w}_t^s - {w}^*  \right \|^2
  +  2 L^2 \mathbb{E} \left  \| w^s
 -{w}^* \right \|^2
 \\ \nonumber \stackrel{ (c) }{=}&  4 L^2 \gamma^2 \mathbb{E} \left  \|  \sum_{v \in \{t,\ldots,u \}} \textbf{U}_{\psi(v)} \widehat{v}^{\psi(v)}_v\right \|^2  +  4 L^2\mathbb{E} \left  \|  {w}_t^s - {w}^*  \right \|^2
 +  2 L^2 \mathbb{E} \left  \| w^s
 -{w}^* \right \|^2
  \\ \nonumber \stackrel{ (d) }{\leq}  &   \frac{8 L^2}{\mu}\mathbb{E} (f(w^s_t)-f(w^*))  +  \frac{4 L^2}{\mu} \mathbb{E} (f(w^s)-f(w^*))
 + 4 L^2 \gamma^2 \eta_1 \sum_{v \in \{t,\ldots,u \}} \mathbb{E} \left  \|   \widehat{v}^{\psi(v)}_v\right \|^2
\end{align}
where the inequality (a) uses $\mathbb{E} \left \| x- \mathbb{E} x\right \|^2 \leq \mathbb{E} \left \| x\right \|^2$, the inequality (b) uses Assumption \ref{assumption3}, the equality (c) uses Eq. (\ref{pi_test1}), and the inequality (d) uses Assumption \ref{assumption4}.

Next, we give the upper bound to $\mathbb{E}   \left \|   \widehat{v}^{ \ell }_u - {v}^{\ell}_u \right  \|^2$ as follows.
\begin{align}\label{Lemma2_2}
  \mathbb{E}   \left \|   \widehat{v}^{ \ell }_u - {v}^{\ell}_u \right  \|^2
 =&\mathbb{E}   \left \|  \nabla_{\mathcal{G}_\ell} f_i (\widehat{w}_u^s)  - \nabla_{\mathcal{G}_\ell} f_i ({w}_u^s) \right  \|^2
\\ \nonumber \leq& L^2 \mathbb{E}   \left \| \widehat{w}_u^s  - {w}_u^s \right  \|^2
\\ \nonumber =& L^2 \gamma^2 \mathbb{E}   \left \|  \sum_{u' \in D(u)}    \textbf{U}_{\psi(u')} \widehat{v}^{\psi(u')}_{u'}  \right  \|^2
\\ \nonumber \leq & \eta_2 L^2 \gamma^2 \mathbb{E}  \sum_{u' \in D(u)}   \left \|    \widehat{v}^{\psi(u')}_{u'}  \right  \|^2
\end{align}
Combining  (\ref{Lemma2_1}) and (\ref{Lemma2_2}), we have that
\begin{align}\label{Lemma2_3}
  \mathbb{E}  \left \|   \widehat{v}^{ \ell }_u \right \|^2
 \leq &    2 \mathbb{E}   \left \|   \widehat{v}^{ \ell }_u - {v}^{\ell}_u \right  \|^2 +   2 \mathbb{E} \left  \| {v}^{\ell}_u \right \|^2
\\ \nonumber \leq &   \frac{16 L^2}{\mu}\mathbb{E} (f(w^s_t)-f(w^*))  +  \frac{8 L^2}{\mu} \mathbb{E} (f(w^s)-f(w^*))
\\ \nonumber    &+ 4 L^2 \gamma^2 \eta_1 \sum_{v \in \{t,\ldots,u \}} \mathbb{E} \left  \|   \widehat{v}^{\psi(v)}_v\right \|^2
 + 2 \eta_2 L^2 \gamma^2 \mathbb{E}  \sum_{u' \in D(u)}   \left \|    \widehat{v}^{\psi(u')}_{u'}  \right  \|^2
\end{align}
This completes the proof.
\end{proof}
Based on the basic inequalities in Lemma \ref{AsySPSAGA_lemma3}, we provide the proof of Theorem \ref{theorem2} in the following.

\begin{proof}
Similar to (\ref{EqThm1_1}),  for $u \in K(t)$ at $s$-th outer loop, we have that
\begin{align}\label{EqThm2_1}
& \mathbb{E} f (w_{u+1}^s)
\\ \nonumber \stackrel{ (a) }{\leq}&  \mathbb{E} \left ( f (w_{u}^s) + \langle \nabla f(w_{u}^s), w_{u+1}^s-w_{u}^s  \rangle  + \frac{L_{\psi(u)}}{2} \|w_{u+1}^s-w_{u}^s   \|^2  \right )
\\ \nonumber =&  \mathbb{E} \left ( f (w_{u}^s) -  \gamma \langle \nabla f(w_{u}^s),  \widehat{v}^{\psi(u)}_u  \rangle + \frac{L_{\psi(u)} \gamma^2}{2} \|  \widehat{v}^{\psi(u)}_u  \|^2  \right )
\\ \nonumber \stackrel{ (b) }{=}&  \mathbb{E}f (w_{u}^s) -  \gamma \mathbb{E} \langle \nabla f(w_{u}^s),  \nabla_{\mathcal{G}_{\psi(u)}} f_{i_u}(\widehat{w}_{u}^s)  \rangle
 + \frac{L_{\psi(u)} \gamma^2}{2} \mathbb{E} \|  \widehat{v}^{\psi(u)}_u  \|^2
\\ \nonumber =&  \mathbb{E} f (w_{u}^s) -  \gamma \mathbb{E} \langle \nabla f(w_{u}^s), \nabla_{\mathcal{G}_{\psi(u)}} f_{i_u}(\widehat{w}_{u}^s)
 - \nabla_{\mathcal{G}_{\psi(u)}} f_{i_u} ({w}_u^s) + \nabla_{\mathcal{G}_{\psi(u)}} f_{i_u} ({w}_u^s) \rangle  + \frac{L_{\psi(u)} \gamma^2}{2} \mathbb{E} \|  \widehat{v}^{\psi(u)}_u  \|^2
\\ \nonumber =&  \mathbb{E}  f (w_{u}^s) -  \gamma \mathbb{E} \langle \nabla f(w_{u}^s),  \nabla_{\mathcal{G}_{\psi(u)}} f ({w}_u^s) \rangle  + \frac{L_{\psi(u)} \gamma^2}{2} \mathbb{E} \|  \widehat{v}^{\psi(u)}_u  \|^2
\\ & \nonumber + \gamma \mathbb{E} \langle \nabla f(w_{u}^s),   \nabla_{\mathcal{G}_{\psi(u)}} f_{i_u} ({w}_u^s) - \nabla_{\mathcal{G}_{\psi(u)}} f_{i_u} (\widehat{w}_u^s)  \rangle
 \\ \nonumber \stackrel{ (c) }{\leq}&  \mathbb{E}  f (w_{u}^s) -  \frac{\gamma}{2} \mathbb{E} \|  \nabla_{\mathcal{G}_{\psi(u)}} f ({w}_u^s) \|^2
+  \frac{\eta_2 \gamma L^2 \gamma^2}{2}  \sum_{u' \in D(u)} \mathbb{E} \|   \widehat{v}^{\psi(u')}_{u'} \|^2
 + \frac{L_{\psi(u)} \gamma^2}{2} \mathbb{E} \|  \widehat{v}^{\psi(u)}_u  \|^2
 \end{align}
where the  inequalities (a)  use Assumption \ref{assumption3}, the equality (b) uses the fact that the  stochastic local gradient
   $\widehat{v}^\ell $ is unbiased, the inequality (c) follows the proof in (\ref{EqThm1_1}).

Summing  (\ref{EqThm2_1}) over all $ u \in K(t)$, we obtain
\begin{align}\label{EqThm2_2}
& \mathbb{E} f (w_{t+|K(t)|}^s) - \mathbb{E}f (w_{t}^s)
\\ \nonumber \leq&   -  \frac{\gamma}{2} \sum_{u \in K(t)} \mathbb{E} \|  \nabla_{\mathcal{G}_{\psi(u)}} f ({w}_u^s) \|^2 + \frac{L_{\max} \gamma^2}{2} \sum_{u \in K(t)}\mathbb{E} \|  \widehat{v}^{\psi(u)}_u  \|^2
\\ \nonumber &  +  \frac{\eta_2 \gamma L^2 \gamma^2}{2}  \sum_{u \in K(t)}\sum_{u' \in D(u)} \mathbb{E} \|   \widehat{v}^{\psi(u')}_{u'} \|^2
\\ \nonumber \stackrel{ (a) }{\leq} &   -  \frac{\gamma}{2} \left ( \frac{1}{2} \sum_{u\in K(t)} \| \nabla_{\mathcal{G}_{\psi(u)}} f(w_t^s) \|^2  - \eta_1 \gamma^2 L^2  \sum_{u\in K(t)}   \sum_{v \in \{t,\ldots,u \}}\|\widehat{v}^{\psi(v)}_v \|^2 \right )
\\ \nonumber &  +  \frac{\eta_2 \gamma L^2 \gamma^2}{2}  \sum_{u \in K(t)}\sum_{u' \in D(u)} \mathbb{E} \|   \widehat{v}^{\psi(u')}_{u'} \|^2
 + \frac{L_{\max} \gamma^2}{2} \sum_{u \in K(t)}\mathbb{E} \|  \widehat{v}^{\psi(u)}_u  \|^2
 \\ \nonumber  = &   -  \frac{\gamma}{4} \sum_{u\in K(t)} \| \nabla_{\mathcal{G}_{\psi(u)}} f(w_t^s) \|^2 + \frac{L_{\max} \gamma^2}{2} \sum_{u \in K(t)}\mathbb{E} \|  \widehat{v}^{\psi(u)}_u  \|^2
 \\ & \nonumber + \frac{\eta_1 \gamma^3 L^2 }{2} \sum_{u\in K(t)}   \sum_{v \in \{t,\ldots,u \}}\|\widehat{v}^{\psi(v)}_v \|^2
 +  \frac{\eta_2 \gamma L^2 \gamma^2}{2}  \sum_{u \in K(t)}\sum_{u' \in D(u)} \mathbb{E} \|   \widehat{v}^{\psi(u')}_{u'} \|^2
 \\ \nonumber  \stackrel{ (b) }{\leq} &   -  \frac{\gamma}{4} \| \nabla f(w_t^s) \|^2 +  \frac{\eta_2 \gamma L^2 \gamma^2}{2}  \sum_{u \in K(t)}\sum_{u' \in D(u)} \mathbb{E} \|   \widehat{v}^{\psi(u')}_{u'} \|^2
  + \left ( \frac{\eta_1 \gamma^3 L^2 q\eta_1 }{2}+ \frac{L_{\max} \gamma^2}{2} \right ) \sum_{u \in K(t)}\mathbb{E} \|  \widehat{v}^{\psi(u)}_u  \|^2
\\ \nonumber  \stackrel{ (c) }{\leq} &   -  \frac{\gamma \mu}{2} \mathbb{E} (f(w^s_t)-f(w^*))
 +  \frac{\eta_2 \gamma L^2 \gamma^2}{2}  \sum_{u \in K(t)}\sum_{u' \in D(u)} \mathbb{E} \|   \widehat{v}^{\psi(u')}_{u'} \|^2
\\ \nonumber &  +C\sum_{u \in K(t)}\left (   \frac{16 L^2}{\mu}\mathbb{E} (f(w^s_t)-f(w^*))   \right .
\\ \nonumber    & \left . +  \frac{8 L^2}{\mu} \mathbb{E} (f(w^s)-f(w^*)) + 4 L^2 \gamma^2 \eta_1 \sum_{v \in \{t,\ldots,u \}} \mathbb{E} \left  \|   \widehat{v}^{\psi(v)}_v\right \|^2   + 2 \eta_2 L^2 \gamma^2 \mathbb{E}  \sum_{u' \in D(u)}   \left \|    \widehat{v}^{\psi(u')}_{u'}  \right  \|^2  \right )
\end{align}
where $C=\left (\eta_1 \gamma L^2 q\eta_1 + L_{\max}  \right )\frac{\gamma^2}{2}$, the  inequality (a) uses Lemma \ref{AsySPSAGA_lemma3}, the  inequality (b) uses Assumption  \ref{assMax}, the inequality (c) uses Lemma \ref{AsySCGD+_lemma1}.

Let $e^s_t=\mathbb{E} (f(w^s_t)-f(w^*))$ and $e^s=\mathbb{E} (f(w^s)-f(w^*))$, we have
\begin{align}\label{EqThm2_3}
& e_{t+|K(t)|}^s
\\ \nonumber  \stackrel{ (a) }{\leq}  & \left ( 1 - \frac{\gamma \mu}{2} +  \frac{16 L^2 \eta_1 q C}{\mu}  \right ) e^s_t +  \frac{8 L^2 \eta_1 q C}{\mu} e^s
\\ & \nonumber+ 4 C    L^2 \gamma^2 \eta_1  \sum_{u \in K(t)}   \sum_{v \in \{t,\ldots,u \}} \mathbb{E} \left  \|   \widehat{v}^{\psi(v)}_v\right \|^2
  +  \left ( \frac{\eta_2 \gamma L^2 \gamma^2}{2} + 2 \eta_2 L^2 \gamma^2 C \right ) \sum_{u \in K(t)}\sum_{u' \in D(u)} \mathbb{E} \|   \widehat{v}^{\psi(u')}_{u'} \|^2
\\ \nonumber  \stackrel{ (b) }{\leq}  & \left ( 1 - \frac{\gamma \mu}{2} +  \frac{16 L^2 \eta_1 q C}{\mu}  \right ) e^s_t +  \frac{8 L^2 \eta_1 q C}{\mu} e^s
  + \gamma^3 \left (  \left ( \frac{ 1  }{2} +     \frac{2C}{\gamma} \right )  \eta_2  \tau    + 4 \frac{C}{\gamma}     \eta_1^2 q  \right ) \eta_1 q  L^2 G
\end{align}
where the   inequality (a) uses (\ref{EqThm2_2}), and the inequality (b) uses Assumption \ref{assumption1}.
We carefully choose $\gamma$ such that $ \frac{\gamma \mu}{2} -  \frac{16 L^2 \eta_1 q C}{\mu}\stackrel{\rm def}{=}\rho>0$. Assume that $\cup_{\kappa \in P(t)}=\{0,1,\ldots,t\}$, applying  (\ref{EqThm2_3}), we have that
\begin{align}\label{EqThm2_4}
 e_{t}^s
  \leq & \left ( 1 - \rho \right )^{\upsilon(t)}  e^s  +   \left ( \frac{8 L^2 \eta_1 q C}{\mu} e^s \right .
\\ & \nonumber \left . + \gamma^3 \left (  \left ( \frac{ 1  }{2} +     \frac{2C}{\gamma} \right )  \eta_2  \tau    + 4 \frac{C}{\gamma}     \eta_1^2 q  \right ) \eta_1 q  L^2 G \right ) \sum_{i=0}^{{\upsilon(t)}}\left (1-  \rho \right )^i
\\ \nonumber  \leq & \left ( 1 - \rho \right )^{\upsilon(t)}  e^s  +   \left ( \frac{8 L^2 \eta_1 q C}{\mu} e^s  + \gamma^3 \left (  \left ( \frac{ 1  }{2} +     \frac{2C}{\gamma} \right )  \eta_2  \tau    + 4 \frac{C}{\gamma}     \eta_1^2 q  \right ) \eta_1 q  L^2 G \right ) \frac{1}{\rho}
 \\ \nonumber  = & \left ( \left ( 1 - \rho \right )^{\upsilon(t)} + \frac{8 L^2 \eta_1 q C}{\rho \mu} \right ) e^s
 +
 \gamma^3 \left (  \left ( \frac{ 1  }{2} +     \frac{2C}{\gamma} \right )  \eta_2  \tau    + 4 \frac{C}{\gamma}     \eta_1^2 q  \right )   \frac{\eta_1 q  L^2 G}{\rho}
\end{align}
Thus, to achieve the accuracy $\epsilon$ of (\ref{formulation1}) for AFSVRG-VP, \emph{i.e.}, $\mathbb{E} f (w_{S}) -f(w^*) \leq \epsilon$,  we can carefully choose $\gamma$ such that
\begin{align}
\frac{8 L^2 \eta_1 q C}{\rho \mu} \leq& 0.5
\\  \gamma^3 \left (  \left ( \frac{ 1  }{2} +     \frac{2C}{\gamma} \right )  \eta_2  \tau    + 4 \frac{C}{\gamma}     \eta_1^2 q  \right )   \frac{\eta_1 q  L^2 G}{\rho} \leq& \frac{\epsilon}{8}
\end{align}
and let $\left ( 1 - \rho \right )^{\upsilon(t)} \leq 0.25$, i.e., $\upsilon(t) \geq \frac{\log 0.25}{\log (1 - \rho)}$, we have that
\begin{eqnarray}\label{EqThm2_5}
e^{s+1} \leq 0.75 e^{s} + \frac{\epsilon}{8}
\end{eqnarray}
Recursively apply (\ref{EqThm2_5}), we have that
\begin{eqnarray}
e^{S} \leq (0.75)^S e^{0}+ \frac{\epsilon}{2}
\end{eqnarray}
Finally, the outer loop  number $S$ should satisfy the  condition of $S \geq \frac{\log \frac{2 e^0}{\epsilon }}{\log \frac{4}{3}} $.
This completes the proof.
\end{proof}

\subsection*{Appendix C: Proof of Theorem  \ref{theorem3}}
Before proving Theorem \ref{theorem3}, we first provide  Lemma \ref{lemma3} to provides an upper bound to $ \mathbb{E}  \left \|   \widehat{v}^{\ell}_u \right \|^2$.
\begin{lemma} \label{AsySGHT_lemma2} For AFSAGA-VP,  we have that
\begin{align}\label{lem4_0}
 & \mathbb{E} \left \|  \alpha_{i_u}^{u,\ell} -  \nabla_{\mathcal{G}_\ell} f_{i_u}(w^*) \right \|^2
   \leq
\frac{L^2}{l} \sum_{u'=1}^{\xi(u,\ell)-1}  \cdot
\\   \nonumber &  \left ( 1 -\frac{1}{l} \right )^{\xi(u,\ell)-u'-1}  \sigma(w_{{\xi^{-1}(u',\ell)}}) + L^2 \left ( 1 -\frac{1}{l} \right )^{\xi(u,\ell)} \sigma(w_0)
\\  &  \mathbb{E}\left \|  {\alpha}_{i_u}^{u,\ell} - \widehat{\alpha}_{i_u}^{u,\ell} \right \|^2 \leq \frac{\eta_2 L^2 \gamma^2}{l} \sum_{u'=1}^{\xi(u,\ell)-1}  \sum_{\widetilde{u} \in D(\xi^{-1}(u',\ell))} \cdot
  \left ( 1 -\frac{1}{l} \right )^{\xi(u,\ell)-u'-1} \mathbb{E}  \left \|       \widehat{v}^{\psi(\widetilde{u})}_{\widetilde{u}} \right \|^2 \nonumber
\end{align}
where $\sigma(w_u) = \mathbb{E} \| w_u - w^* \|^2$.
\end{lemma}
\begin{proof} Firstly, we have that
\begin{align}\label{lem4_1}
& \mathbb{E} \left \| \alpha_{i_u}^{u,\ell} -  \nabla_{\mathcal{G}_\ell} f_{i_u}(w^*)  \right \|^2
\\ \nonumber  =&  \frac{1}{l} \sum_{i=1}^l \mathbb{E} \left \|  \alpha_{i}^{u,\ell} -  \nabla_{\mathcal{G}_\ell} f_{i}(w^*) \right \|^2
\\    = & \nonumber \frac{1}{l} \sum_{i=1}^l   \mathbb{E} \sum_{u'=0}^{\xi(u,\ell) -1} \mathbf{1}_{ \{ \textbf{u}_{i}^u =u' \}} \left \|   \nabla_{\mathcal{G}_\ell} f_i(w_{\xi^{-1}(u',\ell)}) -  \nabla_{\mathcal{G}_\ell} f_i(w^*) \right \|^2
\\    = & \nonumber \frac{1}{l} \sum_{u'=0}^{\xi(u,\ell)-1} \sum_{i=1}^l  \mathbb{E}  \mathbf{1}_{ \{ \textbf{u}_{i}^u =u' \}} \left \|  \nabla_{\mathcal{G}_\ell} f_i(w_{\xi^{-1}(u',\ell)}) -  \nabla_{\mathcal{G}_\ell} f_i(w^*) \right \|^2
\end{align}
where  $\textbf{u}_{i}^u$ denote the last iterate  to update the $\widehat{\alpha}_i^{u,\ell}$. We consider the two cases $u'>0$ and $u'=0$ as following.

For $u'>0$,  we have that
\begin{align}\label{lem4_2}
& \mathbb{E} \left ( \mathbf{1}_{ \{ \textbf{u}_{i}^u =u' \}} \left \|   \nabla_{\mathcal{G}_\ell} f_i(w_{{\xi^{-1}(u',\ell)}}) -  \nabla_{\mathcal{G}_\ell} f_i(w^*) \right \|^2 \right )
\\    \stackrel{ (a) }{\leq}  & \nonumber \mathbb{E} \left ( \mathbf{1}_{ \{ i_{u'} = i \}} \mathbf{1}_{ \{ i_v \neq i, \forall v \ s.t. \ u'+1 \leq v \leq \xi(u,\ell) -1 \}}  \right .
\\ & \nonumber  \left .
 \left \|   \nabla_{\mathcal{G}_\ell} f_i(w_{{\xi^{-1}(u',\ell)}}) -  \nabla_{\mathcal{G}_\ell} f_i(w^*) \right \|^2 \right )
\\    \stackrel{ (b) }{\leq}  & \nonumber  P{ \{ i_{u'} = i \}}  P { \{ i_v \neq i, \forall v \ s.t. \ u'+1 \leq v \leq \xi(u,\ell) -1 \}} \cdot
\\ & \nonumber  \mathbb{E}   \left \|   \nabla_{\mathcal{G}_\ell} f_i(w_{{\xi^{-1}(u',\ell)}}) -  \nabla_{\mathcal{G}_\ell} f_i(w^*) \right \|^2
\\    \stackrel{ (c) }{\leq}  & \nonumber \frac{1}{l} \left ( 1 -\frac{1}{l} \right )^{\xi(u,\ell)-u'-1}  \mathbb{E}  \left \|  \nabla_{\mathcal{G}_\ell} f_i(w_{{\xi^{-1}(u',\ell)}}) -  \nabla_{\mathcal{G}_\ell} f_i(w^*) \right \|^2
\end{align}
where the inequality (a) uses the fact $i_{u'}$ and $i_v$ are independent for $v \neq u'$, the inequality (b) uses the fact that $P{ \{ i_u = i \}} = \frac{1}{l}$ and $P { \{ i_v \neq i\} } =1-\frac{1}{l}$.

For $u'=0$, we have that
\begin{align}\label{lem4_3}
& \mathbb{E} \left ( \mathbf{1}_{ \{ \textbf{u}_{i}^u =0 \}}\left \|   \nabla_{\mathcal{G}_\ell} f_i(w_{0}) -  \nabla_{\mathcal{G}_\ell} f_i(w^*) \right \|^2   \right )
\\    \leq & \nonumber \mathbb{E} \left ( \mathbf{1}_{ \{ i_v \neq i, \forall v \ s.t. \ 0 \leq v \leq \xi(u,\ell)-1 \}}
 \left \|   \nabla_{\mathcal{G}_\ell} f_i(w_{0}) -  \nabla_{\mathcal{G}_\ell} f_i(w^*) \right \|^2 \right )
\\    \leq & \nonumber   P { \{ i_v \neq i, \forall v \ s.t. \ 0 \leq v \leq \xi(u,\ell) -1 \}}
  \mathbb{E}  \left \|   \nabla_{\mathcal{G}_\ell} f_i(w_{0}) -  \nabla_{\mathcal{G}_\ell} f_i(w^*) \right \|^2
\\    \leq & \nonumber  \left ( 1 -\frac{1}{l} \right )^{\xi(u,\ell)}  \mathbb{E}  \left \|   \nabla_{\mathcal{G}_\ell} f_i(w_{0}) -  \nabla_{\mathcal{G}_\ell} f_i(w^*) \right \|^2
\end{align}

Substituting (\ref{lem4_2}) and (\ref{lem4_3}) into (\ref{lem4_1}), we have that
\begin{align}\label{lem4_4}
& \mathbb{E} \left \| \alpha_{i_u}^{u,\ell} -  \nabla_{\mathcal{G}_\ell} f_{i_u}(w^*)  \right \|^2
\\    = & \nonumber \frac{1}{l} \sum_{u'=0}^{\xi(u,\ell) -1} \sum_{i=1}^l  \mathbb{E}  \mathbf{1}_{ \{ \textbf{u}_{i}^u =u' \}} \left \|  \nabla_{\mathcal{G}_\ell} f_i(w_{{\xi^{-1}(u',\ell)}}) -  \nabla_{\mathcal{G}_\ell} f_i(w^*) \right \|^2
\\   \stackrel{ (a) }{\leq} & \nonumber \frac{1}{l} \sum_{u'=1}^{\xi(u,\ell)-1} \sum_{i=1}^l \left ( \frac{1}{l} \left ( 1 -\frac{1}{l} \right )^{\xi(u,\ell)-u'-1} \cdot \right .
\\ & \nonumber \left . \mathbb{E}  \left \|  \nabla_{\mathcal{G}_\ell} f_i(w_{{\xi^{-1}(u',\ell)}}) -  \nabla_{\mathcal{G}_\ell} f_i(w^*) \right \|^2 +  \frac{1}{l}  \sum_{i=1}^l   \left ( 1 -\frac{1}{l} \right )^{\xi(u,\ell)-1}  \mathbb{E}  \left \|  \nabla_{\mathcal{G}_\ell} f_i(w_{0}) -  \nabla_{\mathcal{G}_\ell} f_i(w^*) \right \|^2 \right )
\\    = & \nonumber \frac{1}{l} \sum_{u'=1}^{\xi(u,\ell)-1} \left ( 1 -\frac{1}{l} \right )^{\xi(u,\ell)-u'-1} \cdot
\\ & \nonumber \mathbb{E}  \left \|  \nabla_{\mathcal{G}_\ell} f_i(w_{{\xi^{-1}(u',\ell)}}) -  \nabla_{\mathcal{G}_\ell} f_i(w^*) \right \|^2
 +  \left ( 1 -\frac{1}{l} \right )^{\xi(u,\ell)-1}  \mathbb{E}  \left \|  \nabla_{\mathcal{G}_\ell} f_i(w_{0}) -  \nabla_{\mathcal{G}_\ell} f_i(w^*) \right \|^2
\\    \stackrel{ (b) }{\leq} & \nonumber \frac{L^2}{l} \sum_{u'=1}^{\xi(u,\ell)-1} \left ( 1 -\frac{1}{l} \right )^{\xi(u,\ell)-u'-1} \sigma(w_{{\xi^{-1}(u',\ell)}})
 + L^2 \left ( 1 -\frac{1}{l} \right )^{\xi(u,\ell)} \sigma(w_0)
\end{align}
where the  inequality (a) uses (\ref{lem4_2}) and (\ref{lem4_3}),  the  inequality (b) uses Assumption \ref{assumption4}.

Similarly, we have that
\begin{align}\label{lem4_5}
&  \mathbb{E}\left \|  {\alpha}_{i_u}^{u,\ell} - \widehat{\alpha}_{i_u}^{u,\ell} \right \|^2
=  \frac{1}{l} \sum_{i=1}^l \mathbb{E} \left \| {\alpha}_{i}^{u,\ell} - \widehat{\alpha}_{i}^{u,\ell} \right \|^2
\\ \nonumber   = & \frac{1}{l} \sum_{u'=0}^{\xi(u,\ell)-1}  \sum_{i=1}^l \mathbb{E} \mathbf{1}_{ \{ \textbf{u}_{i}^u =u'\} }   \left \| {\alpha}_{i}^{u',\ell} - \widehat{\alpha}_{i}^{u',\ell} \right \|^2
\\    \stackrel{ (a) }{\leq} & \nonumber \frac{1}{l} \sum_{u'=1}^{\xi(u,\ell)-1} \sum_{i=1}^l \left ( \frac{1}{l} \left ( 1 -\frac{1}{l} \right )^{\xi(u,\ell)-u'-1}  \mathbb{E}  \left \| {\alpha}_{i}^{u',\ell} - \widehat{\alpha}_{i}^{u',\ell} \right \|^2 \right .
\\      & \nonumber \left . +  \frac{1}{l}  \sum_{i=1}^l   \left ( 1 -\frac{1}{l} \right )^{\xi(u,\ell)-1}  \mathbb{E}   \left \| {\alpha}_{i}^{0,\ell} - \widehat{\alpha}_{i}^{0,\ell} \right \|^2 \right )
\\   \stackrel{ (b) }{=} & \nonumber \frac{1}{l} \sum_{u'=1}^{\xi(u,\ell)-1} \sum_{i=1}^l \left ( \frac{1}{l} \left ( 1 -\frac{1}{l} \right )^{\xi(u,\ell)-u'-1}  \mathbb{E}  \left \| {\alpha}_{i}^{u',\ell} - \widehat{\alpha}_{i}^{u',\ell} \right \|^2 \right )
\\   = & \nonumber \frac{1}{l} \sum_{u'=1}^{\xi(u,\ell)-1} \left ( 1 -\frac{1}{l} \right )^{\xi(u,\ell)-u'-1} \cdot
 \mathbb{E}  \left \|  \nabla_{\mathcal{G}_\ell} f_i(\widehat{w}_{\xi^{-1}(u',\ell)}) -  \nabla_{\mathcal{G}_\ell} f_i(w_{\xi^{-1}(u',\ell)}) \right \|^2
\\    \stackrel{ (c) }{\leq} & \nonumber \frac{L^2}{l} \sum_{u'=1}^{\xi(u,\ell)-1} \left ( 1 -\frac{1}{l} \right )^{\xi(u,\ell)-u'-1} \mathbb{E}  \left \|  \widehat{w}_{\xi^{-1}(u',\ell)} - w_{\xi^{-1}(u',\ell)} \right \|^2
\\    = & \nonumber \frac{L^2 \gamma^2}{l} \sum_{u'=1}^{\xi(u,\ell)-1} \left ( 1 -\frac{1}{l} \right )^{\xi(u,\ell)-u'-1}  \cdot
  \mathbb{E}  \left \|   \sum_{\widetilde{u} \in D(\xi^{-1}(u',\ell))}    \textbf{U}_{\psi(\widetilde{u})} \widehat{v}^{\psi(\widetilde{u})}_{\widetilde{u}} \right \|^2
\\    \stackrel{ (d) }{\leq} & \nonumber \frac{\eta_2 L^2 \gamma^2}{l} \sum_{u'=1}^{\xi(u,\ell)-1}  \sum_{\widetilde{u} \in D(\xi^{-1}(u',\ell))} \left ( 1 -\frac{1}{l} \right )^{\xi(u,\ell)-u'-1} \cdot
\mathbb{E}  \left \|       \widehat{v}^{\psi(\widetilde{u})}_{\widetilde{u}} \right \|^2
\end{align}
where the inequality (a) can be obtained similar to (\ref{lem4_4}), the equality (b) uses the fact of ${\alpha}_{i}^{0,\ell} = \widehat{\alpha}_{i}^{0,\ell}$, the inequality (c) uses Assumption \ref{assumption3}, and the  inequality (d) uses Assumption \ref{ass5}.
 This completes the proof.
\end{proof}

\begin{lemma} \label{The3lemma2} Given a global time counter $u$, we let $\{ \overline{u}_0,\overline{u}_1, \ldots,\overline{u}_{\upsilon(u)-1}\}$ be the  all start time counters for the global time counters from 0 to $u$. Thus, for AFSAGA-VP,   we have that
\begin{align}\label{lem4_0}
& \mathbb{E} \left  \| {v}^{\ell}_{u} \right \|^2
\\ \nonumber \leq&
 4\frac{L^2 \eta_1}{l} \sum_{k'=1}^{\upsilon(u)} \left ( 1 -\frac{1}{l} \right )^{\upsilon(u)-k'} \sigma(w_{\overline{u}_{k'}}) + 8 L^2 \gamma^2 \eta_1^2 q G
 + 2 L^2 \left ( 1 -\frac{1}{l} \right )^{\upsilon(u)}  \sigma(w_0)
 + 4 L^2 \sigma(  w_{\varphi(u)})
\end{align}
where $\sigma(w_u) = \mathbb{E} \| w_u - w^* \|^2$.
\end{lemma}
\begin{proof}
We have that
\begin{align}\label{Lemma3_1}
 & \mathbb{E} \left  \| {v}^{\ell}_{u} \right \|^2
\\   = & \nonumber \mathbb{E} \left \| \nabla_{\mathcal{G}_\ell} f_{i_u} (w_u) - \alpha_{i_u}^{u,\ell} +  \frac{1}{l} \sum_{i=1}^l \alpha_{i}^{\ell} \right \|^2
\\   = & \nonumber \mathbb{E} \left \| \nabla_{\mathcal{G}_\ell} f_{i_u} (w_u) - \nabla_{\mathcal{G}_\ell} f_{i_u}(w^*) - \alpha_{i_u}^{u,\ell} + \nabla_{\mathcal{G}_\ell} f_{i_u}(w^*) \right .
\\ & \nonumber \left . + \frac{1}{l} \sum_{i=1}^l \alpha_{i}^{u,\ell} - \nabla_{\mathcal{G}_\ell} f(w^*) + \nabla_{\mathcal{G}_\ell} f(w^*) \right \|^2
\\   \stackrel{ (a) }{\leq} &  \nonumber 2 \mathbb{E} \left \|  \nabla_{\mathcal{G}_\ell} f_{i_u}(w^*) - \alpha_{i_u}^{u,\ell} + \frac{1}{l} \sum_{i=1}^l \alpha_{i}^{u,\ell} - \nabla_{\mathcal{G}_\ell} f(w^*) \right \|^2
\\ & \nonumber  + 2 \mathbb{E} \left \| \nabla_{\mathcal{G}_\ell} f_i (w_u) - \nabla_{\mathcal{G}_\ell} f_{i_t}(w^*) \right \|^2
\\  \stackrel{ (b) }{\leq} &  \nonumber 2 \mathbb{E} \left \| \alpha_{i_u}^{u,\ell} - \nabla_{\mathcal{G}_\ell} f_{i_t}(w^*) \right \|^2
+ 2 \mathbb{E} \left \| \nabla_{\mathcal{G}_\ell} f_i (w_u) - \nabla_{\mathcal{G}_\ell} f_{i_t}(w^*) \right \|^2
\\   \stackrel{ (c) }{\leq} &  \nonumber  2\frac{L^2}{l} \sum_{u'=1}^{\xi(u,\ell)-1} \left ( 1 -\frac{1}{l} \right )^{\xi(u,\ell)-u'-1} \mathbb{E} \| w_{{\xi^{-1}(u',\ell)}} - w^* \|^2
\\ & \nonumber   + 2 L^2 \left ( 1 -\frac{1}{l} \right )^{\xi(u,\ell)}  \sigma(w_0)
+ 2 L^2 \mathbb{E} \| w_u - w^* \|^2
\\  = &  \nonumber  2\frac{L^2}{l} \sum_{u'=1}^{\xi(u,\ell)-1} \left ( 1 -\frac{1}{l} \right )^{\xi(u,\ell)-u'-1} \cdot
\\ & \nonumber \mathbb{E} \| w_{{\xi^{-1}(u',\ell)}} - w_{\varphi({\xi^{-1}(u',\ell)})}+ w_{\varphi({\xi^{-1}(u',\ell)})} - w^* \|^2
\\  \nonumber & + 2 L^2 \left ( 1 -\frac{1}{l} \right )^{\xi(u,\ell)}  \sigma(w_0)
 + 2 L^2 \mathbb{E} \| w_u - w_{\varphi(u)}+ w_{\varphi(u)} - w^* \|^2
\\   \stackrel{ (d) }{\leq} &  \nonumber  2\frac{L^2}{l} \sum_{u'=1}^{\xi(u,\ell)-1} \left ( 1 -\frac{1}{l} \right )^{\xi(u,\ell)-u'-1} \cdot
\\ & \nonumber \mathbb{E} \left ( 2\| w_{{\xi^{-1}(u',\ell)}} - w_{\varphi({\xi^{-1}(u',\ell)})} \|^2 + 2\|  w_{\varphi({\xi^{-1}(u',\ell)})} - w^* \|^2 \right )
\\  \nonumber  & + 2 L^2 \left ( 1 -\frac{1}{l} \right )^{\upsilon(u)}  \sigma(w_0)+ 4 L^2 \mathbb{E} \|  w_{\varphi(u)} - w^* \|^2
 + 4 L^2 \gamma^2 \mathbb{E} \left  \|  \sum_{v \in \{{\varphi(u)},\ldots,u \}} \textbf{U}_{\psi(v)} \widehat{v}^{\psi(v)}_v  \right \|^2
\\   \stackrel{ (e) }{\leq} &  \nonumber  2\frac{L^2}{l} \sum_{u'=1}^{\xi(u,\ell)-1} \left ( 1 -\frac{1}{l} \right )^{\xi(u,\ell)-u'-1} \mathbb{E} \left ( 2 \eta_1 \gamma^2 \cdot    \sum_{v \in \{{\varphi({\xi^{-1}(u',\ell)})},\ldots,{{\xi^{-1}(u',\ell)}} \}} \| \widehat{v}^{\psi(v)}_v   \|^2 \right .
\\ & \nonumber \left . + 2\|  w_{\varphi({\xi^{-1}(u',\ell)})} - w^* \|^2 \right )
\\  \nonumber & + 2 L^2 \left ( 1 -\frac{1}{l} \right )^{\upsilon(u)}  \sigma(w_0)
 + 4 L^2 \mathbb{E} \|  w_{\varphi(u)} - w^* \|^2
 + 4 L^2 \gamma^2 \eta_1 \sum_{v \in \{{\varphi(u)},\ldots,u \}} \mathbb{E} \left  \|   \widehat{v}^{\psi(v)}_v  \right \|^2
\\    \stackrel{ (f) }{\leq} &  \nonumber  2\frac{L^2}{l} \sum_{u'=1}^{\xi(u,\ell)-1} \left ( 1 -\frac{1}{l} \right )^{\xi(u,\ell)-u'-1} \cdot
\mathbb{E} \left ( 2 \eta_1^2 \gamma^2 q G + 2\|  w_{\varphi({\xi^{-1}(u',\ell)})} - w^* \|^2 \right )
\\  \nonumber  & + 2 L^2 \left ( 1 -\frac{1}{l} \right )^{\upsilon(u)}  \sigma(w_0)
  + 4 L^2 \mathbb{E} \|  w_{\varphi(u)} - w^* \|^2 + 4 L^2 \gamma^2 \eta_1^2 q G
\\    \stackrel{ (g) }{\leq} &  \nonumber  4\frac{L^2}{l} \sum_{u'=1}^{\xi(u,\ell)-1} \left ( 1 -\frac{1}{l} \right )^{\xi(u,\ell)-u'-1} \mathbb{E}  \|  w_{\varphi({\xi^{-1}(u',\ell)})} - w^* \|^2
\\  \nonumber & + 2 L^2 \left ( 1 -\frac{1}{l} \right )^{\upsilon(u)}  \sigma(w_0)
 + 4 L^2 \mathbb{E} \|  w_{\varphi(u)} - w^* \|^2
  + 8 L^2 \gamma^2 \eta_1^2 q G
\\    \stackrel{ (h) }{\leq} &  \nonumber  4\frac{L^2 \eta_1}{l} \sum_{k'=1}^{\upsilon(u)} \left ( 1 -\frac{1}{l} \right )^{\upsilon(u)-k'} \sigma(w_{\overline{u}_{k'}}) + 8 L^2 \gamma^2 \eta_1^2 q G
 + 2 L^2 \left ( 1 -\frac{1}{l} \right )^{\upsilon(u)}  \sigma(w_0)
 + 4 L^2 \sigma(  w_{\varphi(u)})
\end{align}
where the  inequalities (a) and (d) uses $\| \sum_{i=1}^n a_i \|^2 \leq n \sum_{i=1}^n \| a_i \|^2 $, the  inequality (b) follows from $\mathbb{E} \left \| x- \mathbb{E} x\right \|^2 \leq \mathbb{E} \left \| x\right \|^2$,  the  inequality (c) uses Lemma \ref{AsySGHT_lemma2}, the inequality (e) uses Assumption \ref{assMax}, the the inequality (f) uses Assumption \ref{assumption1}, and the  inequality (g) uses the fact $ \sum_{u'=1}^{\xi(u,\ell)-1} \left ( 1 -\frac{1}{l} \right )^{\xi(u,\ell)-u'-1} \leq l$.
%
%
This
completes the proof.
\end{proof}

\begin{lemma} \label{lemma3} For  AFSVRG-VP, under Assumptions \ref{assumption4} and \ref{assumption3},  let $u \in K(t)$,  we have that
\begin{align}
\label{lemma3_0_1}  \mathbb{E} \left  \| \widehat{v}^{\ell}_u \right \|^2
\leq &  6 \eta_1 L^2 \gamma^2 \sum_{u' \in D(u)} \mathbb{E}  \left \|       \widehat{v}^{\psi(u')}_{u'} \right \|^2 + \frac{12 \eta_2 L^2 \gamma^2}{l} \sum_{u'=1}^{\xi(u,\ell)-1}  \cdot
\\ & \nonumber \sum_{\widetilde{u} \in D(\xi^{-1}(u',\ell))} \left ( 1 -\frac{1}{l} \right )^{\xi(u,\ell)-u'-1} \mathbb{E}  \left \|       \widehat{v}^{\psi(\widetilde{u})}_{\widetilde{u}} \right \|^2
 +  2 \mathbb{E} \left \|  v_{u}^\ell  \right \|^2
\end{align}
\end{lemma}
\begin{proof} Define ${v}^{\ell}_u= \nabla_{\mathcal{G}_\ell} f_i (w_u) - \alpha_i^{\ell} +  \frac{1}{l} \sum_{i=1}^l \alpha_i^{\ell}$. We have that
$
\label{AsySCGD+_lem3_0_2}  \mathbb{E}  \left \|   \widehat{v}^{ \ell }_u \right \|^2 =\mathbb{E}  \left \|   \widehat{v}^{ \ell }_u - {v}^{\ell}_u + {v}^{\ell}_u \right \|^2 \leq 2 \mathbb{E}   \left \|   \widehat{v}^{ \ell }_u - {v}^{\ell}_u \right  \|^2 +   2 \mathbb{E} \left  \| {v}^{\ell}_u \right \|^2
$.

Next, we give the upper bound to $\mathbb{E}   \left \|   \widehat{v}^{ \ell }_u - {v}^{\ell}_u \right  \|^2$ as follows.
Next, we have that
\begin{align}\label{AsySGHT_lem4_11}
&   \mathbb{E} \left \| \widehat{v}_{u}^\ell -  v_{u}^\ell  \right \|^2
\\    = & \nonumber  \mathbb{E} \left \| \nabla_{\mathcal{G}_\ell}  f_{i_u}(\widehat{w}_u)- \widehat{\alpha}_{i_u}^{u,\ell}  + \frac{1}{l} \sum_{i=1}^l \widehat{\alpha}_{i}^{u,\ell}  - \nabla_{\mathcal{G}_\ell} f_{i_u}(w_t)   + \alpha_{i_u}^{u,\ell} - \frac{1}{l} \sum_{i=1}^l \alpha_{i}^{u,\ell} \right \|^2
\\    \stackrel{ (a) }{\leq} & \nonumber  3  \mathbb{E} \underbrace{\left \| \nabla_{\mathcal{G}_\ell}  f_{i_u}(\widehat{w}_u) -\nabla_{\mathcal{G}_\ell}  f_{i_u}(w_u) \right \|^2 }_{Q_1} + 3\mathbb{E} \underbrace{\left \| {\alpha}_{i_u}^{u,\ell} - \widehat{\alpha}_{i_u}^{u,\ell} \right \|^2 }_{Q_2}
 + 3\mathbb{E} \underbrace{\left \|  \frac{1}{l} \sum_{i=1}^l \alpha_{i}^{u,\ell} - \frac{1}{l} \sum_{i=1}^l \widehat{\alpha}_{i}^{u,\ell}  \right \|^2}_{Q_3}
\end{align}
where the  inequality (a) uses $\| \sum_{i=1}^n a_i \|^2 \leq n \sum_{i=1}^n \| a_i \|^2 $.
We will give the upper bounds for the expectations  of $Q_1$, $Q_2$ and $Q_3$  respectively.
\begin{align}\label{AsySGHT_lem4_2}
& \nonumber \mathbb{E} Q_1 = \mathbb{E} \left \| \nabla_{\mathcal{G}_\ell}  f_{i_u}(\widehat{w}_u) -\nabla_{\mathcal{G}_\ell}  f_{i_u}(w_u) \right \|^2
\\   \nonumber \leq &  L^2 \mathbb{E} \left \| \widehat{w}_u - {w}_u \right \|^2 = L^2 \gamma^2 \mathbb{E} \left \|  \sum_{u' \in D(u)}    \textbf{U}_{\psi(u')} \widehat{v}^{\psi(u')}_{u'} \right \|^2
\\    \leq &    \eta_1 L^2 \gamma^2 \sum_{u' \in D(u)} \mathbb{E}  \left \|       \widehat{v}^{\psi(u')}_{u'} \right \|^2
\end{align}
where the first inequality uses Assumption \ref{assumption3}, the second inequality uses $\| \sum_{i=1}^n a_i \|^2 \leq n \sum_{i=1}^n \| a_i \|^2 $.
\begin{align}\label{AsySGHT_lem4_3}
& \mathbb{E} Q_2 =   \mathbb{E}\left \|  {\alpha}_{i_u}^{u,\ell} - \widehat{\alpha}_{i_u}^{u,\ell} \right \|^2
\\ \nonumber   \leq &   \frac{\eta_2 L^2 \gamma^2}{l} \sum_{u'=1}^{\xi(u,\ell)-1}  \sum_{\widetilde{u} \in D(\xi^{-1}(u',\ell))} \cdot
 \left ( 1 -\frac{1}{l} \right )^{\xi(u,\ell)-u'-1} \mathbb{E}  \left \|       \widehat{v}^{\psi(\widetilde{u})}_{\widetilde{u}} \right \|^2
\end{align}
where the inequality uses Lemma \ref{AsySGHT_lemma2}.
\begin{align}\label{AsySGHT_lem4_4}
& \mathbb{E} Q_3 =    \mathbb{E} \left \|  \frac{1}{l} \sum_{i=1}^l \alpha_{i}^{u,\ell} - \frac{1}{l} \sum_{i=1}^l \widehat{\alpha}_{i}^{u,\ell}   \right \|^2
\\    \leq & \nonumber  \frac{1}{l} \sum_{i=1}^l \mathbb{E} \left \|  \alpha_{i}^{u,\ell}  - \widehat{\alpha}_{i}^{u,\ell}   \right \|^2
\\    \leq & \nonumber \frac{\eta_2 L^2 \gamma^2}{l} \sum_{u'=1}^{\xi(u,\ell)-1}  \sum_{\widetilde{u} \in D(\xi^{-1}(u',\ell))} \cdot
 \left ( 1 -\frac{1}{l} \right )^{\xi(u,\ell)-u'-1} \mathbb{E}  \left \|       \widehat{v}^{\psi(\widetilde{u})}_{\widetilde{u}} \right \|^2
\end{align}
where  the first inequality uses $\| \sum_{i=1}^n a_i \|^2 \leq n \sum_{i=1}^n \| a_i \|^2 $, the second inequality uses (\ref{AsySGHT_lem4_3}).

\begin{align}\label{AsySGHT_lem4_1}
&  \mathbb{E} \left \| \widehat{v}_{u}^\ell  \right \|^2 \leq  2 \mathbb{E} \left \| \widehat{v}_{u}^\ell  - ( v^{t+1} )_{\mathcal{I}} \right \|^2 +  2 \mathbb{E} \left \|  v_{u}^\ell  \right \|^2
\\    \leq & \nonumber    6  \mathbb{E} {Q_1} + 6\mathbb{E} {Q_2}  + 6\mathbb{E} {Q_3} + 2 \mathbb{E} \left \|  v_{u}^\ell  \right \|^2
\\    \leq & \nonumber 6 \eta_1 L^2 \gamma^2 \sum_{u' \in D(u)} \mathbb{E}  \left \|       \widehat{v}^{\psi(u')}_{u'} \right \|^2  +  2 \mathbb{E} \left \|  v_{u}^\ell  \right \|^2
\\    & \nonumber + \frac{12 \eta_2 L^2 \gamma^2}{l} \sum_{u'=1}^{\xi(u,\ell)-1}  \sum_{\widetilde{u} \in D(\xi^{-1}(u',\ell))}
  \left ( 1 -\frac{1}{l} \right )^{\xi(u,\ell)-u'-1} \mathbb{E}  \left \|       \widehat{v}^{\psi(\widetilde{u})}_{\widetilde{u}} \right \|^2
\end{align}
where the second inequality uses Lemma \ref{AsySGHT_lemma2}.
This completes the proof.
\end{proof}

Based on the basic inequalities in Lemma \ref{AsySPSAGA_lemma3}, we provide the proof of Theorem \ref{theorem2} in the following.

\begin{proof}
Similar to (\ref{EqThm2_2}),  we have that
\begin{align}\label{EqThm4_1}
 & \mathbb{E} f (w_{t+|K(t)|}) - \mathbb{E}f (w_{t})
 \\ \nonumber \stackrel{ (a) }{\leq} &   -  \frac{\gamma}{4} \| \nabla f(w_t) \|^2 +  \frac{\eta_2 \gamma L^2 \gamma^2}{2}  \sum_{u \in K(t)}\sum_{u' \in D(u)} \mathbb{E} \|   \widehat{v}^{\psi(u')}_{u'} \|^2
\\ \nonumber &   + \left ( \frac{\eta_1 \gamma^3 L^2 q\eta_1 }{2}+ \frac{L_{\max} \gamma^2}{2} \right ) \sum_{u \in K(t)}\mathbb{E} \|  \widehat{v}^{\psi(u)}_u  \|^2
 \\ \nonumber  \stackrel{ (b) }{\leq} &   -  \frac{\gamma}{4} \| \nabla f(w_t) \|^2 +  \frac{\eta_2  L^2 \gamma^3}{2}  \sum_{u \in K(t)}\sum_{u' \in D(u)} \mathbb{E} \|   \widehat{v}^{\psi(u')}_{u'} \|^2
\\ \nonumber &   + \left ( \frac{\eta_1 \gamma^3 L^2 q\eta_1 }{2}+ \frac{L_{\max} \gamma^2}{2} \right ) \sum_{u \in K(t)}  \left ( 6 \eta_1 L^2 \gamma^2 \sum_{u' \in D(u)} \cdot \right .
\\ & \nonumber
\\ \nonumber &
 \left .  \mathbb{E}  \left \|       \widehat{v}^{\psi(u')}_{u'} \right \|^2  + \frac{12 \eta_2 L^2 \gamma^2}{l} \sum_{u'=1}^{\xi(u,\ell)-1}  \sum_{\widetilde{u} \in D(\xi^{-1}(u',\ell))} \cdot    \left ( 1 -\frac{1}{l} \right )^{\xi(u,\ell)-u'-1} \mathbb{E}  \left \|       \widehat{v}^{\psi(\widetilde{u})}_{\widetilde{u}} \right \|^2
  +  2 \mathbb{E} \left \|  v_{u}^{\psi({u})}  \right \|^2 \right )
\\ \nonumber  \stackrel{ (c) }{\leq} &   -  \frac{\gamma}{4} \| \nabla f(w_t) \|^2 +  \frac{\eta_2 L^2 \gamma^3}{2}  \eta_1 q \tau G
 + \left ( \gamma L^2 q\eta_1^2 + L_{\max} \right )3 \eta_1 L^2 \gamma^4 \eta_1 q \tau G
\\ \nonumber &  + \left ( \gamma L^2 q\eta_1^2 + L_{\max} \right ) 6 \eta_2 L^2   \eta_1 q \tau G \gamma^4
+ \left ( \gamma L^2 q\eta_1^2 + L_{\max} \right ) \gamma^2 \sum_{u \in K(t)}     \mathbb{E} \left \|  v_{u}^{\psi({u})}  \right \|^2
 \\ \nonumber  = &   -  \frac{\gamma}{4} \| \nabla f(w_t) \|^2  + \left ( \gamma L^2 q\eta_1^2 + L_{\max} \right ) \gamma^2 \sum_{u \in K(t)}     \mathbb{E} \left \|  v_{u}^{\psi({u})}  \right \|^2
\\ \nonumber &
 +  \left ( \frac{\eta_2  }{2}   + 3   \gamma\left ( \gamma  q\eta_1^2 + L_{\max} \right ) (\eta_1+2\eta_2)  \right ) \gamma^3 L^2 \eta_1 q \tau G
 \\ \nonumber  \stackrel{ (d) }{\leq} &  -  \frac{\gamma}{4} \| \nabla f(w_t) \|^2  + \left ( \gamma L^2 q\eta_1^2 + L_{\max} \right ) \gamma^2 \cdot
\\ \nonumber &
  \sum_{u \in K(t)}   \left ( 4\frac{L^2 \eta_1}{l} \sum_{k'=1}^{\upsilon(u)} \left ( 1 -\frac{1}{l} \right )^{\upsilon(u)-k'} \sigma(w_{\overline{u}_{k'}}) \right .
\\   \nonumber & \left . + 2 L^2 \left ( 1 -\frac{1}{l} \right )^{\upsilon(u)}  \sigma(w_0)
 + 4 L^2 \sigma(w_{t})  + 8 L^2 \gamma^2 \eta_1^2 q G \right )
\\ \nonumber & +  \left ( \frac{\eta_2  }{2}   + 3   \left ( \gamma  q\eta_1^2 + L_{\max} \right ) (\eta_1+2\eta_2)  \right ) \gamma^4 L^2 \eta_1 q \tau G
 \\ \nonumber \stackrel{ (e) }{\leq} &   -  \frac{\gamma \mu}{4} e(w_t)-  \frac{\gamma \mu^2}{4} \sigma(w_t)
\\ \nonumber &  + \left ( \gamma L^2 q\eta_1^2 + L_{\max} \right ) \gamma^2 \eta_1 q \cdot
\left ( 2 L^2  \left ( 1 -\frac{1}{l} \right )^{\upsilon(t)}  \sigma (w_0)
 +4 L^2 \sigma(w_{t}) \right )
\\ \nonumber & +  \left ( \left ( \frac{\eta_2  }{2}   + 3   \left ( \gamma  q\eta_1^2 + L_{\max} \right ) (\eta_1+2\eta_2)  \right ) \tau  + \left ( \gamma L^2 q\eta_1^2 + 8 L_{\max} \right )  \eta_1 q   \eta_1  \right ) \gamma^4 L^2 \eta_1 q  G
\\ \nonumber &
 + 4\left ( \gamma L^2 q\eta_1^2 + L_{\max} \right )    \frac{L^2 \eta_1^2 q }{l } \gamma^2 \cdot
 \sum_{k'=1}^{\upsilon(t)} \left ( 1 -\frac{1}{l} \right )^{\upsilon(t)-k'} \sigma(w_{\overline{u}_{k'}})
  \end{align}
where the  inequalities (a)  use (\ref{EqThm2_2}), the equality (b) uses Lemma \ref{lemma3}, the inequality (c) uses Assumption \ref{assumption1}, the inequality (d) uses Lemma \ref{The3lemma2}, the inequality (e) uses Assumption \ref{assumption4}.

According to (\ref{EqThm4_1}), we have that
\begin{align}\label{EqThm4_2}
  e (w_{t+|K(t)|})
 \leq & \left ( 1 -  \frac{\gamma \mu}{4} \right ) e(w_t) + c_1  \left (   \left ( 1 -\frac{1}{l} \right )^{\upsilon(t)}  \sigma(w_0)
 + 2 \sigma(w_{t}) \right )
\\ \nonumber & +  c_0
 + c_2 \sum_{k'=1}^{\upsilon(t)} \left ( 1 -\frac{1}{l} \right )^{\upsilon(t)-k'} \sigma(w_{\overline{u}_{k'}})  -  \frac{\gamma \mu^2}{4} \sigma(w_t)
   \\ \nonumber  = & \left ( 1 -  \frac{\gamma \mu}{4} \right ) e(w_t)  +\left (  -  \frac{\gamma \mu^2}{4} +2 c_1+c_2 \right )  \sigma(w_t)
\\ \nonumber &    + c_1    \left ( 1 -\frac{1}{l} \right )^{\upsilon(t)}  \sigma(w_0)
 + c_2  \sum_{k'=1}^{\upsilon(t)-1} \left ( 1 -\frac{1}{l} \right )^{\upsilon(t)-k'} \sigma(w_{\overline{u}_{k'}}) +  c_0
  \end{align}
where ${c_0} =  ( ( \frac{\eta_2  }{2}   + 3   ( \gamma  q\eta_1^2 + L_{\max}  ) (\eta_1+2\eta_2) ) \tau +  ( \gamma L^2 q\eta_1^2 + 8 L_{\max}  )  \eta_1 q   \eta_1  ) \gamma^4 L^2 \eta_1 q  G$,  $c_1 = {\left ( \gamma L^2 q\eta_1^2 + L_{\max} \right ) \gamma^2 \eta_1 q 2 L^2 }$,  ${c_2}  = {4\left ( \gamma L^2 q\eta_1^2 + L_{\max} \right )    \frac{L^2 \eta_1^2 q }{l } \gamma^2}$, $\{ \overline{u}_0,\overline{u}_1, \ldots,\overline{u}_{\upsilon(u)-1}\}$ are the  all start time counters for the global time counters from 0 to $u$.

We define the Lyapunov function as $\mathcal{L}_t=\sum_{k=0}^{\upsilon(t)} \rho^{\upsilon(t)-k} e(w_{\overline{u}_{k}})$ where $\rho \in (1 -\frac{1}{l},1)$, we have that
 \begin{align}\label{EqThm4_3}
 & \mathcal{L}_{t+|K(t)|}
\\   = & \nonumber \rho^{\upsilon(t)+1} e(w_0) +  \sum_{k=0}^{\upsilon(t)} \rho^{\upsilon(t)-k} e(w_{\overline{u}_{k+1}})
\\   \stackrel{ (a) }{\leq} & \nonumber \rho^{\upsilon(t)+1} e(w_0) +   \sum_{k=0}^{\upsilon(t)} \rho^{\upsilon(t)-k} \left [  \left ( 1 -  \frac{\gamma \mu}{4} \right ) e(w_{\overline{u}_{k}})  \right .
\\ \nonumber &
\left .  +\left (  -  \frac{\gamma \mu^2}{4} +2 c_1+c_2 \right )  \sigma(w_{\overline{u}_{k}})  + c_1    \left ( 1 -\frac{1}{l} \right )^{k}  \sigma(w_0) \right .
\\ \nonumber & \left. + c_2  \sum_{k'=1}^{k-1} \left ( 1 -\frac{1}{l} \right )^{k-k'} \sigma(w_{\overline{u}_{k'}}) +  c_0 \right ]
\\ \label{EqThm4_4}   = &  \rho^{\upsilon(t)+1} e(w_0) +  \left ( 1 -  \frac{\gamma \mu}{4} \right )\mathcal{L}_{t} +  \sum_{k=0}^{\upsilon(t)} \rho^{\upsilon(t)-k}
\\ \nonumber &
\left [  \left (  -  \frac{\gamma \mu^2}{4} +2 c_1+c_2 \right )  \sigma(w_{\overline{u}_{k}})   + c_1    \left ( 1 -\frac{1}{l} \right )^{k}  \sigma(w_0) \right .
\\ \nonumber & \left . + c_2  \sum_{k'=1}^{k-1} \left ( 1 -\frac{1}{l} \right )^{k-k'} \sigma(w_{\overline{u}_{k'}})    \right ]+\sum_{k=0}^{\upsilon(t)} \rho^{\upsilon(t)-k}  c_0
\\   \stackrel{ (b) }{\leq} & \nonumber \rho^{\upsilon(t)+1} e(w_0) +  \left ( 1 -  \frac{\gamma \mu}{4} \right )\mathcal{L}_{t}
+    \left (  -  \frac{\gamma \mu^2}{4} +2 c_1+c_2 \right )  \sigma(w_{\overline{u}_{\upsilon(t)}}) +  \frac{c_0}{1-\rho}
\\   \stackrel{ (c) }{\leq} & \nonumber \rho^{\upsilon(t)+1} e(w_0) +  \left ( 1 -  \frac{\gamma \mu}{4} \right )\mathcal{L}_{t}
 -    \left ( \frac{\gamma \mu^2}{4} -2 c_1-c_2 \right )\frac{2}{L}  e(w_{\overline{u}_{\upsilon(t)}}) +  \frac{c_0}{1-\rho}
 \end{align}
where the  inequality (a) uses (\ref{EqThm4_2}), the  inequality (b) holds by appropriately choosing $\gamma$ such that the terms related to  $\sigma(w_{\overline{u}_{k}})$ ($k=0,\cdots,\upsilon(t)-1$) are negative,  because the signs related to  the lowest  orders of $\sigma(w_{\overline{u}_{k}})$ ($k=0,\cdots,\upsilon(t)-1$)  are negative. In the following, we give the detailed analysis of choosing $\gamma$ such that the terms related to $\sigma(w_{\overline{u}_{k}})$ ($k=0,\cdots,\upsilon(t)-1$)  are negative. We first consider $k=0$. Assume that $ C(\sigma(w_{0}))$ is the coefficient term of $\sigma(w_{0}))$ in  (\ref{EqThm4_4}), we have that
 \begin{align}\label{AsySGHT_theorem2_2.1.1}
& C(\sigma(w_{0}))
\\   = & \nonumber   \rho^{\upsilon(t)}\left ( -  \frac{\gamma \mu^2}{4} +2 c_1+c_2 \right )+  c_1   \sum_{k=0}^{\upsilon(t)} \rho^{\upsilon(t)-k}  \left ( 1 -\frac{1}{l} \right )^{k}
\\   = & \nonumber   \rho^{\upsilon(t)}\left ( -  \frac{\gamma \mu^2}{4} +2 c_1+c_2+  c_1   \sum_{k=0}^{\upsilon(t)} \left (\frac{ 1 -\frac{1}{l}}{\rho} \right )^{k}  \right )
\\  \leq & \nonumber  \rho^{\upsilon(t)}\left ( -  \frac{\gamma \mu^2}{4} +2 c_1+c_2+  c_1   \frac{1}{1- \frac{ 1 -\frac{1}{l}}{\rho}} \right )
\\   = & \nonumber  \rho^{\upsilon(t)}\left ( -  \frac{\gamma \mu^2}{4} +c_2+  c_1 \left ( 2+ \frac{1}{1- \frac{ 1 -\frac{1}{l}}{\rho}} \right ) \right )
 \end{align}
Based on (\ref{AsySGHT_theorem2_2.1.1}), we can carefully choose $\gamma$ such that $ -  \frac{\gamma \mu^2}{4} +c_2+  c_1 \left ( 2+ \frac{1}{1- \frac{ 1 -\frac{1}{l}}{\rho}} \right )    \leq 0$.

 Assume that $C(\sigma(w_{\overline{u}_{k}}))$ is the coefficient term of $\sigma(w_{\overline{u}_{k}})$ ($k=1,\cdots,\upsilon(t)-1$) in the big square brackets of (\ref{EqThm4_4}), we have that
  \begin{align}\label{AsySGHT_theorem2_2.1.2}
 & C(\sigma(w_{\overline{u}_{k}}))
\\   = & \nonumber \rho^{\upsilon(t)-k}  \left (  -  \frac{\gamma \mu^2}{4} +2 c_1+c_2 \right ) + c_2  \sum_{v=k+1}^{\upsilon(t)-1} \left ( 1 -\frac{1}{l} \right )^{v-k} \rho^{\upsilon(t)-v}
\\   = & \nonumber \rho^{\upsilon(t)-k} \left ( -  \frac{\gamma \mu^2}{4} +2 c_1+c_2  + c_2  \sum_{v=k+1}^{\upsilon(t)-1} \left ( 1 -\frac{1}{l} \right )^{v-k} \rho^{k-v} \right )
\\   = & \nonumber \rho^{\upsilon(t)-k} \left ( -  \frac{\gamma \mu^2}{4} +2 c_1+c_2  + c_2  \sum_{v=k+1}^{\upsilon(t)-1} \left ( \frac{1 -\frac{1}{l}}{\rho}\right )^{v-k}  \right )
\\   \leq & \nonumber \rho^{\upsilon(t)-k} \left ( -  \frac{\gamma \mu^2}{4} +2 c_1+c_2 \left ( 1+   \frac{1}{1- \frac{ 1 -\frac{1}{l}}{\rho}}  \right ) \right )
 \end{align}
Based on (\ref{AsySGHT_theorem2_2.1.2}), we can carefully choose $\gamma$ such that $-  \frac{\gamma \mu^2}{4} +2 c_1+c_2 \left ( 1+   \frac{1}{1- \frac{ 1 -\frac{1}{l}}{\rho}}  \right ) \leq 0$.

 Thus, based on (\ref{EqThm4_3}), we have that
\begin{align}\label{EqThm4_7}
& \left ( \frac{\gamma \mu^2}{4} -2 c_1-c_2 \right )\frac{2}{L}  e(w_{\overline{u}_{k}})
 \\   \leq & \nonumber  \left ( \frac{\gamma \mu^2}{4} -2 c_1-c_2 \right )\frac{2}{L}  e(w_{\overline{u}_{k}}) + \mathcal{L}_{t+|K(t)|}
 \\   \stackrel{ (a) }{\leq} & \nonumber \rho^{\upsilon(t)+1} e(w_0) +  \left ( 1 -  \frac{\gamma \mu}{4} \right )\mathcal{L}_{t}  +  \frac{c_0}{1-\rho}
\\   \stackrel{ (b) }{\leq} & \nonumber \left ( 1 -  \frac{\gamma \mu}{4} \right )^{\upsilon(t)+1} \mathcal{L}_{0} +  \rho^{\upsilon(t)+1}e(w_0)  \sum_{k=0}^{\upsilon(t)+1} \left (\frac{1 -  \frac{\gamma \mu}{4}}{\rho} \right )^k
\\ \nonumber & + \frac{c_0}{1-\rho}\sum_{k=0}^{\upsilon(t)} \left ( 1 -  \frac{\gamma \mu}{4} \right )^{k}
\\   \leq & \nonumber \left ( 1 -  \frac{\gamma \mu}{4} \right )^{\upsilon(t)+1} e(w_0)   +  \rho^{\upsilon(t)+1}e(w_0)  \frac{1}{1-\frac{1 -  \frac{\gamma \mu}{4}}{\rho} } + \frac{c_0}{1-\rho} \frac{4}{\gamma \mu }
\\   \stackrel{ (c) }{\leq} & \nonumber  \frac{{2\rho- 1 +  \frac{\gamma \mu}{4} }}{\rho- 1 +  \frac{\gamma \mu}{4} } \rho^{\upsilon(t)+1}e(w_0)    + \frac{c_0}{1-\rho} \frac{4}{\gamma \mu }
\end{align}
where the inequality (a) follows from  (\ref{EqThm4_3}), the inequality (b) holds by using the inequality (\ref{EqThm4_3}) recursively, the inequality (c) uses the fact that $ 1 -  \frac{\gamma \mu}{4} < \rho$.

According to (\ref{EqThm4_7}), we have that
\begin{align}\label{EqThm4_8}
 \nonumber  e(w_{\overline{u}_{k}}) \leq & \frac{{2\rho- 1 +  \frac{\gamma \mu}{4} }}{(\rho- 1 +  \frac{\gamma \mu}{4} )\left ( \frac{\gamma \mu^2}{4} -2 c_1-c_2 \right )} \rho^{\upsilon(t)+1}e(w_0)
\\  & + \frac{4 c_0}{ \gamma \mu (1-\rho) \left ( \frac{\gamma \mu^2}{4} -2 c_1-c_2 \right )}
\end{align}

Thus, to achieve the accuracy $\epsilon$ of (\ref{formulation1}) for AFSAGA-VP, \emph{i.e.}, $\mathbb{E} f (w_{\overline{u}_{k}}) -f(w^*) \leq \epsilon$,  we can carefully choose $\gamma$ such that
\begin{eqnarray}
\frac{4 c_0}{ \gamma \mu (1-\rho) \left ( \frac{\gamma \mu^2}{4} -2 c_1-c_2 \right )} &\leq& \frac{\epsilon}{2}
\\
0<1 -  \frac{\gamma \mu}{4} & <&1
\\ -  \frac{\gamma \mu^2}{4} +2 c_1+c_2 \left ( 1+   \frac{1}{1- \frac{ 1 -\frac{1}{l}}{\rho}}  \right ) &\leq& 0
\\  -  \frac{\gamma \mu^2}{4} +c_2+  c_1 \left ( 2+ \frac{1}{1- \frac{ 1 -\frac{1}{l}}{\rho}} \right )   &\leq& 0
\end{eqnarray}
and let $\frac{{2\rho- 1 +  \frac{\gamma \mu}{4} }}{(\rho- 1 +  \frac{\gamma \mu}{4} )\left ( \frac{\gamma \mu^2}{4} -2 c_1-c_2 \right )} \rho^{\upsilon(t)+1}e(w_0)  \leq \frac{\epsilon}{2}$, we have that
\begin{eqnarray}\label{EqThm2_5}
\upsilon(t)  \geq \frac{\log \frac{2 \left ({{2\rho- 1 +  \frac{\gamma \mu}{4} }} \right ) e(w_0) }{\epsilon {(\rho- 1 +  \frac{\gamma \mu}{4} )\left ( \frac{\gamma \mu^2}{4} -2 c_1-c_2 \right )} }}{\log \frac{1}{\rho}}
\end{eqnarray}
This completes the proof.
\end{proof}

\bibliography{sample1-base}

\end{document}